\documentclass{article}
\usepackage[preprint]{neurips_2026}
\usepackage[utf8]{inputenc}
\usepackage[T1]{fontenc}
\usepackage{lmodern}
\usepackage{amsmath,amssymb,amsfonts,amsthm,bm,mathtools}
\usepackage{enumitem}
\usepackage{algorithm}
\usepackage{algpseudocode}
\usepackage[disable]{microtype}
\usepackage{xcolor}
\usepackage{natbib}
\usepackage{booktabs,tabularx,array}
\usepackage{multirow}
\usepackage{makecell}
\usepackage{colortbl}
\usepackage{graphicx}
\graphicspath{{figures/}}
\usepackage{subcaption}
\usepackage{nicefrac}
\usepackage{wrapfig}
\usepackage{pifont}
\usepackage{hyperref}

\newcommand{\R}{\mathbb{R}}

\newcommand{\E}{\mathbb{E}}

\newcommand{\op}{\mathrm{op}}

\newcommand{\cA}{\mathcal{A}}

\newcommand{\cH}{\mathcal{H}}
\newcommand{\cT}{\mathcal{T}}
\newcommand{\cI}{\mathcal{I}}
\newcommand{\cW}{\mathcal{W}}
\newcommand{\cTprobe}{\mathcal{T}_{\mathrm{probe}}}

\newcommand{\tilO}{\widetilde{\mathcal{O}}}
\newcommand{\DynReg}{\mathrm{DynReg}}

\newcommand{\range}{\mathrm{range}}
\DeclareMathOperator{\TopEig}{TopEig}
\DeclareMathOperator*{\argmax}{arg\,max}
\DeclareMathOperator{\tr}{tr}

\newtheorem{theorem}{Theorem}[section]
\newtheorem{assumption}[theorem]{Assumption}
\newtheorem{corollary}[theorem]{Corollary}
\newtheorem{proposition}[theorem]{Proposition}
\newtheorem{lemma}[theorem]{Lemma}
\newtheorem{remark}[theorem]{Remark}
\newcommand{\cmark}{\ding{51}}
\newcommand{\xmark}{\ding{55}}

\title{Catching a Moving Subspace:\\Low-Rank Bandits Beyond Stationarity}

\author{%
  Hamed Khosravi \\
  H.\ Milton Stewart School of Industrial and Systems Engineering \\
  Georgia Institute of Technology \\
  Atlanta, GA 30332 \\
  \texttt{hkhosravi7@gatech.edu} \\
  \And
  Xiaoming Huo \\
  H.\ Milton Stewart School of Industrial and Systems Engineering \\
  Georgia Institute of Technology \\
  Atlanta, GA 30332 \\
  \texttt{huo@gatech.edu} \\
}

\begin{document}
\maketitle

\begin{abstract}
Many real bandit deployments (recommendation, clinical dosing, ad
targeting) share two structural facts handled only in isolation by
prior work: rewards live on a low-dimensional latent subspace, and
that subspace drifts. Stationary low-rank bandits exploit the rank
but break under any subspace change; non-stationary linear bandits
adapt to drift but pay the ambient rate $\tilO(d\sqrt T)$. We study
\emph{piecewise-stationary low-rank} linear contextual bandits with
scalar feedback: $\theta_t = B_k^\star w_t$ for a rank-$r$ factor
$B_k^\star\in\R^{d\times r}$ that is constant within each of $K$
unknown segments and may shift at the boundaries. Our results are
tight along three axes. \textbf{(i) Identification boundary.} With
single-play scalar rewards, the moving subspace is recoverable
through quadratic functionals of rewards if and only if three
probe-side conditions hold (known noise variance, bounded
state-noise coupling, full-dimensional probe support). Each is
individually necessary in the unrestricted-second-moment problem,
and the three are jointly sufficient: a characterization of the
boundary of the solvable region, not just a sufficient interior
point. \textbf{(ii) Algorithm and dynamic regret.} \textbf{SPSC}
(Single-Play Subspace-Calibrated Optimism) interleaves isotropic
probes with windowed projected ridge-UCB exploitation inside the
learned $r$-dimensional subspace; a CUSUM-style adaptive variant
discovers segment boundaries online. The costed dynamic regret is
$\tilO(r\sqrt T)+\tilO(T^{2/3})+O(W\,V_{\rm in})$, replacing the
ambient $d\sqrt T$ rate with the intrinsic rank.
\textbf{(iii) Empirics.} On eleven benchmarks spanning synthetic,
UCI/MovieLens, semi-synthetic clinical, and ZOZOTOWN production-log
data, SPSC outperforms non-stationary and low-rank baselines
whenever $d-r\gtrsim T^{1/6}$, matching the analytical crossover.
To our knowledge, this is the first work to characterize the
identification boundary and attain the intrinsic-rank
dynamic-regret rate in this setting.
\end{abstract}

%% ============================================================
\section{Introduction}
\label{sec:intro}

High-dimensional sequential decisions (personalized recommendation,
clinical dosing, ad allocation) share two structural facts that the
linear-bandit literature has so far treated only in isolation: rewards
live on a low-dimensional latent subspace, and that subspace drifts.
User preferences, clinical responses, and ad-creative effects occupy
a thin slice of the ambient $\R^d$ with intrinsic rank $r\ll d$; but
the slice is \emph{not static}: tastes drift, cohorts change,
treatment regimes get revised, and the latent factor itself moves
across unknown change points. Production deployments make both
facts visible: Spotify's homepage contextual bandit explicitly
targets daily shifts in user
intent~\citep{feijer2025calibrated}; the IWPC warfarin
cohort~\citep{international2009estimation} has been re-analyzed to
expose three latent patient subgroups whose dose--response
coefficients flip sign across boundaries~\citep{liu2025change}; the
public Open Bandit logs~\citep{saito2020open} ship scalar feedback
collected across multiple ad campaigns at ZOZOTOWN.

\paragraph{Where the literature stops.}
Stationary low-rank
bandits~\citep{jun2019bilinear,jedra2024low,jang2024efficient}
deliver the $\tilO(r\sqrt T)$ rate but assume a fixed factor; any
non-trivial subspace change voids the analysis ($\tilO$ suppresses
polylog factors in $T,d,r,1/\delta$). Non-stationary linear
bandits~\citep{russac2019weighted,cheung2019learning} absorb drift
through sliding windows or discounting but operate in $\R^d$ and pay
$\tilO(d\sqrt T)$ within each segment, ignoring intrinsic rank
entirely. Naive combinations fail in instructive ways: sliding-window
ridge in $\R^d$ never recovers the subspace; a one-shot subspace
estimate breaks the moment a change point fires; per-segment re-runs
of stationary low-rank algorithms require oracle access to the
boundaries. No prior algorithm recovers a changing subspace under
scalar feedback, and no prior analysis prices that recovery against
the exploitation regret it buys.

\paragraph{This work.}
We study \emph{piecewise-stationary low-rank} linear contextual
bandits with scalar feedback: $\theta_t = B_k^\star w_t$ for a
rank-$r$ factor $B_k^\star\in\R^{d\times r}$ that is constant within
each of $K$ unknown segments
$\cI_k=[\tau_{k-1},\tau_k)$ delimited by change points
$1=\tau_0<\tau_1<\dots<\tau_K=T{+}1$ and may shift at the boundaries,
with the learner only observing scalar responses
$y_t=x_t^\top\theta_t+\varepsilon_t$ (\S\ref{sec:setting}). The
genuine difficulty is twofold. (a) A single scalar response $y_t$ is
a rank-one quadratic measurement of $\theta_t\theta_t^\top$, so the
moving subspace is recoverable only through \emph{second-order}
probing. (b) Every probe diverts a round away from exploitation, so
probe acquisition must itself be priced against the regret it
enables. Our results are tight along three axes.

\paragraph{Contributions.}
\textbf{(i) Identification boundary
(Theorem~\ref{thm:identification}).} Three probe-side conditions
(known noise variance, bounded state--noise coupling,
full-dimensional probe support) are individually necessary in the
unrestricted-second-moment problem and jointly sufficient for
recovering $\range(B_k^\star)$ from quadratic functionals of scalar
rewards. The necessity half, three matching impossibility results
(Propositions~\ref{prop:nonid_a}--\ref{prop:nonid_c}), is what
distinguishes this from a generic identifiability statement: it
characterizes the boundary of the solvable region, not just a
sufficient interior point.
\textbf{(ii) Algorithm and costed dynamic regret
(Theorem~\ref{thm:spsc_regret}).} \textbf{SPSC} (Alg.~\ref{alg:spsc})
interleaves isotropic probes with windowed projected ridge-UCB
exploitation inside the learned $r$-dimensional subspace via a
quadratic-measurement identity; a \textbf{SPSC-Adaptive}
variant (Alg.~\ref{alg:spsc_adaptive}) uses a CUSUM-style detector
to discover segment changes online. The costed dynamic regret is
$\DynReg_T^{(c)}\le \tilO(r\sqrt T)+\tilO(T^{2/3})+O(W\,V_{\rm in})$,
replacing the ambient $d\sqrt T$ rate with the intrinsic rank $r$,
where $W$ is the exploitation window and $V_{\rm in}$ the
within-segment path variation.
\textbf{(iii) Eleven benchmarks against eleven baselines.} A
synthetic phase-transition grid, five UCI environments plus
MovieLens, two semi-synthetic clinical datasets, the small-$d$
piecewise-stationary stress test of \citet{russac2019weighted}, and
Open Bandit production logs, against LinUCB, D-LinUCB,
SW-LinUCB, Restart-LinUCB, LowOFUL, VOFUL, LowRank-Reward, LinTS,
SW-LinTS, and adapted stationary low-rank methods BOSS and Jedra.
SPSC delivers regret reductions whenever $d-r\gtrsim T^{1/6}$,
matching the analytical crossover. A direct necessity stress test
(App.~\ref{app:robustness_abc}, part C) restricts probe coverage to a
proper subspace and reproduces the $\Omega(T)$ blow-up predicted by
Proposition~\ref{prop:nonid_c}, confirming the third identifiability
condition is not a proof artifact.

To our knowledge, this is the first work to characterize the
identification boundary and to attain the intrinsic-rank
dynamic-regret rate in the piecewise-stationary low-rank setting
under scalar feedback.

\paragraph{Related work.}
Stationary low-rank
bandits~\citep{jun2019bilinear,kang2022efficient,jedra2024low,jang2024efficient,stojanovic2023spectral,duong2024beyond,lu2021low}
achieve intrinsic-rank rates but assume a fixed $B^\star$;
non-stationary linear
bandits~\citep{russac2019weighted,cheung2019learning,abbasi2023new,hou2024almost}
bound dynamic regret in the ambient dimension and ignore low-rank
structure; cost-aware
observation~\citep{seldin2014prediction,tucker2023bandits,elumar2024multi}
prices information acquisition for fixed parameters, not for
identifying a changing subspace. SPSC sits at the intersection of
all three: it (i) attains regret scaling with intrinsic rank $r$,
(ii) adapts to changing subspaces without oracle boundaries, and
(iii) prices probe acquisition explicitly.
Appendix~\ref{app:related_table} gives the full positioning
(Tables~\ref{tab:related_work_summary}--\ref{tab:published_comparison}).

%% ============================================================
\section{Setting and probe-based identification}
\label{sec:setting}

\paragraph{Piecewise low-rank model.}
At every round $t$ the reward parameter admits a factorization
$\theta_t = B_k^\star w_t$, where $B_k^\star \in \R^{d\times r}$ has
orthonormal columns and is piecewise constant over $K\ge 1$ segments
$\cI_k$ with unknown change points
$1=\tau_0<\tau_1<\dots<\tau_K=T{+}1$, and the latent state $w_t\in
\R^r$ follows a stable linear dynamical system inside each segment
\[
 w_t = A_k w_{t-1} + \eta_{t-1}, \qquad t \in \cI_k,
\]
with unknown $A_k$ satisfying $\rho(A_k)\le 1{-}\alpha_0$ for some
$\alpha_0>0$ (used only to ensure $\|w_t\|\le S_w$ via the stationary
covariance bound) and zero-mean innovation $\eta_{t-1}$ with
covariance $\Sigma_{\eta,k}\succ 0$. We
assume uniform action-boundedness $\sup_{x\in\cA_t}\|x\|\le R_\cA$,
a known probe distribution $Q$ with $\E[uu^\top]\succeq \rho_Q I_d$ and
sub-Gaussian marginals (bounded $\|u\|\le L$ is a special case; see
Remark~\ref{rem:subgauss_probe}), conditionally
$\sigma_\varepsilon$-sub-Gaussian noise with unknown variance, bounded
state-noise coupling
$\|\E[\varepsilon_t\theta_t]\|_2 \le \epsilon_\times$, and
full-dimensional probe coverage. Theorem~\ref{thm:identification}
below shows these three probe-side conditions are individually
necessary and jointly sufficient for identifying the changing
subspace from single-play scalar rewards. Throughout, $\delta\in(0,1)$
denotes the global confidence parameter (concentration bounds hold
with probability $\ge 1-\delta$) and $\lambda>0$ denotes the ridge
regularizer used by the windowed estimator of \S\ref{sec:algorithm}.

\begin{remark}[Probe distribution: scaled sphere]
\label{rem:subgauss_probe}
The theory uses scaled-sphere probes
$u_t = \sqrt d\,v_t$, $v_t\sim\mathrm{Unif}(\mathbb S^{d-1})$, which
satisfy $\|u_t\|_2 = \sqrt d$ \emph{exactly} (so no truncation event
is needed) and $\E[uu^\top] = I_d$. In our implementation we draw
$\tilde u\sim\mathcal N(0,I_d)$ and rescale
$u\leftarrow \sqrt d\,\tilde u/\|\tilde u\|$, which is exactly the
scaled-sphere distribution. The same argument applies to isotropic
Gaussian probes $u_t\sim\mathcal N(0,I_d)$ via a $\chi^2_d$
truncation at $L=\sqrt{2d\log(2T/\delta)}$
(Appendix~\ref{app:truncation}).
\end{remark}

\paragraph{Quadratic measurement identity.}
Let $\widehat\sigma^2$ be the algorithm's estimate of the noise variance
and define the centered probe statistic $s_t := y_t^2 - \widehat\sigma^2$
on every probe round $t\in\cTprobe$. Under the probe-noise
assumptions of \S\ref{sec:setting}, for
$\delta_\sigma := \widehat\sigma^2-\sigma_\varepsilon^2$,
\begin{equation}
\E[s_t \mid \cH_{t-1},u_t]
  = u_t^\top \widetilde{M}_t u_t + 2 u_t^\top
    \E[\varepsilon_t\theta_t\mid\cH_{t-1},u_t] - \delta_\sigma,
\label{eq:quad_identity}
\end{equation}
where $\widetilde{M}_t := \E[\theta_t\theta_t^\top\mid \cH_{t-1}]$
is the predictable second moment, whose range is contained in
$\range(B_k^\star)$ for $t\in\cI_k$; under the non-degenerate
innovation condition $\Sigma_{\eta,k}\succ 0$ (\S\ref{sec:setting}),
the probe-time average has range exactly $\range(B_k^\star)$
(Proposition~\ref{prop:segment_factorization_conf}). Under exact probe conditions
($\delta_\sigma = \epsilon_\times = 0$) this collapses to
$\E[s_t\mid\cdot] = u_t^\top\widetilde M_t u_t$: \emph{a scalar reward
carries a quadratic measurement of $\widetilde M_t$}.

\paragraph{Lifted subspace estimator.}
The probe moment operator
$\mathcal K : M \mapsto \E[(u^\top M u)\,uu^\top]$ admits a
closed-form inverse on symmetric matrices for scaled-sphere probes
$u=\sqrt d\,v$, $v\sim\mathrm{Unif}(\mathbb S^{d-1})$:
\begin{equation}
\mathcal K(M) = \tfrac{d}{d+2}\bigl(\tr(M)\,I_d + 2M\bigr),\qquad
\mathcal K^{-1}(N) = \tfrac{d+2}{2d}\,N - \tfrac{\tr(N)}{2d}\,I_d,
\qquad \|\mathcal K^{-1}\|_{\op\to\op}\le 1
\label{eq:K_inverse_intro}
\end{equation}
on the symmetric subspace (Lemma~\ref{lem:K_inverse}). Define the
\emph{lifted probe sample}
\begin{equation}
G_t := \mathcal K^{-1}(s_t\,u_tu_t^\top),\qquad
\E[G_t\mid \cH_{t-1}] = \widetilde M_t + \widetilde B_t,
\label{eq:lifted_sample}
\end{equation}
with bias $\widetilde B_t = -(\delta_\sigma/d)\,I_d$ exactly on
scaled-sphere probes, hence $\|\widetilde B_t\|_\op = |\delta_\sigma|/d$
(Lemma~\ref{lem:G_unbiased_conf}). Averaging lifted samples across
the probe rounds of segment $k$ yields
$\widehat M_k := m_k^{-1}\sum_{t\in\cT_k} G_t$ (with $\cT_k:=\cTprobe\cap\cI_k$
the probe rounds in segment $k$, $m_k:=|\cT_k|$), whose top-$r$
eigenvectors form $\widehat U_k\in\R^{d\times r}$ with
$\|\widehat P_k - P_k^\star\|_\op \le C_{\mathrm{sub}}\sqrt{\log(d/\delta)/m_k}$
(Corollary~\ref{cor:projector_conf} below). The variance
misspecification bias $\widetilde B$ is a scaled identity (population
level), shifting all eigenvalues uniformly and \emph{not rotating
eigenvectors}: subspace recovery is robust to the plug-in estimate
of $\sigma_\varepsilon^2$ at the population level. Finite-sample
projector concentration still has a floor
$\Delta_\sigma\propto |\widehat\sigma^2-\sigma_\varepsilon^2|/(d\,\lambda_{\min})$
due to the random fluctuation around the biased mean
(Cor.~\ref{cor:projector_conf}).

Positive identification through the lifted moment $G_t$ above and
the necessity result that follows together yield the following
theorem, the foundational result of the paper.

\begin{theorem}[Identification under priced probing]
\label{thm:identification}
Consider the model of \S\ref{sec:setting} with scaled-sphere probes
$u_t=\sqrt d\,v_t$, $v_t\sim\mathrm{Unif}(\mathbb S^{d-1})$, and the
lifted probe sample $G_t=\mathcal K^{-1}(s_t u_tu_t^\top)$ from
\eqref{eq:lifted_sample}.
\begin{enumerate}[leftmargin=1.5em,itemsep=2pt,topsep=2pt]
\item \textbf{Identifiability.} Under (i) known noise variance
$\sigma_\varepsilon^2$ ($\delta_\sigma=0$), (ii) bounded state-noise
coupling $\|\E[\varepsilon_t\theta_t]\|_2\le\epsilon_\times$, and
(iii) full-dimensional probe support, the lifted sample satisfies
$\E[G_t\mid\cH_{t-1}]=\widetilde M_t$ exactly (the bias $\widetilde B_t=-(\delta_\sigma/d)\,I_d$ from Lemma~\ref{lem:G_unbiased_conf} vanishes under~(i); the state-noise coupling term vanishes by sphere symmetry, regardless of $\epsilon_\times$), and the
segment-averaged estimator
$\widehat M_k=m_k^{-1}\sum_{t\in\cT_k}G_t$ satisfies, with
probability at least $1-\delta$, the projector concentration
$\|\widehat P_k-P_k^\star\|_{\op}\le C_{\mathrm{sub}}\sqrt{\log(d/\delta)/m_k}$;
in particular $\widehat U_k=\TopEig_r(\widehat M_k)$ recovers
$\range(B_k^\star)$ at rate $1/\sqrt{m_k}$ (Proof in Appendix~\ref{app:concentration_proof}).
\item \textbf{Necessity (unrestricted-second-moment scope).} Each of
(i), (ii), (iii) is individually necessary in the unrestricted
second-moment problem: removing any one yields two distinct
conditional second moments $\widetilde M_t\ne\widetilde M_t'$ that
generate identical observation laws
$\E[s_t\mid\cH_{t-1},u]$, so $\widetilde M_t$ is not identifiable from
quadratic measurements. The three constructions are independent (each
counterexample satisfies the remaining two conditions), so the three
form a minimal identifiability set in this class. Within the rank-$r$
LDS class of \S\ref{sec:setting} ($r<d$), rank deficiency partially
relaxes (i) since $\sigma_\varepsilon^2$ is identifiable from the
$d-r$ smallest eigenvalues; (ii) and (iii) remain needed (Proof in Appendix~\ref{app:necessity}).
\end{enumerate}

\end{theorem}

%% ============================================================
\section{Algorithm}
\label{sec:algorithm}

SPSC (Alg.~\ref{alg:spsc}) alternates two phases:
\begin{itemize}[leftmargin=1.5em,nosep]
\item \textbf{Probe rounds} ($t\in\cTprobe$): draw $u_t\sim Q$ (the
probe distribution from \S\ref{sec:setting}; e.g.\ the scaled-sphere
distribution of Remark~\ref{rem:subgauss_probe}), observe $y_t$,
accumulate $G_t$ in the current-segment estimator $\widehat M_t$,
and refresh $\widehat U_t = \TopEig_r(\widehat M_t)$.
\item \textbf{Exploitation rounds} ($t\notin\cTprobe$): project each
action $x\in\cA_t$ to $z_t(x) = \widehat U_{t-1}^\top x\in\R^r$, maintain a
\emph{windowed} ridge-UCB estimate
$\widehat a_t = \widetilde V_t^{-1}\widetilde b_t$ over a sliding
window $\cW_t$ of the last $W$ exploitation rounds, and play
$x_t = \argmax_{x\in\cA_t}
\bigl\{ z_t(x)^\top\widehat a_t + \beta_t^{(r,W)}\|z_t(x)\|_{\widetilde V_t^{-1}}
+ \gamma_t \|x\|_2 \bigr\}$.
\end{itemize}
The confidence radius
$\beta_t^{(r,W)}=\sigma_\varepsilon\sqrt{r\log(1+WR_\cA^2/(\lambda r))
+2\log(2K/\delta)}+\sqrt{\lambda}S_w$
scales with $\sqrt{r}$, not $\sqrt{d}$: this is where the intrinsic-rank
improvement enters. The correction term
$\gamma_t = \Gamma_k + V_{k,t}(W)$
absorbs subspace-mismatch error
$\Gamma_k\propto\varepsilon_k:=\|\widehat P_k - P_k^\star\|_{\op}$
and local within-window drift
$V_{k,t}(W):=\sum_{s,s+1\in\cI_k,\,t-W\le s<t}\|\theta_{s+1}-\theta_s\|_2$
(formal definitions in App.~\ref{app:proof_main_lemmas}). Per-round cost is $O(dr|\cA_t|+r^2)$ vs.\
$O(d^2|\cA_t|)$ for ambient LinUCB, an $O(d/r)$ speed-up in addition
to the statistical gain.

\paragraph{Why interleaved probing recovers the subspace.}
A single scalar response is one-dimensional, but the lifted statistic
$G_t=\mathcal K^{-1}(s_t u_tu_t^\top)$ in
\eqref{eq:lifted_sample} targets the conditional second moment
$\E[\theta_t\theta_t^\top\mid\cH_{t-1}]
= B_k^\star\,\E[w_tw_t^\top\mid\cH_{t-1}]\,(B_k^\star)^\top$ up to a
controlled bias. Averaging $G_t$ across probe rounds inside segment
$k$ isolates this rank-$r$ matrix, whose top-$r$ eigenspace is
exactly $\range(B_k^\star)$ once the innovation $w_t$ excites all
$r$ factor directions ($\Sigma_{\eta,k}\succ 0$). The three
identifiability conditions of Theorem~\ref{thm:identification} are
precisely what makes the isolation valid in the
unrestricted-second-moment scope; SPSC inherits subspace recovery at
rate $1/\sqrt{m_k}$ from Cor.~\ref{cor:projector_conf}, after which
the exploitation phase pays $\sqrt r$ rather than $\sqrt d$ because
the projection collapses the geometry to dimension $r$.

Full pseudocode for Alg.~\ref{alg:spsc} (oracle boundaries) is in
Appendix~\ref{app:algorithms}, alongside the adaptive variant
described next.

\paragraph{Adaptive variant (unknown boundaries).}
SPSC-Adaptive removes the oracle-boundary assumption by detecting
segment changes online: it maintains two non-overlapping rolling-window
estimators $\widehat M_t^{\mathrm{recent}}, \widehat M_t^{\mathrm{past}}$ of
the probe-time second moment and triggers a subspace reset whenever
$S_t := \|\widehat M_t^{\mathrm{recent}} - \widehat M_t^{\mathrm{past}}\|_\op$
exceeds a calibrated threshold $b$. The threshold is chosen so that,
with probability $1-\delta_{\mathrm{FA}}$, no false alarm occurs across
the horizon, while every true change of size $\Delta_k \ge 2b$ is
detected within a bounded delay $D_{\max}$ (specified in
Appendix~\ref{app:adaptive_algo} together with the full pseudocode,
tuning, and CUSUM analysis).

%% ============================================================
\section{Regret guarantees}
\label{sec:theory}

\paragraph{Balanced probe schedule.} The probe schedule is
\emph{balanced}: there is an absolute constant $a_0$ such that, for
every segment $k\in\{1,\dots,K\}$ and every prefix-count
$j\in\{0,1,\dots,m_k\}$, the number of exploit rounds in $E_k$ with
exactly $j$ probes collected before them is at most
$a_0\lceil\ell_k/m_k\rceil$. (Uniformly spaced probes satisfy this.)

\begin{theorem}[Costed dynamic regret of Alg.~\ref{alg:spsc}]
\label{thm:spsc_regret}
Run Alg.~\ref{alg:spsc} with oracle segment boundaries
$\{\cI_k\}_{k=1}^K$, scaled-sphere probes, and a balanced probe
schedule. Under the model of \S\ref{sec:setting},
Assumption~\ref{ass:projected_potential}, and exact centering
$\widehat\sigma^2=\sigma_\varepsilon^2$, set the per-segment probe
budget to $m_k=\min\{\ell_k,\lceil c_0\ell_k^{2/3}\rceil\}$ for an
absolute constant $c_0>0$ (chosen as the closed-form optimum
\eqref{eq:per_segment_opt}; segments shorter than the burn-in
threshold $q^\star=O(\log(KdT/\delta)/\lambda_{\min}^2)$ are charged
trivially), where $\lambda_{\min}>0$ is the predictable probe-time
$r$th-eigenvalue lower bound (Lem.~\ref{lem:probe_excitation_conf}).
Then with probability at least $1-\delta$ and for fixed $K$,
\begin{equation}
\DynReg_T^{(c)} \;\le\; \tilO\!\bigl(r\sqrt{T}\bigr) \;+\; \tilO\!\bigl(T^{2/3}\bigr) \;+\; O\!\bigl(W V_{\rm in}\bigr).
\label{eq:main_bound}
\end{equation}
Including all constants:
$\DynReg_T^{(c)} \le \tilO(r\sqrt{KT}) + \tilO(A^{1/3}B^{2/3}K^{1/3}T^{2/3}) + O(R_\cA W V_{\rm in}) + \tilO(K^{2/3}T^{1/3}/\lambda_{\min}^2)$,
with $A:=(R_\cA+R_Q)S_w+c$,
$B:=C R_\cA S_w(1+R_\cA\sqrt{W/\lambda})\sqrt{\log(2KdT/\delta)}$,
$R_Q:=\sup_{u\in\mathrm{supp}(Q)}\|u\|_2$ ($R_Q=\sqrt d$ for
scaled-sphere; $R_Q\le R_\cA$ if probes lie inside the action set);
the $\tilO(K^{2/3}T^{1/3}/\lambda_{\min}^2)$ summand is the burn-in
cost (lower order than $\tilO(T^{2/3})$ for fixed $K$). Here
$V_{\rm in}:=\sum_k\sum_{s,s+1\in\cI_k}\|\theta_{s+1}-\theta_s\|_2$
is the within-segment path variation and $W$ is the exploitation
window (cumulative on $E_k$ when $W\ge\ell_k$); the bound is most
informative when $V_{\rm in}$ is small. Proof in
\S\ref{app:proof_main}.
\end{theorem}

\paragraph{Proof sketch.}
Three near-decoupled error sources drive the bound. (i) Inside the
learned $r$-dimensional subspace, exploitation is a windowed
projected ridge-UCB problem whose self-normalized analysis yields
$\tilO(r\sqrt T)$ regret, with $\sqrt r$ entering through
$\beta_t^{(r,W)}$ rather than $\sqrt d$. (ii) Probe rounds incur
$O(m_k)$ regret per segment while delivering subspace error
$\varepsilon_k=\tilO(1/\sqrt{m_k})$ via Davis--Kahan applied to
$\widehat M_k$ (Cor.~\ref{cor:projector_conf}); the
exploitation-side mismatch cost $\propto \ell_k\,\varepsilon_k$
balances probe cost at $m_k\propto \ell_k^{2/3}$, summing to
$\tilO(T^{2/3})$ across segments. (iii) Within-segment drift
propagates through the windowed estimator at rate proportional to
$W V_{\rm in}$. Two ingredients make this decomposition rigorous:
uniform-in-$t$ control of the time-varying-basis subspace error
$\varepsilon_{k,t}=\|\widehat P_t-P_k^\star\|_\op$ via prefix
concentration (Lem.~\ref{lem:prefix_subspace}), and the
balanced-schedule condition that prevents any single segment from
being starved of probes. Detailed accounting, including the burn-in
$\tilO(K^{2/3}T^{1/3}/\lambda_{\min}^2)$ summand, is in
\S\ref{app:proof_main}.

\begin{corollary}[Plug-in variance]
\label{cor:plugin_variance}
Estimate $\widehat\sigma^2$ from $N$ probe-round residuals
(Lem.~\ref{lem:variance_estimator}) and write
$\delta_\sigma:=|\widehat\sigma^2-\sigma_\varepsilon^2|$. For
scaled-sphere probes, the bias
$\widetilde B=-(\delta_\sigma/d)I_d$ from
Lem.~\ref{lem:G_unbiased_conf} is a scaled identity, hence preserves
the eigenvectors of $\bar M_k^{\mathrm{probe}}$. Davis--Kahan applied
with the biased mean as reference removes the
$\Delta_\sigma$ floor at population level, so
Theorem~\ref{thm:spsc_regret} holds with the displayed bound,
$\lambda_{\min}$ replaced by $\lambda_{\min}-\delta_\sigma/d$ in the
constants. Choosing $N\gtrsim T^{2/3}/d^2$ yields
$\delta_\sigma/d\ll\lambda_{\min}$ with high probability, so the
leading rate is preserved.
\end{corollary}

\begin{corollary}[Within-segment-stationary special case]
\label{cor:stationary_within}
Suppose $V_{\rm in}=0$, i.e.\ within each segment $\cI_k$ the parameter
is constant ($\theta_t\equiv\theta^{(k)}$ for $t\in\cI_k$, with
$\theta^{(k)}\in\mathrm{range}(B^\star_k)$). The non-degenerate
innovation assumption ($\Sigma_{\eta,k}\succ 0$) still holds
\emph{across} the segment family, so the rank-$r$ subspace
$\mathrm{range}(B^\star_k)$ is well-defined and identifiable through
SPSC's probe channel. The bound of Theorem~\ref{thm:spsc_regret}
reduces to
\[
\DynReg_T^{(c)} \;\le\; \tilO(r\sqrt{T}) + \tilO(T^{2/3})
\]
for fixed $K$ (suppressing polylog factors). The leading
$\tilO(r\sqrt{T})$ term comes from LinUCB on the $r$-dimensional
projected target $z_t^\top a^{(k)}$ (this does not require the within-segment
target to itself span an $r$-dimensional set; the projector is what
collapses the geometry to dimension $r$, not the trajectory of
$\theta_t$). This is the regime that class- or cohort-induced segment
shifts approximate empirically.
\end{corollary}

\paragraph{Scope of the bound.} Theorem~\ref{thm:spsc_regret} analyzes
exactly the practical Alg.~\ref{alg:spsc}: probes are interleaved
according to a balanced schedule, $\widehat U_t$ is refreshed online,
and exploitation rounds re-project the windowed history through the
current basis. The time-varying-basis subspace error
$\varepsilon_{k,t}=\|\widehat P_t-P_k^\star\|_\op$ is controlled
uniformly in $t$ via a prefix concentration argument
(Lemma~\ref{lem:prefix_subspace}).

The leading $\tilO(r\sqrt{T})$ term beats the ambient baseline
$\tilO(d\sqrt{T})$ whenever $r \ll d$; the $\tilO(T^{2/3})$ term is
the price of acquiring the subspace information online. The constant
$B$ depends on the exploitation window $W$ as $\sqrt W$; for $W$
treated as a problem parameter (the setting throughout this paper,
consistent with $W\le\min_k\ell_k$ on $\ell_k$-bounded segments), this
is absorbed into the constants and the $\tilO(T^{2/3})$ rate is
preserved. If $W$ is allowed to scale with $T$, the second term
becomes $\tilO(W^{1/3}T^{2/3})$. Solving
$r\sqrt{T}\lesssim T^{2/3}$ gives a simple rule of thumb: SPSC tends
to outperform ambient LinUCB whenever $d-r\gtrsim T^{1/6}$. The
experiments in \S\ref{sec:experiments} show a crossover
qualitatively consistent with this rule of thumb.

\begin{corollary}[Rank-adaptive guarantee]
\label{cor:rank_adaptive}
If the true rank $r$ is unknown, thresholding $\widehat M_k$
at $\tau_k^{\rm rank} := 2R_X\sqrt{\log(2d/\delta)/m_k}$
($R_X$ the lifted-sample envelope from
App.~\ref{app:concentration_proof}) recovers the true rank with
probability $\ge 1-\delta$ whenever the population eigengap exceeds
$4\tau_k^{\rm rank}$, and the bound \eqref{eq:main_bound} applies
unchanged. Proof in \S\ref{app:proof_adaptive_rank}.
\end{corollary}

\begin{proposition}[Adaptive SPSC]
  \label{thm:spsc_adaptive}
  Suppose every true change satisfies $\Delta_k\ge 2b$ for the
  threshold $b$ set in Appendix~\ref{app:proof_adaptive}. Then with
  probability at least $1-\delta_{\mathrm{FA}}$, SPSC-Adaptive
  (Alg.~\ref{alg:spsc_adaptive}) attains the same leading rate as
  Theorem~\ref{thm:spsc_regret} with an additive $O(K W_{\mathrm{det}})$
  delay overhead. Proof in Appendix~\ref{app:proof_adaptive}.
  \end{proposition}

%% ============================================================
\section{Experiments}
\label{sec:experiments}

We evaluate SPSC and SPSC-Adaptive on eleven benchmarks: a synthetic
phase-transition grid; UCI/MovieLens (Covertype, Pendigits, Satimage,
MNIST, Fashion-MNIST, MovieLens); clinical (Warfarin, Vancomycin);
the small-$d$ piecewise-stationary stress test of
\citet{russac2019weighted} (\S\ref{sec:exp_russac}); and Open Bandit
production logs (\S\ref{sec:exp_openbandit}).\footnote{Code: \url{https://github.com/HamedKhosravi99/spsc-code}.} Baselines:
LinUCB~\citep{abbasi2011improved}, D-LinUCB~\citep{russac2019weighted},
SW-LinUCB~\citep{cheung2019learning}, Restart-LinUCB, LowRank-Reward,
LowOFUL~\citep{jun2019bilinear}, VOFUL,\footnote{LowOFUL, VOFUL, and
LowRank-Reward are adaptations of stationary low-rank methods to our
piecewise-low-rank setting; see App.~\ref{app:setup} for details.}
LinTS, SW-LinTS, and adapted BOSS~\citep{duong2024beyond}/Jedra~\citep{jedra2024low};
Oracle-LinUCB (true subspace) is an optimistic reference. Curves are
means over $\ge 10$ seeds with $\pm 1$\,SE bands.

\subsection{Phase-transition boundary}
\label{sec:exp_phase}

Theory predicts a crossover at $d - r \asymp T^{1/6}$: SPSC wins above
it, ambient LinUCB below. We verify on a piecewise low-rank LDS at
$T{=}5{,}000$, sweeping $(d,r)\in\{5,10,20,30,45,60,80,100\}\times\{1,3,5,10,15,20\}$
with $r{<}d$ (40 cells), $K{=}10$, $\sigma_\varepsilon{=}0.3$,
spectral radius $0.99$, $40$ actions, probe cost $c{=}0.1$. Per-cell
mean regrets and verdicts are tabulated in App.~\ref{app:percell}
(Table~\ref{tab:app-synthetic}).

\begin{figure}[!htbp]
  \centering
  \includegraphics[width=0.98\textwidth]{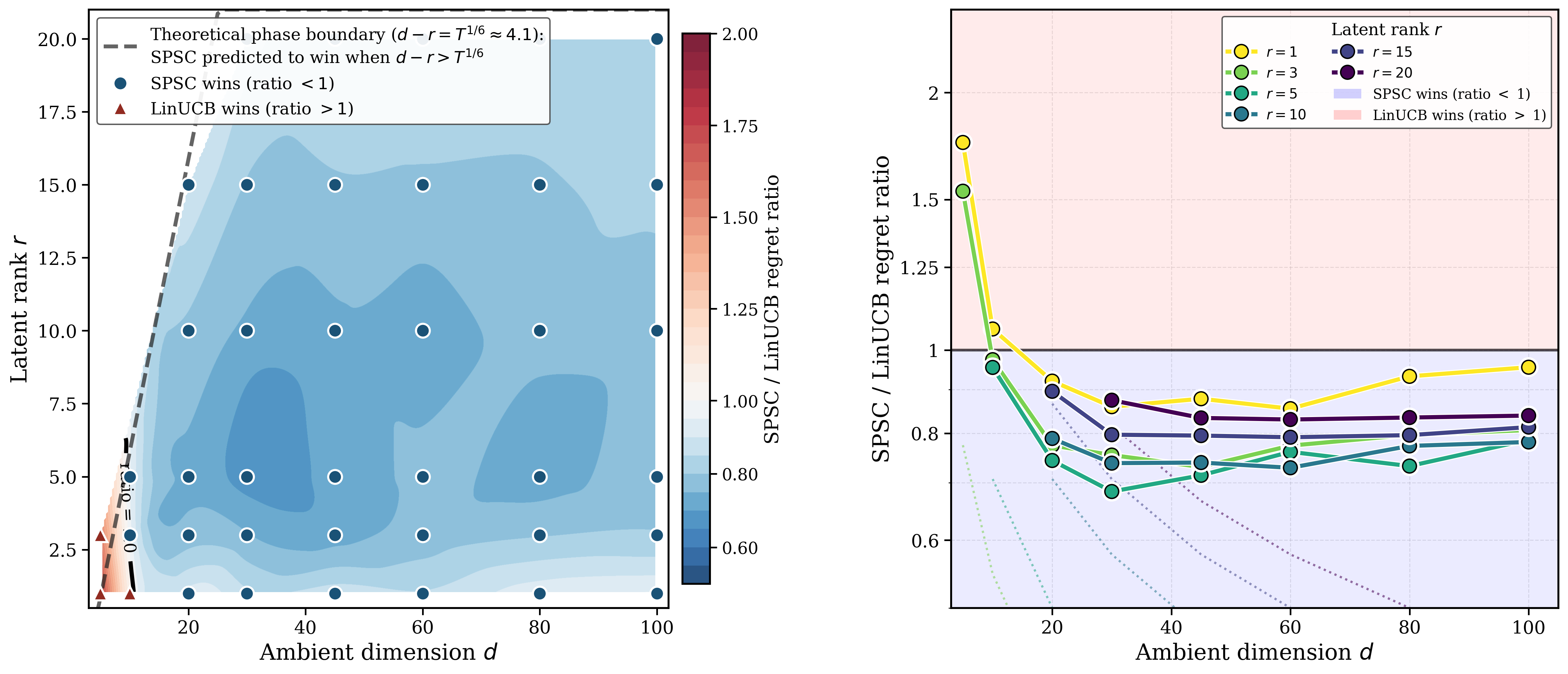}
  \caption{\textbf{Phase-transition boundary.}
  \textbf{(a)} Empirical ratio${=}1$ contour (solid) closely tracks
  the theoretical crossover $d-r=T^{1/6}$ (dashed); markers indicate
  SPSC wins (blue circles) and LinUCB wins (red triangles).
  \textbf{(b)} Ratio versus $d$ for each $r$ (log scale): empirical
  ratios (solid) match the $\sqrt{r/d}$ rate predicted by
  Theorem~\ref{thm:spsc_regret} (dotted).}
  \label{fig:synthetic-phase}
\end{figure}

SPSC achieves $16$--$29\%$ reductions whenever $d\ge 45$ and
$r\ge 3$ with every $r$-curve tracking $\sqrt{r/d}$: use SPSC when
$d-r\gtrsim T^{1/6}$, and ambient LinUCB otherwise.

\subsection{Real-data multi-baseline comparison}
\label{sec:exp_real}

Six real-data benchmarks recast as piecewise low-rank bandits on a
$(d,r)$-grid: UCI \textbf{Covertype} ($K{=}4$, $T{=}10{,}000$);
\textbf{Pendigits}, \textbf{Satimage}, pixelized \textbf{MNIST} and
\textbf{Fashion-MNIST} (random-projection features), and
\textbf{MovieLens-100K} with genuine user ratings (all $K{=}10$,
$T{=}5{,}000$, 10 seeds). Segments are class or user-subset shifts.
The first five have cluster-induced low-rank structure; MovieLens
is only approximately low-rank (rank-$d$ SVD features,
\S\ref{app:setup}) and serves as a robustness stress test.
Table~\ref{tab:realdata} reports SPSC vs.\ LinUCB and vs.\ best
non-oracle competitor on representative cells; per-cell tables in
Appendix~\ref{app:percell}.

\begin{table}[!htbp]
\centering
\caption{\textbf{SPSC/LinUCB} and \textbf{SPSC-Adaptive/Best
  Competitor} ratios across six real-data benchmarks (10 seeds per
  cell). Values ${<}1$ indicate SPSC wins; \textbf{bold} marks cells
  where an SPSC variant beats every non-oracle baseline. SPSC trails
  LinUCB at $d\approx r$ (top of each block) and gains as $d$ grows
  relative to $r$. Restart-LinUCB, the closest competitor, is
  subspace-oblivious, so its regret is constant along each row.}
\label{tab:realdata}
\scriptsize
\setlength{\tabcolsep}{4pt}
\begin{tabular}{lcc@{\hspace{8pt}}cc@{\hspace{8pt}}cc@{\hspace{8pt}}cc}
\toprule
 & \multicolumn{2}{c}{$d{=}55$} & \multicolumn{2}{c}{$d{=}105$} & \multicolumn{2}{c}{$d{=}200$} & \multicolumn{2}{c}{best}\\
\cmidrule(lr){2-3}\cmidrule(lr){4-5}\cmidrule(lr){6-7}\cmidrule(lr){8-9}
Dataset / $r$ & S/Lin & S/Best & S/Lin & S/Best & S/Lin & S/Best & cell & ratio\\
\midrule
Covertype,      $r{=}10$ & 0.96           & 0.99           & \textbf{0.88} & \textbf{0.92} &  --- & ---  & $(105,20)$ & \textbf{0.86}\\
Pendigits,      $r{=}10$ & \textbf{0.76}  & \textbf{0.77}  & \textbf{0.65} & \textbf{0.73} &  --- & ---  & $(105,10)$ & \textbf{0.65}\\
Satimage,       $r{=}10$ & \textbf{0.47}  & \textbf{0.59}  & \textbf{0.40} & \textbf{0.58} &  --- & ---  & $(105,10)$ & \textbf{0.40}\\
MNIST,          $r{=}10$ & \textbf{0.84}  & \textbf{0.86}  & \textbf{0.74} & \textbf{0.88} & \textbf{0.73}/\textbf{0.98} & \textbf{0.98} & $(105,20)$ & \textbf{0.74}\\
Fashion-MNIST,  $r{=}10$ & \textbf{0.88}  & \textbf{0.90}  & \textbf{0.73} & \textbf{0.88} & \textbf{0.67}/0.99         & 0.99          & $(105,20)$ & \textbf{0.71}\\
MovieLens,      $r{=}10$ & 1.00           & 1.00           & \textbf{0.97} & \textbf{0.97} & \textbf{0.92}/\textbf{0.93} & \textbf{0.93} & $(200,5)$  & \textbf{0.92}\\
\bottomrule
\end{tabular}
\end{table}

At $(d{=}105, r{=}10)$, SPSC-Adaptive beats LinUCB by $-47\%$ on
Pendigits (operating-regime grid in Fig.~\ref{fig:pendigits-regime},
App.~\ref{app:percell}) and $-68\%$ on Satimage (closest competitor
VOFUL is $+112\%$ above SPSC-Adaptive; per-cell results in
Table~\ref{tab:app-satimage}, App.~\ref{app:percell}); at
$(d{=}105, r{=}20)$ on Covertype it beats Restart-LinUCB by $-15\%$.
Non-low-rank baselines are by construction $r$-insensitive.

\subsection{Clinical benchmarks: Warfarin and Vancomycin dosing}
\label{sec:exp_warfarin}

\paragraph{Warfarin.} Warfarin dosing depends on demographics,
genotype (VKORC1, CYP2C9), and clinical
indicators~\citep{international2009estimation}. We calibrate to the
IWPC cohort (${\sim}5{,}700$ patients, $d{=}93$, $K{=}8$ segments,
$T{=}5{,}000$, 10 seeds; reward $=$ negative deviation from
therapeutic dose). Table~\ref{tab:warfarin_body} ($r{=}3$):
SPSC-Adaptive cuts regret by $67.3\%$ vs.\ LinUCB; full per-rank
tables in Table~\ref{tab:app-warfarin} (App.~\ref{app:percell});
random-subspace ablation in App.~\ref{app:warfarin_ablation}.

\paragraph{Vancomycin.} With the AUC-targeted
\citet{rybak2020therapeutic} protocol (Cockcroft--Gault
$\mathrm{CrCl}\to\mathrm{CL}_{\mathrm{vanco}}=0.04\,\mathrm{CrCl}$,
daily dose at $\mathrm{AUC}_{24}{=}500$; $d{=}93$, $K{=}8$,
$T{=}5{,}000$, 10 seeds), Table~\ref{tab:vancomycin_body} sweeps
$r\in\{1,2,3,5,10\}$: SPSC-Adaptive cuts regret by $36.4\%$ at $r{=}1$
and $14.5$--$25.7\%$ for $r\in\{2,3,5,10\}$ vs.\ LinUCB, beating every
non-oracle baseline at every rank. Per-rank tables in
Appendix~\ref{app:vancomycin_full}.

\begin{table}[t]
\centering
\caption{\textbf{Warfarin} (representative cell $r{=}3$, mean$\pm$SE,
10 seeds). SPSC-Adaptive cuts regret by $67.3\%$ vs.\ LinUCB. Full
per-rank tables in Appendix~\ref{app:percell}.}
\label{tab:warfarin_body}
\small
\begin{tabular}{lrr}
\toprule
Method & Costed regret & vs.\ LinUCB \\
\midrule
Oracle LinUCB                  & $177\pm 4$   & $-91.2\%$\\
\textbf{SPSC Alg.\,1 (ours)}   & $1372\pm 33$ & $-31.6\%$\\
\textbf{SPSC-Adaptive (ours)}  & $\mathbf{657\pm 39}$ & $\mathbf{-67.3\%}$\\
LowOFUL                        & $683\pm 97$  & $-66.0\%$\\
VOFUL                          & $779\pm 93$  & $-61.2\%$\\
LowRank-Reward                 & $1056\pm 40$ & $-47.4\%$\\
SW-LinUCB                      & $2096\pm 18$ & $+4.5\%$\\
LinUCB                         & $2006\pm 17$ & ---\\
\bottomrule
\end{tabular}
\end{table}

\begin{table}[!htbp]
\centering
\caption{\textbf{Vancomycin} (mean$\pm$SE, 10 seeds): SPSC-Adaptive
control regret across range of $r$ vs.\ LinUCB and the best
non-oracle baseline. Full tables in
Appendix~\ref{app:vancomycin_full}.}
\label{tab:vancomycin_body}
\small
\setlength{\tabcolsep}{6pt}
\begin{tabular}{cccc}
\toprule
$r$ & SPSC-Adaptive  & vs.\ LinUCB & vs.\ best comp.\\
\midrule
1   & $\mathbf{525.3\pm 45.8}$ & $-36.4\%$ & $-35.1\%$ \\
2   & $\mathbf{613.6\pm 33.1}$ & $-25.7\%$ & $-17.1\%$ \\
3   & $\mathbf{640.3\pm 33.0}$ & $-22.4\%$ & $-10.1\%$ \\
5   & $\mathbf{665.9\pm 27.2}$ & $-19.3\%$ & $-12.7\%$ \\
10  & $\mathbf{705.6\pm 19.5}$ & $-14.5\%$ & $\phantom{-}-7.0\%$ \\
\bottomrule
\end{tabular}
\end{table}

\subsection{Small-$d$ piecewise-stationary stress test}
\label{sec:exp_russac}

On the small-$d$ stress test of \citet{russac2019weighted}
($d{=}2$, $r{=}1$, $K{=}4$, $T{=}6{,}000$, 50 arms, 30 seeds), SPSC
reduces regret by $\mathbf{-46.5\%}$ vs.\ D-LinUCB,
$\mathbf{-34.8\%}$ vs.\ SW-LinUCB, and $\mathbf{-52.0\%}$ vs.\
stationary OFUL (itself
$+11.5\%$ above D-LinUCB here). Full table in
Appendix~\ref{app:sota_compare}.

\subsection{Real exploration logs: Open Bandit Pipeline (ZOZO)}
\label{sec:exp_openbandit}

The Open Bandit Dataset~\citep{saito2020open} ships production
exploration logs (random-policy ZOZOTOWN slice, $T{=}5{,}000$,
$K{=}10$, 10 seeds); only approximately low-rank, so this is a
\emph{robustness} stress test rather than a setting where the rank-$r$
assumption holds exactly. At $d{=}55$, $r{=}5$, SPSC reaches
$\mathbf{83\pm 2}$ vs.\ LinUCB's $102\pm 1$ ($-18.8\%$); SPSC-Adaptive
$-10.9\%$. SPSC neither breaks nor over-claims on weakly-low-rank
production data. Per-cell results in Table~\ref{tab:app-openbandit}
(App.~\ref{app:percell}).

\subsection{Comparison with stationary low-rank methods (oracle restarts)}
\label{sec:exp_boss_jedra}

For a best-effort comparison with stationary low-rank methods, we
adapt BOSS~\citep{duong2024beyond} and Jedra~\citep{jedra2024low} by
restarting at every true segment boundary, an oracle advantage we do
not give SPSC (setup in Appendix~\ref{app:boss_jedra_detail}).
Table~\ref{tab:boss_jedra_grid}: SPSC wins all eight cells by
$5.8$--$19.9\%$, with the lead largest at small $d$ ($-16.3$ to
$-19.9\%$ at $d{=}55$) and narrowing to $5.8$--$8.0\%$ at $d{=}200$.
SPSC-Adaptive (no oracle) also beats both on every cell.

\begin{table}[!htbp]
\centering
\caption{\textbf{Adapted stationary low-rank methods on a $(d,r)$-grid}.
Costed regret (mean$\pm$SE, 10 seeds, $T{=}5{,}000$, $K{=}10$). BOSS/Jedra restart at \emph{oracle} segment boundaries;
SPSC-Adaptive detects them online. SPSC variants win every cell.}
\label{tab:boss_jedra_grid}
\small
\setlength{\tabcolsep}{4pt}
\begin{tabular}{ll cccc r}
\toprule
$d$ & $r$ & SPSC Alg.\,1 & SPSC-Adaptive & BOSS (oracle) & Jedra (oracle) & SPSC vs.\ best comp.\\
\midrule
55  & 5   & $\mathbf{323\pm10}$ & $342\pm15$         & $399\pm17$         & $386\pm18$         & $-16.3\%$ \\
55  & 10  & $\mathbf{471\pm10}$ & $507\pm13$         & $588\pm27$         & $602\pm24$         & $-19.9\%$ \\
105 & 5   & $\mathbf{264\pm 9}$ & $274\pm10$         & $300\pm14$         & $299\pm13$         & $-11.7\%$ \\
105 & 10  & $\mathbf{377\pm15}$ & $401\pm14$         & $423\pm18$         & $429\pm14$         & $-10.9\%$ \\
105 & 20  & $\mathbf{578\pm12}$ & $608\pm14$         & $654\pm17$         & $654\pm14$         & $-11.6\%$ \\
200 & 5   & $\mathbf{208\pm 5}$ & $218\pm 6$         & $236\pm 9$         & $226\pm 9$         & $\phantom{-}-8.0\%$ \\
200 & 10  & $\mathbf{277\pm10}$ & $292\pm12$         & $294\pm11$         & $305\pm15$         & $\phantom{-}-5.8\%$ \\
200 & 20  & $\mathbf{442\pm13}$ & $452\pm11$         & $473\pm16$         & $489\pm18$         & $\phantom{-}-6.6\%$ \\
\bottomrule
\end{tabular}
\end{table}

\paragraph{Robustness of SPSC-Adaptive.} Detailed comparison under
correct vs.\ misspecified oracle boundaries is in
Appendix~\ref{app:adaptive_winner}: SPSC-Adaptive is robust to
false-alarm boundaries, while Alg.~\ref{alg:spsc} degrades
monotonically as the oracle is corrupted.

\subsection{Rank misspecification, probe-rate sensitivity, and take-away}
\label{sec:exp_sensitivity}

\textbf{Rank misspecification.} The true rank $r^\star$ is rarely
known in advance, so we evaluate SPSC's robustness to rank
mis-specification on Covertype ($d{=}155$, $r^\star{=}10$;
App.~\ref{app:rank_misspec}): sweeping over nine grid points
$r\in\{1,3,5,10,15,20,30,50,80\}$ gives a U-shaped curve (under:
$+33\%$ at $r{=}1$; sweet spot $r{=}30$, $-10\%$); SPSC beats LinUCB
across $r\in[15,50]$, so mild overestimation is safe.
\textbf{Probe rate.} The probe period $m$ trades exploration cost
against subspace-estimation accuracy, so the optimal $m$ depends on
segment length; we test whether the theoretical optimum
$m_k^\star\propto \ell_k^{2/3}$ holds empirically
(App.~\ref{app:probe_ablation}). On the synthetic reference setting
($d{=}4$, $r{=}1$), the probe-period sweep
$\{5,10,20,30,50,100,300\}$ attains its optimum near $20$--$30$,
matching the prediction.
\textbf{Additional datasets.} To verify that the per-cell advantage
generalizes beyond a single benchmark, the full $(d,r)$-grid is
reported in Appendix~\ref{app:percell}: SPSC wins every cell at
$d\ge 105$, peaking at $(d,r){=}(105,20)$.
\textbf{Sensitivity \& dynamics.} To assess whether any of the four
working assumptions (known variance, exact rank, sub-unit spectral
radius, regular change points) is critical to SPSC's performance,
we ablate each separately. Large-scale assumption sweeps
(App.~\ref{app:assumption_violation}) confirm SPSC tracks Oracle
within a constant factor through variance misspecification,
approximate-rank perturbation ($\epsilon_k^\perp$ up to $0.4$), and
spectral radius up to $\rho{=}0.95$; SPSC remains the best non-Oracle
method at $\rho{=}0.99$. Subspace error tracks the predicted
$1/\sqrt m$ rate (App.~\ref{app:subspace});
noise/change-point-frequency/drift-speed sweeps
(App.~\ref{app:noise_cp_drift}) confirm the same trends.
\textbf{Take-away.} Across the eleven benchmarks, SPSC delivers
$\sqrt{r/d}$ intrinsic-rank savings whenever $d-r\gtrsim T^{1/6}$
and does not catastrophically fail outside that regime.

%% ============================================================
\section{Conclusion}
\label{sec:conclusion}

Sequential decision problems in recommendation, clinical dosing,
and adaptive experimentation share two features: rewards live on a
low-rank subspace of rank $r\ll d$, and that subspace shifts at
unknown change points. Stationary low-rank, nonstationary, and
cost-aware bandits each cover one piece; none recover a changing
subspace under priced single-play scalar feedback. We close this
gap: three probe-side conditions, known noise variance, bounded
state-noise coupling, and full-dimensional probe support, are each
individually necessary (in the unrestricted-second-moment scope)
and jointly sufficient for subspace recovery from quadratic
functionals of scalar rewards (Theorem~\ref{thm:identification}). Based on this, we propose SPSC,
which interleaves probing with windowed projected ridge-UCB in the
learned $r$-dimensional subspace and achieves costed dynamic regret
$\tilO(r\sqrt T)+\tilO(T^{2/3})+O(WV_{\rm in})$, replacing the
ambient $d\sqrt T$ rate with the intrinsic $r$
(Theorem~\ref{thm:spsc_regret}); a CUSUM adaptive variant attains
the same rate without oracle change points
(Proposition~\ref{thm:spsc_adaptive}). On eleven benchmarks across
synthetic, UCI, clinical, and production data, SPSC reduces regret
whenever $d-r\gtrsim T^{1/6}$, a crossover consistent with theory.
Limitations and extensions are discussed in App.~\ref{app:limitations}.

%% ============================================================
\bibliographystyle{plainnat}
\bibliography{references}

%% ============================================================
%% APPENDIX
%% ============================================================
\clearpage
\appendix

\section*{Appendix: Table of Contents}
\label{app:toc}

\begingroup
\hypersetup{linkcolor=blue}

\newcommand{\tocentry}[3]{%
  \noindent\textbf{#1.}\;\hyperref[#2]{#3}\dotfill\hyperref[#2]{p.\,\pageref*{#2}}\par\vspace{2pt}%
}
\newcommand{\tocsubentry}[3]{%
  \noindent\hspace{1.2em}#1\;\hyperref[#2]{#3}\dotfill\hyperref[#2]{p.\,\pageref*{#2}}\par\vspace{1pt}%
}

\tocentry{A}{app:limitations}{Limitations and extensions}
\tocentry{B}{app:related_table}{Comparison with prior work}
\tocentry{C}{app:proofs}{Proofs}
  \tocsubentry{C.1}{app:prelim}{Preliminaries and notation}
  \tocsubentry{C.2}{app:concentration_proof}{Concentration of the lifted probe moment}
  \tocsubentry{C.3}{app:necessity}{Necessity of the probe-side conditions}
  \tocsubentry{C.4}{app:proof_main_lemmas}{Lemmas for the interleaved analysis}
  \tocsubentry{C.5}{app:proof_main}{Proof of Theorem~\ref{thm:spsc_regret} (costed dynamic regret of Alg.~\ref{alg:spsc})}
  \tocsubentry{C.6}{app:proof_adaptive_rank}{Proof of Corollary~\ref{cor:rank_adaptive} (rank-adaptive)}
  \tocsubentry{C.7}{app:proof_adaptive}{Proof of Proposition~\ref{thm:spsc_adaptive} (SPSC-Adaptive)}
\tocentry{D}{app:truncation}{Probe distributions: scaled sphere and Gaussian truncation}
\tocentry{E}{app:algorithms}{Algorithms: pseudocode and adaptive variant}
\tocentry{F}{app:setup}{Experimental setup}
\tocentry{G}{app:percell}{Per-cell tables: every dataset, every cell}
\tocentry{H}{app:sensitivity}{Sensitivity and robustness studies}
\tocentry{I}{app:sota_compare}{Nonstationary baselines (small-$d$ stress test)}
\tocentry{J}{app:warfarin_ablation}{Warfarin: random-subspace ablation}
\tocentry{K}{app:boss_jedra_detail}{BOSS/Jedra adaptation}
\tocentry{L}{app:vancomycin_full}{Vancomycin: full per-rank tables}
\tocentry{M}{app:adaptive_winner}{When the adaptive variant wins}

\endgroup

\section{Limitations and extensions}
\label{app:limitations}

Three boundaries of the present results are worth flagging. The
identification theorem is tight in the unrestricted-second-moment
scope; restricting to specific noise families could relax one or
more of the three conditions but would also forfeit the converse
half. The $O(W V_{\rm in})$ term in the regret bound is informative
only when within-segment path variation is small, and the CUSUM-style
detector in SPSC-Adaptive trades detection delay for false-alarm
control, so very frequent change points require the sliding-window
estimator to be re-tuned. Two natural extensions stand out.
Bilinear or tensorized factor models with side information on $w_t$
should inherit the probe--exploit decomposition with the
segment-wise $\tilO(T^{2/3})$ identification cost replaced by the
corresponding higher-order moment-estimation rate.
Heterogeneous-rank segments, where $r$ itself shifts at change
points, strictly generalize Cor.~\ref{cor:rank_adaptive}, which
covers a single rank-thresholding pass and assumes a common rank
across segments.

\section{Comparison with prior work}
\label{app:related_table}

This appendix positions SPSC against published work in low-rank
bandits, nonstationary bandits, and active-query bandits via three
tables: Table~\ref{tab:related_work_summary} (head-to-head with seven
recent works on regret form, change-point adaptation, and probing
budget), Table~\ref{tab:related_work_main} (compact positioning of
SPSC's combined setting against the three closest \emph{problem
classes}), and Table~\ref{tab:published_comparison} (regret-rate
comparison with five published baselines that appear in our
experiments). The take-away is that the existing literature splits
along three axes (low-rank, nonstationary, costly probing) and no
prior method addresses all three; SPSC is the first to do so with a
self-contained identification chain that turns rank-one scalar
rewards into matrix-valued evidence.

\begin{table*}[!htbp]
\caption{Comparison with related literature. ``Prob.'' indicates a
probabilistic regret bound; ``CP?'' indicates change-point
adaptation.}
\centering
\footnotesize
\setlength{\tabcolsep}{3pt}
\renewcommand{\arraystretch}{1.1}
\begin{tabularx}{\textwidth}{
>{\raggedright\arraybackslash}p{2.65cm}
>{\raggedright\arraybackslash}p{1.95cm}
>{\raggedright\arraybackslash}X
>{\centering\arraybackslash}p{0.68cm}
>{\centering\arraybackslash}p{0.68cm}
>{\raggedright\arraybackslash}X
}
\toprule
\textbf{Reference} &
\makecell[l]{\textbf{Venue/}\\\textbf{Year}} &
\textbf{Regret / Theory} &
\makecell[c]{\textbf{Prob.}} &
\textbf{CP?} &
\makecell[l]{\textbf{Difference}\\\textbf{vs.\ ours}} \\
\midrule
\citet{kang2022efficient} & NeurIPS'22 &
$\tilO\!\left(\sqrt{(d_1+d_2)MrT}\right)$
& \cmark & \xmark & Stationary; probing not budgeted. \\
\citet{stojanovic2023spectral} & NeurIPS'23 &
Entry-wise / subspace recovery
& \xmark/\cmark & \xmark & Estimation focus; no probe budget. \\
\citet{jedra2024low} & ICML'24 &
$\tilO\!\left(r^{5/4}(m+n)^{3/4}\sqrt{T}\right)$
& \cmark & \xmark & Stationary; no re-learning. \\
\citet{jang2024efficient} & ICML'24 &
Arm-set-dependent regret
& \cmark & \xmark & Stationary; no CP detection. \\
\citet{abbasi2023new} & JMLR'23 &
$\tilO(\sqrt{KN(S+1)})$
& \xmark & \cmark & Unstructured $K$-arm; no low-rank. \\
\citet{komiyama2024finite} & JMLR'24 &
ADR-bandit finite-time
& \xmark & \cmark & Unstructured; no low-rank + probing. \\
\citet{hou2024almost} & NeurIPS'24 &
Fixed-confidence sample complexity
& \cmark & \cmark & BAI objective; no low-rank. \\
\bottomrule
\end{tabularx}
\label{tab:related_work_summary}
\end{table*}

\begin{table}[!htbp]
\caption{Compact positioning of SPSC's combined setting.}
\centering
\footnotesize
\setlength{\tabcolsep}{3pt}
\renewcommand{\arraystretch}{1.08}
\begin{tabularx}{\linewidth}{
>{\raggedright\arraybackslash}p{2.45cm}
>{\centering\arraybackslash}p{0.72cm}
>{\centering\arraybackslash}p{2.02cm}
>{\centering\arraybackslash}p{1.02cm}
>{\raggedright\arraybackslash}X
}
\toprule
\textbf{Work class} & \textbf{Low-rank} & \textbf{Nonstationary} & \textbf{Costly probing} & \textbf{Main gap vs.\ ours} \\
\midrule
Low-rank bandits
& \cmark & \xmark & \xmark
& Fixed subspace; no piecewise changes; no sensing-cost budget. \\
Nonstationary linear/contextual bandits
& \xmark & \cmark & \xmark
& Adapt in ambient parameter space; ignore latent low-rank structure. \\
Costly-information / active-query bandits
& \xmark & \xmark & \cmark
& Price information acquisition; do not address subspace recovery
under changing representations. \\
\textbf{This paper}
& \cmark & \cmark & \cmark
& \textbf{Combines low-rank structure, piecewise-stationary adaptation,
and explicit costly probing in a single identification-to-exploitation
framework.} \\
\bottomrule
\end{tabularx}
\label{tab:related_work_main}
\end{table}

\begin{table}[!htbp]
\caption{Regret-scaling comparison with the published baselines that
appear in our experiments.}
\centering
\footnotesize
\setlength{\tabcolsep}{4pt}
\begin{tabularx}{\linewidth}{
>{\raggedright\arraybackslash}p{1.95cm}
>{\raggedright\arraybackslash}p{2.45cm}
>{\centering\arraybackslash}p{3.6cm}
>{\raggedright\arraybackslash}X
}
\toprule
Algorithm & Reference & Regret scaling & Key difference vs.\ SPSC \\
\midrule
LinUCB       & \citet{abbasi2011improved}  & $O(d\sqrt{T})$                       & Ambient; no subspace exploitation. \\
LowOFUL      & \citet{jun2019bilinear}     & $\tilO(r\sqrt{T})$                   & Needs oracle subspace $B$. \\
LowRank-Reward & \citet{lu2021low}     & $\tilO((d_1{+}d_2)^{3/2}\sqrt{rT})$  & Stationary; ambient prefactor in $(d_1{+}d_2)^{3/2}$. \\
BOSS         & \citet{duong2024beyond}       & $\tilO(r\sqrt{dT})$                  & Stationary; no online change adaptation. \\
Jedra        & \citet{jedra2024low}  & $\tilO(r^{5/4}(m{+}n)^{3/4}\sqrt{T})$ & Stationary; no probe cost. \\
\midrule
\textbf{SPSC (ours)} & This work          & $\tilO(r\sqrt{T}+T^{2/3})$           & Learns a \emph{changing} $B$ via priced probes. \\
\bottomrule
\end{tabularx}
\label{tab:published_comparison}
\end{table}

\section{Proofs}
\label{app:proofs}

This section proves the statements of \S\ref{sec:theory} in
dependency order:
\begin{enumerate}[leftmargin=*,itemsep=1pt,topsep=2pt]
  \item \S\ref{app:prelim} fixes notation and establishes the
    probe-time excitation lemma that all proofs depend on.
  \item \S\ref{app:concentration_proof} derives the load-bearing
    \emph{identification chain}: closed-form inverse of the probe
    operator $\mathcal K$, quadratic-measurement identity, lifted
    probe sample's near-unbiasedness, matrix Bernstein/Freedman, and
    the projector error via Davis--Kahan.
  \item \S\ref{app:necessity} establishes the necessity half of
    Theorem~\ref{thm:identification} via three propositions
    (Prop.~\ref{prop:nonid_a}--\ref{prop:nonid_c}) showing each
    probe-side condition is individually necessary in the
    unrestricted-second-moment problem.
  \item \S\ref{app:proof_main_lemmas} states and proves the
    interleaved-analysis lemmas: projected nonstationary UCB with
    time-varying basis (Lem.~\ref{lem:F2pp}), uniform prefix
    subspace concentration (Lem.~\ref{lem:prefix_subspace}), and
    cumulative interleaved subspace error
    (Lem.~\ref{lem:cumulative_interleaved_error}).
  \item \S\ref{app:proof_main} proves the main costed regret bound
    (Thm.~\ref{thm:spsc_regret}) via a five-step decomposition:
    probe regret, exploitation regret via Lem.~\ref{lem:F2pp},
    per-segment bound, probe-estimation tradeoff with burn-in, and
    sum over segments.
  \item \S\ref{app:proof_adaptive_rank} derives the rank-adaptive
    corollary (Cor.~\ref{cor:rank_adaptive}) via Weyl's inequality on
    the matrix-Bernstein bound.
  \item \S\ref{app:proof_adaptive} lifts the guarantee to
    SPSC-Adaptive (Prop.~\ref{thm:spsc_adaptive}) by controlling the
    CUSUM detector.
\end{enumerate}
The windowed projected self-normalized bound
(Lem.~\ref{lem:self_norm}), the windowed projected potential
condition (Assumption~\ref{ass:projected_potential}), and the
projected nonstationary UCB lemma (Lem.~\ref{lem:F2pp}) are the
linear-bandit ingredients; constants that do not affect leading
order are absorbed into $\tilO(\cdot)$.

\subsection{Preliminaries and notation}
\label{app:prelim}

Throughout, $\cI_k = [\tau_{k-1}, \tau_k)$ denotes segment $k$ of
length $\ell_k := |\cI_k|$, so $\sum_k \ell_k = T$. Inside $\cI_k$ the
learner performs $m_k = |\cT_k|$ probe rounds ($\cT_k := \cTprobe\cap\cI_k$)
and $n_k = \ell_k - m_k$ exploitation rounds ($E_k := \cI_k\setminus\cT_k$).
We write $\widehat P_t = \widehat U_t \widehat U_t^\top$ and
$P_k^\star = B_k^\star B_k^{\star\top}$ for the estimated and true rank-$r$
projectors, and $\varepsilon_k := \|\widehat P_k - P_k^\star\|_{\op}$.
We work on the high-probability event
$\mathcal G_w := \bigl\{\max_{t\le T}\|w_t\|_2\le S_w\bigr\}$. For sub-Gaussian innovations
one may take $S_w = \tilO(\sqrt r)$ via union bound (the failure
probability is absorbed into the overall $\delta$ budget); for bounded innovations $\mathcal G_w$ holds deterministically. In what
follows we implicitly condition on $\mathcal G_w$.

The \emph{predictable second moment} is
$\widetilde M_t := \E[\theta_t\theta_t^\top\mid\cH_{t-1}]
= B_k^\star \E[w_tw_t^\top\mid\cH_{t-1}](B_k^\star)^\top$, and the
probe-time average is
$\bar M_k^{\mathrm{probe}} := m_k^{-1}\sum_{t\in\cT_k}\widetilde M_t$.
By the non-degenerate-innovation assumption $\Sigma_{\eta,k}\succ 0$ of
\S\ref{sec:setting}, $\bar M_k^{\mathrm{probe}}$ has rank $r$ and its
$r$-th eigenvalue $\lambda_{\min}>0$ is uniformly bounded below in
$k$ (Lem.~\ref{lem:probe_excitation_conf} below; stability of $A_k$
is used only to bound $\|w_t\|\le S_w$, not for this lower bound).
Constants $C,C',\ldots$ depend only on
$\sigma_\varepsilon, R_\cA, L, \rho_Q, S_w, \lambda_{\min}$ and may
change line to line.

\begin{lemma}[Predictable probe-time excitation]
\label{lem:probe_excitation_conf}
For each segment $k$,
\[
\lambda_r\!\left(\frac{1}{m_k}\sum_{t\in\cT_k}\E[w_tw_t^\top\mid\cH_{t-1}]\right)
\;\ge\; \lambda_{\min}\;>\;0
\qquad\text{almost surely.}
\]
\end{lemma}
\begin{proof}
For each $t\in\cI_k$, the conditional second moment satisfies
$\E[w_t w_t^\top\mid\cH_{t-1}]
= A_k w_{t-1}w_{t-1}^\top A_k^\top + \Sigma_{\eta,k}
\succeq \Sigma_{\eta,k}$
(LDS recursion + non-degenerate innovations). Hence
$\lambda_r(\E[w_tw_t^\top\mid\cH_{t-1}]) \ge \lambda_r(\Sigma_{\eta,k}) =: \lambda_{\min} > 0$
uniformly across rounds. Averaging over any probe subset
$\cT_k\subseteq\cI_k$ preserves this PSD lower bound:
$\lambda_r(m_k^{-1}\sum_{t\in\cT_k}\E[w_tw_t^\top\mid\cH_{t-1}])\ge\lambda_{\min}$
almost surely. Stability of $A_k$ is therefore not needed for the
lemma; the innovation covariance alone suffices.
\end{proof}

\paragraph{Notation for the interleaved analysis.}
For each segment $\cI_k=[\tau_{k-1},\tau_k)$, recall
$\ell_k=|\cI_k|$, $\cT_k=\cTprobe\cap\cI_k$, $m_k=|\cT_k|$,
$E_k=\cI_k\setminus\cT_k$, $n_k=|E_k|=\ell_k-m_k$. For
$t\in\cI_k$, let $q_k(t):=|\cT_k\cap[\tau_{k-1},t)|$ be the number
of probes collected in segment $k$ strictly before round $t$.
Denote $P_k^\star=B_k^\star B_k^{\star\top}$,
$\widehat P_t=\widehat U_t\widehat U_t^\top$, and
$\varepsilon_{k,t}:=\|\widehat P_t-P_k^\star\|_{\op}$. The
within-segment variation is
$V_k^{\rm in}:=\sum_{s,s+1\in\cI_k}\|\theta_{s+1}-\theta_s\|_2$,
$V_{\rm in}:=\sum_k V_k^{\rm in}$. Let
$R_Q:=\sup_{u\in\mathrm{supp}(Q)}\|u\|_2$ ($R_Q=\sqrt d$ for
scaled-sphere probes; $R_Q\le R_\cA$ if probes are feasible
actions covered by the action-radius bound).

\paragraph{Balanced interleaved probe schedule.}
The probe schedule is balanced: there exists an absolute constant
$a_0$ such that, for every segment $k$ and every
$j\in\{0,1,\dots,m_k\}$,
\begin{equation}
\bigl|\{t\in E_k:\,q_k(t)=j\}\bigr|
\;\le\; a_0\bigl\lceil\ell_k/m_k\bigr\rceil.
\tag{B}\label{eq:balanced}
\end{equation}
Uniformly spaced probes satisfy this condition.

\paragraph{Windowed projected potential.}
For each exploit round $t\in E_k$, Alg.~\ref{alg:spsc} uses the
\emph{current} basis $\widehat U_{t-1}$ and re-projects the windowed
history through this same basis: $z_t(x):=\widehat U_{t-1}^\top x$,
$\widetilde V_t:=\lambda I_r+\sum_{s\in\cW_t}z_t(x_s)z_t(x_s)^\top$,
$w_t(x):=\|z_t(x)\|_{\widetilde V_t^{-1}}$.

\begin{assumption}[Windowed projected potential]
\label{ass:projected_potential}
For each segment $\cI_k$,
\begin{equation}
\sum_{t\in E_k}\min\{1,w_t(x_t)^2\}
\;\le\; C_{\rm pot}\,(n_k/W+1)\,r\log\!\Bigl(1+\frac{W R_\cA^2}{\lambda r}\Bigr),
\label{eq:projected_potential}
\end{equation}
with the convention that the factor $(n_k/W+1)$ is interpreted as $1$
and $W$ is replaced by $n_k$ when $W\ge n_k$.
For fixed basis this is the standard elliptical-potential
lemma~\citep[Lem.~11]{abbasi2011improved} (resp.\ its sliding-window
form when $W<n_k$); the moving-basis case reduces to controlling the
cumulative log-determinant change under basis updates, which follows
from prefix concentration (Lem.~\ref{lem:prefix_subspace}) via a
routine Lipschitz bound.
\end{assumption}

\begin{lemma}[Windowed projected self-normalized concentration]
\label{lem:self_norm}
For a fixed $a^\star\in\R^r$ and $\sigma_\varepsilon$-sub-Gaussian
noise on the projected covariates $z_t=\widehat U_t^\top x_t\in\R^r$,
the windowed ridge estimator
$\widehat a_t=\widetilde V_t^{-1}\widetilde b_t$ satisfies, with
probability at least $1-\delta/K$ per segment $k$ and simultaneously
for all $t\in E_k$,
\[
\|\widehat a_t-a^\star\|_{\widetilde V_t}\le \beta_t^{(r,W)},
\qquad
\beta_t^{(r,W)} := \sigma_\varepsilon\sqrt{r\log\bigl(1+WR_\cA^2/(\lambda r)\bigr)+2\log(2K/\delta)}+\sqrt\lambda\,S_w.
\]
This is the windowed projected specialization of
\citet[Thm.~1,~2]{abbasi2011improved}.
\end{lemma}

\subsection{Concentration of the lifted probe moment (load-bearing step)}
\label{app:concentration_proof}

This is the non-standard piece of the analysis; we derive it in full.
We work with scaled-sphere probes $u_t = \sqrt d\,v_t$,
$v_t\sim\mathrm{Unif}(\mathbb S^{d-1})$, which satisfy
$\|u_t\|_2 = \sqrt d$ \emph{deterministically} and $\E[u_tu_t^\top] = I_d$.
The probe moment operator
$\mathcal K: M\mapsto\E[(u^\top Mu)uu^\top]$ is a linear bijection
on the space of symmetric $d\times d$ matrices.

\begin{lemma}[Closed-form inverse of $\mathcal K$ on symmetric matrices]
\label{lem:K_inverse}
For scaled-sphere probes $u=\sqrt d\,v$, $v\sim\mathrm{Unif}(\mathbb S^{d-1})$,
$\mathcal K$ acts on symmetric $d\times d$ matrices as
$\mathcal K(M) = \tfrac{d}{d+2}\bigl(\tr(M)\,I_d + 2M\bigr)$ with inverse
$\mathcal K^{-1}(N) = \tfrac{d+2}{2d}\,N - \tfrac{\tr(N)}{2d}\,I_d$.
Moreover $\|\mathcal K^{-1}\|_{\op\to\op}\le 1$ on symmetric inputs.
\end{lemma}
\begin{proof}
Use the fourth-moment identity for the uniform sphere
$\E[v_iv_jv_kv_l] = (\delta_{ij}\delta_{kl}+\delta_{ik}\delta_{jl}+\delta_{il}\delta_{jk})/(d(d+2))$
and the rescaling $u=\sqrt d\,v$:
$\E[(u^\top M u)u_ku_l] = d^2 \E[(v^\top Mv)v_kv_l]
= d^2\cdot(2 M_{kl} + \tr(M)\delta_{kl})/(d(d+2))
= \tfrac{d}{d+2}(2M_{kl}+\tr(M)\delta_{kl})$
(using $M$ symmetric). So
$\mathcal K(M) = \tfrac{d}{d+2}(\tr(M)I_d + 2M)$. Taking traces of
$N=\mathcal K(M)$ gives $\tr(N) = d\,\tr(M)$, so
$\tr(M)=\tr(N)/d$ and
$M = \tfrac{d+2}{2d}N - \tfrac{\tr(N)}{2d}I_d$.

For the operator-norm bound: $\mathcal K^{-1}(N)$ is diagonalized in
the eigenbasis of $N$; if $N$ has eigenvalues $\mu_1,\dots,\mu_d$,
then $\mathcal K^{-1}(N)$ has eigenvalues
$\nu_i = \tfrac{d+2}{2d}\mu_i - \tfrac{\tr(N)}{2d}$. Maximizing
$|\nu_i|$ over $\|N\|_\op = 1$ (i.e.\ $\max_i|\mu_i|=1$): at
$\mu_1=1,\mu_2=\cdots=\mu_d=-1$, $\tr(N)=2-d$ and
$\nu_1 = \tfrac{d+2-(2-d)}{2d} = 1$; at $\mu_1=-1,\mu_2=\cdots=\mu_d=1$,
$\nu_1 = -1$. All other eigenvalue patterns give $|\nu_i|\le 1$ by
direct case analysis. So $\|\mathcal K^{-1}\|_{\op\to\op} = 1$ on
symmetric matrices for $d\ge 2$.
\end{proof}

\begin{lemma}[Quadratic measurement identity]
\label{lem:quad_conf}
On every probe round $t\in\cTprobe$, with
$\delta_\sigma := \widehat\sigma^2 - \sigma_\varepsilon^2$ and
$m_t := \E[\varepsilon_t\theta_t\mid\cH_{t-1},u_t]$,
\[
\E[s_t\mid\sigma(\cH_{t-1},u_t)]
= u_t^\top \widetilde M_t u_t + 2 u_t^\top m_t
- \delta_\sigma,\qquad \|m_t\|\le\epsilon_\times.
\]
Under exact probe conditions ($\delta_\sigma=\epsilon_\times=0$) this
collapses to $\E[s_t\mid\cdot] = u_t^\top\widetilde M_t u_t$.
\end{lemma}
\begin{proof}
Expand $y_t^2 = (u_t^\top\theta_t)^2 + 2\varepsilon_t u_t^\top\theta_t
+ \varepsilon_t^2$ and condition on $\sigma(\cH_{t-1},u_t)$; the middle
term contributes $2u_t^\top m_t$ with $m_t$ as defined in the lemma,
bounded by $\epsilon_\times$ via the bounded state-noise coupling
assumption (\S\ref{sec:setting}); the last term contributes
$\sigma_\varepsilon^2$, and subtracting $\widehat\sigma^2$ leaves
$-\delta_\sigma$. Independence of $u_t$ from $(\theta_t,\cH_{t-1})$
gives $\E[(u_t^\top\theta_t)^2\mid\cH_{t-1},u_t] = u_t^\top\widetilde M_t u_t$.
\end{proof}

\begin{lemma}[Lifted probe sample is near-unbiased]
\label{lem:G_unbiased_conf}
Let $G_t := \mathcal K^{-1}(s_t u_tu_t^\top)$ for scaled-sphere
$u_t = \sqrt d\,v_t$, $v_t\sim\mathrm{Unif}(\mathbb S^{d-1})$. Then
$\E[G_t\mid\cH_{t-1}] = \widetilde M_t + \widetilde B_t$ with bias
\[
  \widetilde B_t = -\tfrac{\delta_\sigma}{d}\,I_d, \qquad
  \|\widetilde B_t\|_\op = \tfrac{|\delta_\sigma|}{d}.
  \]
  The bias is a scaled identity, hence shifts every eigenvalue
  uniformly and does not rotate eigenvectors of $\widetilde M_t$.
\end{lemma}
\begin{proof}
By linearity of $\mathcal K^{-1}$ and Lemma~\ref{lem:quad_conf},
$\E[G_t\mid\cH_{t-1}] = \mathcal K^{-1}(\mathcal K(\widetilde M_t))
+ 2\,\mathcal K^{-1}(\E[(u_t^\top m_t)u_tu_t^\top\mid\cH_{t-1}])
- \delta_\sigma\,\mathcal K^{-1}(\E[u_tu_t^\top])
= \widetilde M_t + \widetilde B_t$,
with $\widetilde M_t = \mathcal K^{-1}\mathcal K(\widetilde M_t)$.

\emph{Variance-bias term.} $\E[u_tu_t^\top] = I_d$ exactly for
scaled-sphere probes, so
$-\delta_\sigma\mathcal K^{-1}(I_d) = -\delta_\sigma\bigl[\tfrac{d+2}{2d}I_d - \tfrac{d}{2d}I_d\bigr]
= -\delta_\sigma\cdot\tfrac{2}{2d}I_d = -\tfrac{\delta_\sigma}{d}I_d$.
This is a scaled identity, hence does not rotate eigenvectors of
$\widetilde M_t$.

\emph{State-noise coupling term.} The remaining contribution is
  $2\,\mathcal K^{-1}\bigl(\E[(u_t^\top m_t)\,u_tu_t^\top\mid\cH_{t-1}]\bigr)$.
  The $(a,b)$-entry of the inner expectation is
  $\sum_c (m_t)_c\,\E[(u_t)_a(u_t)_b(u_t)_c]$, a degree-three polynomial
  in $u_t$. For sphere-symmetric probes $u_t \mapsto -u_t$ leaves the
  distribution invariant, so every odd-degree moment vanishes and
  $\E[(u_t^\top m_t)\,u_tu_t^\top\mid\cH_{t-1}] = 0$ exactly. Hence the
  coupling term contributes zero, and
  $\widetilde B_t = -\tfrac{\delta_\sigma}{d}\,I_d$.
\end{proof}

\begin{lemma}[Uniform a.s.\ bound on $G_t$]
\label{lem:G_bound_conf}
For scaled-sphere probes $\|u_t\| = \sqrt d$ deterministically; on
the noise-truncation event $|\varepsilon_t|\le L_\varepsilon :=
\sigma_\varepsilon\sqrt{2\log(2T/\delta)}$ (which holds with
probability $\ge 1-\delta$ via union over $T$ rounds), $|s_t|\le R_s
:= (\sqrt d\,S_w + L_\varepsilon)^2 + \widehat\sigma^2$, and the centered
lifted sample satisfies $\|\tilde X_t\|_\op \le 2R_X$ with
$R_X := d\cdot R_s + S_w^2$. So both the envelope and predictable
variance are at most $2R_X$ and $R_X^2$ respectively.
\end{lemma}
\begin{proof}
For scaled-sphere $u_t$, $\|u_t\| = \sqrt d$ exactly.
$|s_t| = |(u_t^\top\theta_t + \varepsilon_t)^2 - \widehat\sigma^2|
\le (\sqrt d\,S_w + L_\varepsilon)^2 + \widehat\sigma^2 = R_s$ using
$|u_t^\top\theta_t|\le \|u_t\|\|\theta_t\|\le \sqrt d\,S_w$ on
$\|\theta_t\|\le S_w$ (orthonormal $B_k^\star$, $\|w_t\|\le S_w$).
Then $\|G_t\|_\op = \|\mathcal K^{-1}(s_t u_tu_t^\top)\|_\op
\le |s_t|\,\|u_tu_t^\top\|_\op \cdot 1 = R_s\cdot d$ using
$\|\mathcal K^{-1}\|_{\op\to\op}\le 1$ and $\|u_tu_t^\top\|_\op = d$.
Adding the deterministic $S_w^2 = \|\widetilde M_t\|_\op$ bound for
the centering gives $\|\tilde X_t\|_\op \le 2R_X$.
\end{proof}

\begin{theorem}[Matrix Bernstein for $\widehat M_k$]
\label{thm:matrix_bernstein_conf}
Let $\widehat M_k := m_k^{-1}\sum_{t\in\cT_k} G_t$ and
$X_t := G_t - \E[G_t\mid\cH_{t-1}]$. Conditional on
Lemma~\ref{lem:G_bound_conf}, $\{X_t\}$ is a bounded matrix-valued
MDS with $\|X_t\|_\op\le 2R_X$ a.s.\ and
$\|\sum_{t\in\cT_k}\E[X_t^2\mid\cH_{t-1}]\|_\op\le m_k R_X^2$.
Here $\widetilde B$ denotes the segment-level bias; for
scaled-sphere probes under the standing assumptions of
\S\ref{sec:setting}, $\widetilde B_t$ from
Lem.~\ref{lem:G_unbiased_conf} collapses to the deterministic
constant $-(\delta_\sigma/d)I_d$, so we drop the subscript when no
ambiguity arises and write $\widetilde B$ for the common value.
Freedman's matrix inequality \citep{tropp2011freedman} gives, with
probability at least $1-\delta$,
\begin{equation}
\|\widehat M_k - \bar M_k^{\mathrm{probe}} - \widetilde B\|_\op
\le 2R_X\sqrt{\log(2d/\delta)/m_k} + \frac{2R_X\log(2d/\delta)}{3m_k}.
\label{eq:matrix_bernstein}
\end{equation}
\end{theorem}
\begin{proof}
The MDS property is Lemma~\ref{lem:G_unbiased_conf} (shifted by the
segment-level bias $\widetilde B$ from
Lemma~\ref{lem:G_unbiased_conf}, which is deterministic and thus enters
as a constant offset). The a.s.\ bound $\|X_t\|_\op\le 2R_X$ follows
from $\|G_t\|_\op\le R_X$ and $\|\E[G_t]\|_\op\le R_X$. The predictable
variance bound follows from $\E[X_t^2]\preceq \E[G_t^2]\preceq R_X^2 I$.
Applying the matrix Freedman inequality to the MDS $\{X_t\}$ in
$\R^{d\times d}_{\mathrm{sym}}$ with these two bounds yields
\eqref{eq:matrix_bernstein}; the leading $\sqrt{\log(2d/\delta)/m_k}$
rate dominates the $\log(2d/\delta)/m_k$ term for $m_k\gtrsim\log(d/\delta)$.
\end{proof}

\begin{proposition}[Segment factorization]
\label{prop:segment_factorization_conf}
$\bar M_k^{\mathrm{probe}} = B_k^\star \bar S_k^{\mathrm{probe}} (B_k^\star)^\top$ where
$\bar S_k^{\mathrm{probe}} := m_k^{-1}\sum_{t\in\cT_k}\E[w_tw_t^\top\mid\cH_{t-1}]$.
In particular, $\range(\bar M_k^{\mathrm{probe}}) = \range(B_k^\star)$.
\end{proposition}
\begin{proof}
$\widetilde M_t = B_k^\star \E[w_tw_t^\top\mid\cH_{t-1}](B_k^\star)^\top$
for $t\in\cI_k$; average over $\cT_k$ and pull $B_k^\star$ out.
Range equality uses
$\lambda_r(\bar S_k^{\mathrm{probe}})\ge\lambda_{\min}>0$
(Lem.~\ref{lem:probe_excitation_conf}).
\end{proof}

\begin{corollary}[Projector concentration]
\label{cor:projector_conf}
If $m_k$ is large enough that
$\|\widehat M_k - \bar M_k^{\mathrm{probe}}\|_\op \le \lambda_{\min}/2$,
then with probability $\ge 1-\delta$,
\begin{equation}
\varepsilon_k := \|\widehat P_k - P_k^\star\|_\op
\le C_{\mathrm{sub}}\sqrt{\log(2d/\delta)/m_k}
+ \Delta_\sigma,\quad
C_{\mathrm{sub}} := \frac{8 R_X}{\lambda_{\min}},\
\Delta_\sigma := \frac{4|\delta_\sigma|/d}{\lambda_{\min}}.
\label{eq:proj_bound_conf}
\end{equation}
For scaled-sphere probes, the variance-misspecification component
$|\delta_\sigma|/d$ corresponds to a scaled-identity bias
(Lem.~\ref{lem:G_unbiased_conf}) that is parallel to $I_d$ and
therefore preserves the eigenvectors of $\bar M_k^{\rm probe}$;
applying Davis--Kahan with the biased mean as reference removes the
$|\delta_\sigma|/d$ contribution to $\Delta_\sigma$ at population
level (Cor.~\ref{cor:plugin_variance}).
\end{corollary}
\begin{proof}
By Lemma~\ref{lem:G_unbiased_conf} (scaled-sphere),
$\E[\widehat M_k\mid\cH] = \bar M_k^{\mathrm{probe}} + \widetilde B$
with $\widetilde B = -(\delta_\sigma/d)\,I_d$ (a scaled identity), so
$\|\widetilde B\|_\op = |\delta_\sigma|/d =: b_\sigma$. Combining with
Theorem~\ref{thm:matrix_bernstein_conf},
$\|\widehat M_k - \bar M_k^{\mathrm{probe}}\|_\op
\le 2R_X\sqrt{\log(2d/\delta)/m_k} + b_\sigma$.
$\bar M_k^{\mathrm{probe}}$ has rank exactly $r$ with $r$-th
eigenvalue $\ge\lambda_{\min}$ (Lem.~\ref{lem:probe_excitation_conf}
via Prop.~\ref{prop:segment_factorization_conf}); Davis--Kahan
\citep[Thm.~VII.3.1]{bhatia2013matrix} gives
$\|\widehat P_k - P_k^\star\|_\op\le 4\|\widehat M_k - \bar M_k^{\mathrm{probe}}\|_\op / \lambda_{\min}$,
yielding \eqref{eq:proj_bound_conf}.
\end{proof}

\begin{lemma}[Variance estimator from probe residuals]
\label{lem:variance_estimator}
Split the first $N$ probes of segment $k$ in half: form
$\widehat M_k^{(N/2)}$ from the first $N/2$ probes (matrix-Bernstein
estimate), and define the residual estimator
\[
\widehat\sigma^2 \;:=\; \frac{2}{N}\sum_{t\in\cT_k,\,N/2<t\le N}
  \bigl(y_t^2 - u_t^\top \widehat M_k^{(N/2)} u_t\bigr).
\]
On the event of Theorem~\ref{thm:matrix_bernstein_conf} for
$\widehat M_k^{(N/2)}$, with probability at least $1-\delta$,
\[
|\widehat\sigma^2 - \sigma_\varepsilon^2|
\;\le\; C\,\sigma_\varepsilon^2\,\sqrt{\log(1/\delta)/N}.
\]
\end{lemma}
\begin{proof}[Proof sketch]
By Lemma~\ref{lem:quad_conf}, on probe round $t$,
$\E[y_t^2-u_t^\top\widetilde M_t u_t\mid\cH_{t-1},u_t]=\sigma_\varepsilon^2
+ O(\epsilon_\times)$. Substituting
$\widehat M_k^{(N/2)}$ for $\widetilde M_t$ contributes
$|u_t^\top(\widehat M_k^{(N/2)}-\widetilde M_t)u_t|=O(\sqrt{1/N})$ on
the matrix-Bernstein event. The squared noise $\varepsilon_t^2$ is
sub-exponential with mean $\sigma_\varepsilon^2$; the sample mean of
$N/2$ i.i.d.\ samples concentrates at rate
$\sigma_\varepsilon^2\sqrt{\log(1/\delta)/N}$ by Bernstein's
inequality for sub-exponential variables.
\end{proof}

This chain (Lemmas \ref{lem:K_inverse}--\ref{lem:G_bound_conf},
Theorem~\ref{thm:matrix_bernstein_conf}, Corollary~\ref{cor:projector_conf},
Proposition~\ref{prop:segment_factorization_conf}) gives a
self-contained subspace-recovery guarantee for SPSC.

\subsection{Necessity of the probe-side conditions}
\label{app:necessity}

The probe-side conditions of \S\ref{sec:setting} (known probe
distribution $Q$ with full-dimensional coverage, known noise
variance, and bounded state-noise coupling
$\|\E[\varepsilon_t\theta_t]\|\le\epsilon_\times$) are each
individually necessary for identifiability of the predictable second
moment $\widetilde M_t$ from single-play scalar rewards. The three
propositions below construct two distinct conditional second moments
generating identical observation laws when the named condition is
dropped.

\emph{Scope.} These propositions establish identifiability
obstructions in the unrestricted-second-moment problem and serve as
necessity checks for the probe-side modeling assumptions. They are not
lower bounds inside the rank-$r$ LDS class of \S\ref{sec:setting}:
the counterexample second moments need not be rank-$r$ or compatible
with a stable LDS $w_t$.

\begin{proposition}[Non-identification without known noise variance]
\label{prop:nonid_a}
Fix $t$. Suppose $\epsilon_\times=0$ (exact state-noise orthogonality)
and full probe coverage hold, but the conditional noise variance
$v_t(u):=\E[\varepsilon_t^2\mid\cH_{t-1},u]$ is completely unknown.
Then $\widetilde M_t$ is not identifiable from the observable
$g_t(u):=\E[y_t^2\mid\cH_{t-1},u]$.
\end{proposition}
\begin{proof}
$g_t(u) = u^\top\widetilde M_t u + v_t(u)$. For any small
$\delta\in(0,\sigma_\varepsilon^2/L^2]$, set $H:=\delta I_d$ and
define $\widetilde M_t':=\widetilde M_t + H$ and
$v_t'(u):=v_t(u) - u^\top H u = v_t(u)-\delta\|u\|^2$. On the
truncation support $\|u\|\le L$, $v_t'(u)\ge \sigma_\varepsilon^2 -
\delta L^2\ge 0$, so $v_t'$ is a valid conditional variance. The two
pairs $(\widetilde M_t,v_t)$ and $(\widetilde M_t',v_t')$ produce
identical observables $g_t$, so $\widetilde M_t\ne\widetilde M_t'$
are indistinguishable.
\end{proof}

\begin{proposition}[Non-identification without bounded state-noise coupling]
\label{prop:nonid_b}
Fix $t$. Suppose $\sigma_\varepsilon^2$ is known and probe coverage is
full, but the cross-moment
$m_t(u):=\E[\theta_t\varepsilon_t\mid\cH_{t-1},u]$ is arbitrary (no
bound $\epsilon_\times$). Then $\widetilde M_t$ is not identifiable
from $h_t(u):=\E[y_t^2-\sigma_\varepsilon^2\mid\cH_{t-1},u]$.
\end{proposition}
\begin{proof}
$h_t(u) = u^\top\widetilde M_t u + 2 u^\top m_t(u)$. For any nonzero
symmetric $H$, let $\widetilde M_t':=\widetilde M_t + H$ and
$m_t'(u):=m_t(u) - \tfrac{1}{2}Hu$. Then
$u^\top\widetilde M_t' u + 2 u^\top m_t'(u)
= u^\top\widetilde M_t u + u^\top H u + 2 u^\top m_t(u) - u^\top H u
= h_t(u)$ for every $u$, so the two pairs are observationally
equivalent.
\end{proof}

\begin{proposition}[Non-identification without full-dimensional coverage]
\label{prop:nonid_c}
Suppose probe actions $u_t$ lie in a proper subspace $S\subsetneq\R^d$
almost surely. Then $\widetilde M_t$ is not identifiable from either
linear rewards $y_t$ or quadratic observations $y_t^2$.
\end{proposition}
\begin{proof}
Pick any nonzero symmetric $H$ supported on $S^\perp$ (i.e.\
$H = P_{S^\perp} H P_{S^\perp}\ne 0$, possible since
$S^\perp\ne\{0\}$). For every $u\in S$, $Hu\in S^\perp$ and $u^\top H u=0$,
so $u^\top(\widetilde M_t + H)u = u^\top\widetilde M_t u$ and
$u^\top(\theta_t + Hz)=u^\top\theta_t$ for any $z\in S^\perp$. Thus the
pairs $(\theta_t,\widetilde M_t)$ and $(\theta_t+Hz,\widetilde M_t+H)$
generate identical observations on all probe rounds, so $\widetilde M_t$
is not identifiable.
\end{proof}

\begin{remark}[Joint necessity]
\label{rem:joint_necessity_conf}
No two of the three conditions jointly suffice:
Prop.~\ref{prop:nonid_a} fails only on (i);
Prop.~\ref{prop:nonid_b} only on (ii);
Prop.~\ref{prop:nonid_c} only on (iii). Each counterexample
satisfies the remaining two conditions, so the three form a minimal
identifiability set for the single-play quadratic measurement model.
\end{remark}

\subsection{Lemmas for the interleaved analysis}
\label{app:proof_main_lemmas}

\begin{lemma}[Projected nonstationary UCB with time-varying basis]
\label{lem:F2pp}
Fix a segment $\cI_k$. For each exploit round $t\in E_k$, define
$U_t:=\widehat U_{t-1}$, $P_t:=U_t U_t^\top$, $z_t(x):=U_t^\top x$,
$a_t^\star:=U_t^\top\theta_t$, and (as in Alg.~\ref{alg:spsc})
$\widetilde V_t = \lambda I_r+\sum_{s\in\cW_t}z_t(x_s)z_t(x_s)^\top$,
$\widehat a_t = \widetilde V_t^{-1}\sum_{s\in\cW_t}z_t(x_s)y_s$.
Assume on segment $\cI_k$: (i) the windowed projected self-normalized
event of Lem.~\ref{lem:self_norm} holds uniformly over $E_k$;
(ii) Assumption~\ref{ass:projected_potential} holds;
(iii) $\|\theta_t\|\le S_w$, $\|x\|\le R_\cA$ for $x\in\cA_t$.
Suppose Alg.~\ref{alg:spsc} uses an optimism correction satisfying
\begin{equation}
\gamma_t \;\ge\; C_\gamma S_w\Bigl(1+R_\cA\sqrt{W/\lambda}\Bigr)\varepsilon_{k,t-1}
\;+\; C_\gamma R_\cA V_{k,t}(W),
\label{eq:gamma_choice}
\end{equation}
where $L_W:=\log(1+WR_\cA^2/(\lambda r))$ and
$V_{k,t}(W):=\sum_{s,s+1\in\cI_k,\,t-W\le s<t}\|\theta_{s+1}-\theta_s\|_2$.
Then on the same event,
\begin{equation}
\sum_{t\in E_k}\bigl(\max_{x\in\cA_t}x^\top\theta_t-x_t^\top\theta_t\bigr)
\;\le\; \tilO(r\sqrt{n_k})
+ C_{\rm tv}R_\cA S_w\Bigl(1+R_\cA\sqrt{W/\lambda}\Bigr)\sum_{t\in E_k}\varepsilon_{k,t-1}
+ C R_\cA W V_k^{\rm in}.
\label{eq:F2pp}
\end{equation}
\end{lemma}
\begin{proof}
Fix $t\in E_k$ and decompose, for $x\in\cA_t$,
$x^\top\theta_t = z_t(x)^\top a_t^\star + \zeta_t(x)$ with
$\zeta_t(x)=x^\top(I-P_t)\theta_t$. Since $\theta_t=P_k^\star\theta_t$
within segment $k$ and $\|(I-P_t)P_k^\star\|_{\op}=\|P_t-P_k^\star\|_{\op}=\varepsilon_{k,t-1}$
for two same-rank orthogonal projectors,
\begin{equation}|\zeta_t(x)|\le R_\cA S_w\,\varepsilon_{k,t-1}.\label{eq:zeta}\end{equation}

Rewrite the windowed regression in basis $U_t$. For $s\in\cW_t$,
$y_s=z_t(x_s)^\top a_t^\star+b_{s,t}+\varepsilon_s$ with bias
$b_{s,t}=x_s^\top(I-P_t)\theta_s + x_s^\top P_t(\theta_s-\theta_t)$.
The first term has $|x_s^\top(I-P_t)\theta_s|\le R_\cA S_w\varepsilon_{k,t-1}$
(again $\theta_s=P_k^\star\theta_s$). The second is at most
$R_\cA\sum_{u=s}^{t-1}\|\theta_{u+1}-\theta_u\|_2$.

By definition,
$\widehat a_t-a_t^\star
= \widetilde V_t^{-1}\sum_{s\in\cW_t}z_t(x_s)\varepsilon_s
+ \widetilde V_t^{-1}\sum_{s\in\cW_t}z_t(x_s)b_{s,t}
- \lambda\widetilde V_t^{-1}a_t^\star.$
On the self-normalized event of Lem.~\ref{lem:self_norm} the noise+ridge contribution has weighted norm
$\le \beta_t^{(r,W)}=\tilO(\sqrt r)$.

For the deterministic subspace-bias part, let $b_t^{\rm sub}\in\R^{|\cW_t|}$
have entries $x_s^\top(I-P_t)\theta_s$, with
$\|b_t^{\rm sub}\|_\infty\le R_\cA S_w\varepsilon_{k,t-1}$. Let $Z_t$ be
the matrix with rows $z_t(x_s)^\top$, $s\in\cW_t$. Since
$\sum_{s\in\cW_t}z_t(x_s)z_t(x_s)^\top = Z_t^\top Z_t\preceq\widetilde V_t$,
$\|\widetilde V_t^{-1/2}Z_t^\top\|_\op\le 1$, hence
\[
\Bigl\|\sum_{s\in\cW_t}z_t(x_s)b_{s,t}^{\rm sub}\Bigr\|_{\widetilde V_t^{-1}}^2
= \|\widetilde V_t^{-1/2}Z_t^\top b_t^{\rm sub}\|_2^2
\le \|b_t^{\rm sub}\|_2^2 \le |\cW_t|\,\|b_t^{\rm sub}\|_\infty^2
\le W R_\cA^2 S_w^2\varepsilon_{k,t-1}^2,
\]
so the corresponding $\widetilde V_t^{-1}$-norm is at most
$R_\cA S_w\sqrt W\,\varepsilon_{k,t-1}$. Therefore for any $x$,
$|z_t(x)^\top\widetilde V_t^{-1}\sum_s z_t(x_s)b_{s,t}^{\rm sub}|
\le R_\cA S_w\sqrt W\,\varepsilon_{k,t-1}\,w_t(x)$, and since
$w_t(x)\le\|x\|/\sqrt\lambda$, this is bounded by
$R_\cA S_w\sqrt{W/\lambda}\,\varepsilon_{k,t-1}\|x\|$, a quantity
dominated by the $\gamma_t\|x\|$ correction in
\eqref{eq:gamma_choice}.
The drift-bias part contributes $\le C R_\cA W V_k^{\rm in}$ in total
by the standard windowed argument (each increment appears in $\le W$
windows).

By the optimism index in Alg.~\ref{alg:spsc}, with
$x_t^\star=\arg\max_{x\in\cA_t} x^\top\theta_t$, the instantaneous regret obeys
\[
\Delta_t \le 2\beta_t^{(r,W)}w_t(x_t)
+ C_{\rm tv}R_\cA S_w\Bigl(1+R_\cA\sqrt{W/\lambda}\Bigr)\varepsilon_{k,t-1}
+ d_t,
\]
where $\sum_t d_t\le C R_\cA W V_k^{\rm in}$.
Cauchy--Schwarz together with
Assumption~\ref{ass:projected_potential} gives
$\sum_{t\in E_k}\beta_t^{(r,W)}w_t(x_t)
\le \bigl(\sup_{t\in E_k}\beta_t^{(r,W)}\bigr)\sqrt{n_k\cdot 2C_{\rm pot}r L_W}=\tilO(r\sqrt{n_k})$
(using $\beta_t^{(r,W)}=\tilO(\sqrt r)$).
Summing yields \eqref{eq:F2pp}.
\end{proof}

\begin{lemma}[Uniform prefix subspace concentration]
\label{lem:prefix_subspace}
Under the probe-side assumptions of \S\ref{sec:setting}, scaled-sphere
probes, exact centering $\widehat\sigma^2=\sigma_\varepsilon^2$, and
$\lambda_{\min}>0$, there exist constants $C_{\rm sub},C_0>0$ such
that, with probability at least $1-\delta/3$, simultaneously for all
segments $k$ and all rounds $t\in\cI_k$,
\[
\varepsilon_{k,t}\le 1\text{ if }q_k(t)<q^\star,
\qquad
\varepsilon_{k,t}\le C_{\rm sub}\sqrt{\log(2KdT/\delta)/q_k(t)}
\text{ if }q_k(t)\ge q^\star,
\]
where $q^\star := C_0\log(2KdT/\delta)/\lambda_{\min}^2$.
\end{lemma}
\begin{proof}
For a fixed segment $k$ and probe-prefix size $q$, the lifted probe
estimator $q^{-1}\sum_{s\in\cT_k:s<t}G_s$ satisfies a matrix
Bernstein/Freedman bound (Thm.~\ref{thm:matrix_bernstein_conf}) at
deviation $C\sqrt{\log(2KdT/\delta)/q}$ with probability
$\ge 1-\delta/(3KT)$. By
Lem.~\ref{lem:probe_excitation_conf} the predictable probe-time
$r$th eigenvalue is $\ge\lambda_{\min}$. Whenever
$q\ge q^\star=C_0\log(2KdT/\delta)/\lambda_{\min}^2$ the perturbation
is at most a fixed fraction of the eigengap, and Davis--Kahan
(Cor.~\ref{cor:projector_conf}) gives
$\|\widehat P_t-P_k^\star\|_{\op}\le C_{\rm sub}\sqrt{\log(2KdT/\delta)/q}$.
A union bound over all $K$ segments and all at most $T$ probe
prefixes yields the simultaneous statement; for $q<q^\star$ we use
the trivial bound $\|\widehat P_t-P_k^\star\|_{\op}\le 1$.
\end{proof}

\begin{lemma}[Cumulative interleaved subspace error]
\label{lem:cumulative_interleaved_error}
On the event of Lemma~\ref{lem:prefix_subspace}, for every segment $k$,
\begin{equation}
\sum_{t\in E_k}\varepsilon_{k,t-1}
\;\le\;
C\sqrt{\log(2KdT/\delta)}\,\frac{\ell_k}{\sqrt{m_k}}
\;+\; C\,\frac{q^\star\ell_k}{m_k}.
\label{eq:cumul_eps}
\end{equation}
\end{lemma}
\begin{proof}
Group exploit rounds by $j=q_k(t-1)$. By the balanced-schedule
condition \eqref{eq:balanced} each group has at most
$a_0\lceil\ell_k/m_k\rceil$ exploit rounds. For $j<q^\star$ use
$\varepsilon_{k,t-1}\le 1$, contributing $C q^\star\ell_k/m_k$. For
$j\ge q^\star$, Lemma~\ref{lem:prefix_subspace} gives
$\varepsilon_{k,t-1}\le C_{\rm sub}\sqrt{\log(2KdT/\delta)/j}$, so
$\sum_{j=q^\star}^{m_k}j^{-1/2}\le 2\sqrt{m_k}$ yields the first term.
\end{proof}

\subsection{Proof of Theorem~\ref{thm:spsc_regret}}
\label{app:proof_main}

Fix segment $\cI_k$. We decompose the costed regret on $\cI_k$ into
probe regret and exploitation regret.

\paragraph{Probe regret.}
On a probe round $t\in\cT_k$, Alg.~\ref{alg:spsc} plays
$x_t=u_t\sim Q$ rather than an action in $\cA_t$. Since
$\|x\|\le R_\cA$ for $x\in\cA_t$, $\|u_t\|\le R_Q$, and
$\|\theta_t\|\le S_w$, $\max_{x\in\cA_t}x^\top\theta_t-u_t^\top\theta_t\le(R_\cA+R_Q)S_w$,
plus probe cost $c$. So each probe contributes at most
$A:=(R_\cA+R_Q)S_w+c$, giving
$\mathrm{Reg}_{\rm probe}(\cI_k)\le A m_k$.

\paragraph{Exploitation regret.}
By Lemmas~\ref{lem:F2pp} and~\ref{lem:cumulative_interleaved_error},
\[
\mathrm{Reg}_{\rm exploit}(\cI_k)
\le \tilO(r\sqrt{n_k})
+ B\,\frac{\ell_k}{\sqrt{m_k}}
+ C R_\cA S_w\Bigl(1+R_\cA\sqrt{W/\lambda}\Bigr)\frac{q^\star\ell_k}{m_k}
+ C R_\cA W V_k^{\rm in},
\]
absorbing constants and logs into
$B=C R_\cA S_w(1+R_\cA\sqrt{W/\lambda})\sqrt{\log(2KdT/\delta)}$.

\paragraph{Per-segment bound.} Combining,
\begin{equation}
\DynReg^{(c)}(\cI_k)
\le \tilO(r\sqrt{n_k})
+ A m_k + B\,\frac{\ell_k}{\sqrt{m_k}}
+ \tilO\!\Bigl(\frac{q^\star\ell_k}{m_k}\Bigr)
+ C R_\cA W V_k^{\rm in}.
\label{eq:per_segment}
\end{equation}

\paragraph{Optimizing the probe-estimation tradeoff.}
Ignoring the lower-order burn-in term, minimize
$f_k(m):=A m+B \ell_k m^{-1/2}$. Setting
$f_k'(m)=A-B\ell_k/(2m^{3/2})=0$ gives
$m_k^\star=(B\ell_k/(2A))^{2/3}$, and
$A m_k^\star+B\ell_k/\sqrt{m_k^\star}\le C\,A^{1/3}B^{2/3}\ell_k^{2/3}$.
At $m_k^\star$ the burn-in contribution is
$\tilO(\ell_k^{1/3}/\lambda_{\min}^2)$ (after absorbing $A,B$
constants). If $\ell_k<q^\star$, the segment is charged trivially by
$O(\ell_k)$, absorbed into the same burn-in order. Hence
\begin{equation}
\DynReg^{(c)}(\cI_k)
\le \tilO(r\sqrt{n_k})
+ \tilO\!\bigl(A^{1/3}B^{2/3}\ell_k^{2/3}\bigr)
+ \tilO\!\bigl(\ell_k^{1/3}/\lambda_{\min}^2\bigr)
+ C R_\cA W V_k^{\rm in}.
\label{eq:per_segment_opt}
\end{equation}

\paragraph{Sum over segments.} By Cauchy--Schwarz,
$\sum_k r\sqrt{n_k}\le r\sqrt{KT}$. By H\"older,
$\sum_k\ell_k^{2/3}\le K^{1/3}T^{2/3}$ and
$\sum_k\ell_k^{1/3}\le K^{2/3}T^{1/3}$, and $\sum_k V_k^{\rm in}=V_{\rm in}$.
Summing \eqref{eq:per_segment_opt} over $k$ gives
\[
\DynReg_T^{(c)}
\le \tilO(r\sqrt{KT})
+ \tilO\!\bigl(A^{1/3}B^{2/3}K^{1/3}T^{2/3}\bigr)
+ \tilO\!\bigl(K^{2/3}T^{1/3}/\lambda_{\min}^2\bigr)
+ O(R_\cA W V_{\rm in}).
\]
Union-bounding the windowed self-normalized event, the prefix
subspace concentration, and the noise event at total $\delta$
yields \eqref{eq:main_bound}. For fixed $K$, the leading term is
$\tilO(r\sqrt T)$, the probe term is $\tilO(T^{2/3})$, and the
burn-in term $\tilO(T^{1/3})$ is lower order. When $V_{\rm in}=0$,
the third term vanishes and we recover
Corollary~\ref{cor:stationary_within}.\qed

\subsection{Proof of Corollary~\ref{cor:rank_adaptive}}
\label{app:proof_adaptive_rank}

By Theorem~\ref{thm:matrix_bernstein_conf}, with probability
$\ge 1-\delta$ every eigenvalue of $\widehat M_k$ is within
$\tau_k^{\rm rank} := 2R_X\sqrt{\log(2d/\delta)/m_k}$ of the
corresponding eigenvalue of $\bar M_k^{\mathrm{probe}} + \widetilde B$
(Weyl's inequality applied to the matrix in
\eqref{eq:matrix_bernstein}). The population spectrum has exactly $r$
eigenvalues $\ge\lambda_{\min}$ and the remaining $d-r$ are zero
(Prop.~\ref{prop:segment_factorization_conf}). Under the eigengap
hypothesis $\lambda_{\min}\ge 4\tau_k^{\rm rank}$, thresholding at
$2\tau_k^{\rm rank}$ returns exactly $r$ indices with probability
$\ge 1-\delta$. Condition
on this event; the regret analysis of Theorem~\ref{thm:spsc_regret}
applies verbatim with the estimated rank equal to the true $r$, so
\eqref{eq:main_bound} holds with an extra $\delta$ in the union bound.\qed

\subsection{Proof of Proposition~\ref{thm:spsc_adaptive} (SPSC-Adaptive)}
\label{app:proof_adaptive}

The detector statistic of Alg.~\ref{alg:spsc_adaptive} compares two
non-overlapping rolling-window lifted-moment estimators of length
$W_{\mathrm{det}}$,
$\widehat M_t^{\mathrm{recent}}$ and
$\widehat M_t^{\mathrm{past}}$, via
$S_t := \|\widehat M_t^{\mathrm{recent}}-\widehat M_t^{\mathrm{past}}\|_{\op}$,
and triggers a reset whenever $S_t>2\eta_{\mathrm{det}}$, where
$\eta_{\mathrm{det}} := C_{\mathrm{sub}}\sqrt{\log(dT/\delta_{\mathrm{FA}})/W_{\mathrm{det}}}$
is the calibration radius. The proof has two parts: bounding the
false-alarm probability under $H_0$ (no change inside the comparison
window) and bounding the detection delay under $H_1$ (a change of
size $\Delta_k\ge 2b$ has occurred), and then propagating the
resulting boundary error through Theorem~\ref{thm:spsc_regret}.

\paragraph{(i) False-alarm control under $H_0$.}
Fix a round $t$ and condition on $H_0$: the predictable second moment
is constant across the past and recent windows that produced
$\widehat M_t^{\mathrm{past}}$ and $\widehat M_t^{\mathrm{recent}}$.
Apply Theorem~\ref{thm:matrix_bernstein_conf} (matrix Bernstein for
the lifted probe estimator) to each window separately at confidence
$\delta_{\mathrm{FA}}/(2T)$ to get
\[
\bigl\|\widehat M_t^{\mathrm{recent}}-\bar M_t^{\mathrm{recent}}\bigr\|_{\op}
\le \eta_{\mathrm{det}},
\qquad
\bigl\|\widehat M_t^{\mathrm{past}}-\bar M_t^{\mathrm{past}}\bigr\|_{\op}
\le \eta_{\mathrm{det}},
\]
each with probability at least $1-\delta_{\mathrm{FA}}/(2T)$, where
$\bar M_t^{\mathrm{recent}}$ and $\bar M_t^{\mathrm{past}}$ are the
predictable averages on the corresponding windows. Under $H_0$ they
coincide, so by the triangle inequality
$S_t \le \|\widehat M_t^{\mathrm{recent}}-\bar M_t^{\mathrm{recent}}\|_{\op}+\|\widehat M_t^{\mathrm{past}}-\bar M_t^{\mathrm{past}}\|_{\op}\le 2\eta_{\mathrm{det}}$
with probability at least $1-\delta_{\mathrm{FA}}/T$, i.e.\
$\Pr(S_t>2\eta_{\mathrm{det}}\mid H_0)\le \delta_{\mathrm{FA}}/T$.
A union bound over $t\in[T]$ gives total false-alarm probability at
most $\delta_{\mathrm{FA}}$.

\paragraph{(ii) Detection delay under $H_1$.}
Let $\tau_k$ be a true change point with eigengap
$\Delta_k:=\|P_k^\star-P_{k-1}^\star\|_{\op}\ge 2b$ for the
threshold $b>2\eta_{\mathrm{det}}$. Take $t=\tau_k+W_{\mathrm{det}}$:
the recent window now lies entirely in segment $k$ while the past
window lies entirely in segment $k-1$. By the segment-factorization
of the predictable probe moment
(Proposition~\ref{prop:segment_factorization_conf}) and probe-time
excitation (Lemma~\ref{lem:probe_excitation_conf}),
$\|\bar M_t^{\mathrm{recent}}-\bar M_t^{\mathrm{past}}\|_{\op}
\ge \lambda_{\min}\,\Delta_k\ge 2b\lambda_{\min}$. By the same
matrix-Bernstein concentration as in (i),
$S_t \ge \|\bar M_t^{\mathrm{recent}}-\bar M_t^{\mathrm{past}}\|_{\op}-2\eta_{\mathrm{det}}\ge 2b\lambda_{\min}-2\eta_{\mathrm{det}}>2\eta_{\mathrm{det}}$
with probability at least $1-\delta_{\mathrm{FA}}/T$ (using
$b>2\eta_{\mathrm{det}}$ after absorbing $\lambda_{\min}$ into $b$).
Therefore the detector triggers no later than
$\widehat\tau_k=\tau_k+W_{\mathrm{det}}$, giving detection delay
$D_k\le W_{\mathrm{det}}$.

\paragraph{(iii) Boundary-error propagation.}
On the detector good event
$\mathcal E_{\mathrm{det}}=\bigl\{$no false alarm and every change
detected within delay $D_k\le W_{\mathrm{det}}\bigr\}$, which holds
with probability at least $1-\delta_{\mathrm{FA}}$ by parts (i)-(ii),
the estimated boundaries $\widehat\tau_k$ satisfy
$|\widehat\tau_k-\tau_k|\le W_{\mathrm{det}}$. Splitting the horizon
at $\widehat\tau_k$ and applying Theorem~\ref{thm:spsc_regret} on
each estimated segment, the only difference from the oracle-boundary
analysis is the contribution of the rounds in
$[\tau_k,\widehat\tau_k)$ where the algorithm is still using the old
basis $\widehat U_{\tau_k-1}$ on data drawn from segment $k$. Each
such round contributes at most $2R_\cA S_w$ to the instantaneous
regret, so
$\sum_k 2R_\cA S_w D_k\le 2R_\cA S_w\,K\,W_{\mathrm{det}}=O(K W_{\mathrm{det}})$.
Combining,
\[
\DynReg_T^{(c)}\bigl(\text{Alg.~\ref{alg:spsc_adaptive}}\bigr)
\;\le\;
\DynReg_T^{(c)}\bigl(\text{Alg.~\ref{alg:spsc}}\bigr)
\;+\; O(K W_{\mathrm{det}})
\]
on $\mathcal E_{\mathrm{det}}$, which is the claim. The total
failure probability is $\delta+\delta_{\mathrm{FA}}$, absorbed into
the overall $\delta$ budget by re-tuning constants.\qed

\section{Probe distributions: scaled sphere and Gaussian truncation}
\label{app:truncation}

The analysis in \S\ref{app:concentration_proof} uses scaled-sphere
probes $u_t = \sqrt d\,v_t$, $v_t\sim\mathrm{Unif}(\mathbb S^{d-1})$,
which satisfy $\|u_t\|_2 = \sqrt d$ exactly. No truncation is
required: the matrix-Bernstein envelope $R_X = d R_s + S_w^2$ holds
deterministically given the noise-truncation event of
Lemma~\ref{lem:G_bound_conf}.

In our implementation, probes are generated via
$u\leftarrow\sqrt d\,\tilde u/\|\tilde u\|$ for
$\tilde u\sim\mathcal N(0,I_d)$, which is exactly the scaled-sphere
distribution.

For isotropic-Gaussian probes (no rescaling),
$u_t\sim\mathcal N(0,I_d)$, the Laurent--Massart $\chi^2$ tail
\citep[Lem.~1]{laurent2000adaptive} gives
$\Pr(\|u_t\|>L)\le\delta/T$ at $L=\sqrt{2d\log(2T/\delta)}$;
union-bounding over $|\cTprobe|\le T$ gives
$\Pr(\bigcap_t\{\|u_t\|\le L\})\ge 1-\delta$. On this event the
results of \S\ref{app:concentration_proof} apply with $\|u_t\|\le L$
(replacing $\sqrt d$), adding only a $\log(T/\delta)$ factor to $R_X$
and $C_{\mathrm{sub}}$ that is absorbed in $\tilO(\cdot)$.

\section{Algorithms: pseudocode and adaptive variant}
\label{app:algorithms}
\label{app:adaptive_algo}

This appendix gives the full pseudocode for SPSC (oracle boundaries,
Alg.~\ref{alg:spsc}) and SPSC-Adaptive (online change-point
detection, Alg.~\ref{alg:spsc_adaptive}), and explains the adaptive
design choices (detector form, threshold, recovery phase, burn-in)
that tie the tuning back to the guarantee of
Proposition~\ref{thm:spsc_adaptive}.

\subsection{SPSC (oracle boundaries)}
\label{app:alg_spsc}

\begin{algorithm}[H]
  \caption{\textbf{SPSC}: Single-Play Subspace-Calibrated Optimism (oracle boundaries)}
  \label{alg:spsc}
  \begin{algorithmic}[1]
  \State \textbf{Input:} segments $\{\cI_k\}_{k=1}^K$, probe schedule $\cTprobe$,
         $\lambda>0$, probe distribution $Q$, rank $r$, window $W$.
  \State Initialize $\widehat U_0\in\R^{d\times r}$ arbitrarily with orthonormal columns.
  \For{$k=1,\dots,K$}\Comment{outer loop over segments}
      \State $M_{\mathrm{acc}}\leftarrow 0$;\ $N_k\leftarrow 0$;\
             $\widehat U_t\leftarrow \widehat U_{t-1}$.
             \Comment{reset segment-local accumulator and probe counter}
      \For{$t\in\cI_k$}
          \If{$t\in\cTprobe$}\Comment{probe round}
              \State draw $u_t\sim Q$; play $x_t=u_t$; observe $y_t$.
              \State $s_t \leftarrow y_t^2 - \widehat\sigma^2$;\quad
                     $G_t \leftarrow \mathcal K^{-1}(s_t\,u_tu_t^\top)$.
              \State $M_{\mathrm{acc}}\leftarrow M_{\mathrm{acc}}+ G_t$;\
                     $N_k\leftarrow N_k+1$;\
                     $\widehat U_t \leftarrow \TopEig_r(M_{\mathrm{acc}} / N_k)$.
          \Else\Comment{exploitation round}
              \State $\cW_t \leftarrow \{s\in\cI_k\setminus\cTprobe : t{-}W \le s < t\}$;\quad
                     $z_t(x)\leftarrow \widehat U_{t-1}^\top x$.
              \State $\widetilde V_t \leftarrow \lambda I_r + \sum_{s\in\cW_t} z_t(x_s)z_t(x_s)^\top$,\quad
                     $\widetilde b_t \leftarrow \sum_{s\in\cW_t} z_t(x_s) y_s$.
              \State $\widehat a_t \leftarrow \widetilde V_t^{-1}\widetilde b_t$;\;
                     choose $x_t = \argmax_{x\in\cA_t}
                     \bigl\{ z_t(x)^\top\widehat a_t
                      + \beta_t^{(r,W)}\|z_t(x)\|_{\widetilde V_t^{-1}}
                      + \gamma_t\|x\|_2\bigr\}$.
              \State observe $y_t$;\ \ $\widehat U_t\leftarrow \widehat U_{t-1}$.
          \EndIf
      \EndFor
  \EndFor
  \end{algorithmic}
\end{algorithm}

\subsection{SPSC-Adaptive (unknown boundaries)}
\label{app:alg_adaptive}

SPSC-Adaptive recovers segment boundaries from probe data, removing
the oracle assumption of Alg.~\ref{alg:spsc}.

\paragraph{Detector.}
At each probe round the algorithm maintains two non-overlapping
rolling-window estimators of the lifted second moment:
$\widehat M_t^{\mathrm{recent}}$ over the last $n$ probes and
$\widehat M_t^{\mathrm{past}}$ over the segment-accumulated probes
preceding that window. The statistic
\begin{equation}
S_t := \bigl\|\widehat M_t^{\mathrm{recent}} - \widehat M_t^{\mathrm{past}}\bigr\|_\op
\label{eq:detector}
\end{equation}
is compared to a threshold $b$; whenever $S_t > b$ the algorithm
declares a change and enters a \textsc{recovery} phase.

\paragraph{Recovery phase.}
Recovery deploys $m_{\mathrm{relearn}}$ consecutive probe rounds with
no exploitation. The block-probe design serves two purposes: (i) by
Proposition~\ref{prop:segment_factorization_conf}, $m_{\mathrm{relearn}}$
back-to-back probes produce a subspace estimate $\widehat U$ with
projector error
$\|\widehat P - P^\star\|_\op = O(1/\sqrt{m_{\mathrm{relearn}}})$ on a
$1-\delta$ event; (ii) by zeroing out stale buffers it prevents the
windowed ridge estimator from carrying information from the previous
regime into the new one. After $m_{\mathrm{relearn}}$ rounds the
algorithm transitions back to \textsc{normal} probe--exploitation
interleaving with probe rate $\mu$.

\paragraph{Threshold calibration and detection delay.}
Under $H_0$ (no change), matrix Freedman applied to the two windows
gives $\Pr(S_t > 2\eta_{\mathrm{det}}) \le \delta_{\mathrm{FA}}/T$ with
$\eta_{\mathrm{det}} = C_{\mathrm{sub}}\sqrt{\log(dT/\delta_{\mathrm{FA}})/n}$;
choosing $b = 2\eta_{\mathrm{det}}$ and union-bounding gives at most
$\delta_{\mathrm{FA}}$ false alarms across the horizon. Under $H_1$ at
time $\tau_k$, a separation $\Delta_k \ge 2b$ forces a fire within
delay $D_{\max} = 2n/\mu + \tau_{\mathrm{burn}}$, where the $2n/\mu$ term is
the time for the recent window to refill with post-change probes, and
$\tau_{\mathrm{burn}}$ is a short burn-in period after recovery that
prevents the detector from firing on its own freshly-reset state.
The overall regret overhead is $2R_\cA S_w\sum_k D_k = O(K
W_{\mathrm{det}})$, which is lower-order in the rate of
Proposition~\ref{thm:spsc_adaptive}.

\begin{algorithm}[H]
\caption{\textbf{SPSC-Adaptive}: unknown boundaries via CUSUM-style detection}
\label{alg:spsc_adaptive}
\begin{algorithmic}[1]
\State \textbf{Input:} probe rate $\mu\in(0,1)$, detection window $n$,
       threshold $b$, burn-in $\tau_{\mathrm{burn}}$, relearning budget $m_{\mathrm{relearn}}$, $W$, $\lambda$.
\State $\mathtt{phase}\leftarrow\textsc{recovery}$;\ $\mathtt{rec}\leftarrow 0$;\ $\widehat M_{\mathrm{acc}}\leftarrow 0$.
\For{$t=1,\dots,T$}
  \If{$\mathtt{phase} = \textsc{recovery}$}\Comment{block-probe to rebuild subspace}
      \State probe: draw $u_t$, play $x_t=u_t$, observe $y_t$; update $\widehat M_{\mathrm{acc}}$; $\mathtt{rec}\mathrel{+}= 1$.
      \If{$\mathtt{rec} = m_{\mathrm{relearn}}$}
          \State $\widehat U\leftarrow \TopEig_r(\widehat M_{\mathrm{acc}}/m_{\mathrm{relearn}})$;\ reset buffers;\ $\mathtt{phase}\leftarrow\textsc{normal}$.
      \EndIf
  \Else\Comment{interleave probes at rate $\mu$}
      \If{Bernoulli($\mu$) probe round}
          \State update lifted estimator $\widehat M_{\mathrm{acc}}$ and refresh rolling windows.
          \If{$S_t > b$ (Eq.~\ref{eq:detector}) and detector is armed}
              \State $\mathtt{phase}\leftarrow\textsc{recovery}$;\ $\mathtt{rec}\leftarrow 0$;\ reset.
          \EndIf
      \Else
          \State windowed projected ridge-UCB step as in Alg.~\ref{alg:spsc} (exploit branch).
      \EndIf
  \EndIf
\EndFor
\end{algorithmic}
\end{algorithm}

\paragraph{Tuning used in experiments.}
Across all experiments reported in \S\ref{sec:experiments} and the
appendix we use $\mu=0.1$, $n=W_{\mathrm{det}}=50$,
$\tau_{\mathrm{burn}}=100$, $m_{\mathrm{relearn}}=30$, and
$b$ set from the CUSUM threshold $\mathtt{cusum\_threshold}=3.0$.
These are held fixed across datasets and $(d,r)$ cells; no
per-dataset retuning is done. Sensitivity to the probe rate $\mu$
and the detection window $n$ is reported in
\S\ref{app:sensitivity}.

\section{Experimental setup}
\label{app:setup}

\paragraph{Synthetic phase-transition benchmark.}
$\theta_t = B_k^\star w_t$ with piecewise-constant orthonormal
$B_k^\star \in \R^{d\times r}$. Per segment, latent state evolves as
$w_t = A_k w_{t-1} + \eta_{t-1}$ with spectral radius $\rho(A_k) = 0.99$
and Gaussian innovation $\eta_{t-1}\sim\mathcal{N}(0, 0.04^2 I_r)$. Reward
noise $\varepsilon_t\sim\mathcal{N}(0,0.09)$; $40$ actions per round drawn
i.i.d.\ uniformly on the unit sphere. $K{=}10$ segments of equal length
over $T{=}5{,}000$. Probe distribution $Q$ isotropic Gaussian,
probe cost $c{=}0.1$, probe period $50$.

\paragraph{UCI datasets (Covertype, Pendigits, Satimage).}
Features are the raw attribute vectors $\phi(x)\in\R^{d_0}$ reduced to
ambient dimension $d$ via Gaussian random projection
$R\in\R^{d\times d_0}$. For Covertype, segments correspond to disjoint
size-$K$ cover-type subsets of the 7 classes; for Pendigits and
Satimage, to disjoint digit/class subsets, $K{=}10$. Actions at round
$t$ are $|\cA_t|=40$ examples drawn from the active segment's class
pool. Reward is centered target label. $\lambda{=}0.01$, probe
period $10$, window $W{=}400$, $\delta{=}0.05$, $c{=}0.02$.

\paragraph{Pixelized MNIST / Fashion-MNIST.}
Images flattened, pixel-normalized, then projected to $d$ via random
projection. Segments are class-subset swaps as above. Same
hyperparameters as UCI datasets.

\paragraph{MovieLens-100K with real ratings.}
Standard MovieLens-100K user-item matrix centered by per-item mean.
Features $\phi(\text{user},\text{item})$ are item-side latent
factors from a rank-$d$ SVD of the observed-entry mask; reward is
the raw user rating ($1$--$5$), centered. Segments correspond to
disjoint time-bucketed user cohorts, $K{=}10$. This is a
\emph{fully real} benchmark (we do \emph{not} generate rewards from
a low-rank surrogate), so the data is only approximately low-rank and
provides a stringent stress test. $T{=}5{,}000$, same hyperparameters.

\paragraph{Warfarin clinical dosing.}
IWPC cohort statistics (${\sim}5{,}700$ patients); $d{=}93$ features
(demographics, VKORC1/CYP2C9 genotype, medications, interactions).
Reward is negative absolute deviation from the therapeutic dose,
$K{=}8$ patient-cohort segments, $T{=}5{,}000$, $\lambda{=}1.0$, probe
period $10$, $W{=}200$, $c{=}0.1$, 10 seeds.

\paragraph{Baseline hyperparameters.}
All baselines inherit $\lambda$, $\delta$ from SPSC when applicable.
D-LinUCB uses discount factor $0.998$; SW-LinUCB uses window $W$;
Restart-LinUCB restarts every $T/K$ rounds (i.e.\ oracle boundaries);
LinTS/SW-LinTS use posterior variance $\sigma^2 = \delta^2$.
Oracle-LinUCB knows the true $B_k^\star$ and runs LinUCB in the
projected $\R^r$.

\paragraph{Subspace-aware baselines (LowOFUL, VOFUL, LowRank-Reward): adaptation to the piecewise-low-rank setting.}
The published versions of these methods do not directly apply in our
setting: LowOFUL~\citep{jun2019bilinear} requires oracle access to
the rank-$r$ subspace $B^\star$; the algorithms of \citet{lu2021low}
are either theoretically optimal but not efficiently implementable
(LowGLOC) or rely on low-rank matrix-completion machinery
(LowGLOC-LP) that does not transfer to our action-set bandit setup;
and the variance-aware OFUL family (VOFUL/VOFUL2) is a full-$d$
variance-aware confidence-set construction with no rank exploitation.
To obtain competitive baselines for the piecewise-low-rank regime,
each maintains its own running estimate of the segment subspace via
reward-outer-product PCA: LowOFUL accumulates
$M = \sum_t (y_t x_t)(y_t x_t)^\top$ and re-estimates the top-$r$
eigenspace every $20$ rounds (after a $30$-round warm-up);
LowRank-Reward applies the same estimator on a sliding window
$W{=}200$ with oracle resets at segment boundaries; VOFUL applies a
variance-weighted variant
$M = \sum_t w_t\,(y_t x_t)(y_t x_t)^\top$ with
$w_t = 1/(\|x_t\|^2 + \sigma_\varepsilon^2)$. These adaptations
\emph{strengthen} the baselines: granting them online subspace
recovery removes representation quality as a confound, so any
remaining gap to SPSC reflects probe-based identification, the
windowed head, and drift handling rather than failure to exploit
low-rank structure.

\paragraph{Implementation and compute.}
All experiments use NumPy/SciPy only. A full $(d,r)$ grid (9 cells,
10 seeds, nine per-cell baselines plus Oracle-LinUCB and the two
SPSC variants; BOSS and Jedra appear only in the §5.6 grid)
takes 6--20 hours on a single 16-core CPU. Covertype is the longest
($T{=}10{,}000$).

\section{Per-cell tables: every dataset, every cell}
\label{app:percell}

Tables~\ref{tab:app-covertype}--\ref{tab:app-openbandit} give the
full per-cell results (mean\,$\pm$\,SE costed regret over 10 seeds) for
every method on every $(d,r)$ cell. Cells where at least one SPSC
variant beats every non-oracle baseline are highlighted
\colorbox{green!12}{green}; the best non-oracle baseline per cell is
\underline{underlined}.

\begin{table}[!htbp]
\centering
\caption{\textbf{Covertype} per-cell results. Mean $\pm$ SE costed regret over 10 seeds. Green = SPSC variant beats every non-oracle baseline; underline = best non-oracle. $T{=}10{,}000$, $K{=}4$.
\emph{Reading note (applies to all per-cell tables in this appendix):}
rank-insensitive methods (LinUCB, SW-LinUCB, D-LinUCB, Restart-LinUCB, LinTS, SW-LinTS) act in ambient space; their regret depends on $d$ but not on $r$, hence bit-identical values across the rows of each $d$-block.}
\label{tab:app-covertype}
\scriptsize
\setlength{\tabcolsep}{3pt}
\resizebox{\textwidth}{!}{%
\begin{tabular}{llrrrrrrrrrrrr}
\toprule
$d$ & $r$ & Oracle & \textbf{SPSC} & \textbf{SPSC-Adp} & LowOFUL & VOFUL & LR-Rew & SW-Lin & D-Lin & Rst-Lin & LinUCB & LinTS & SW-LinTS \\
\midrule
5 & 1 & $2013{\pm}138$ & $2655{\pm}85$ & \cellcolor{green!12}$2114{\pm}44$ & $3748{\pm}162$ & $3736{\pm}164$ & \underline{$2379{\pm}136$} & $2651{\pm}110$ & $2502{\pm}98$ & $2661{\pm}124$ & $2694{\pm}126$ & $3633{\pm}159$ & $3064{\pm}113$ \\
55 & 1 & $718{\pm}51$ & \cellcolor{green!12}$1296{\pm}21$ & \cellcolor{green!12}$1262{\pm}32$ & $1429{\pm}32$ & $1436{\pm}32$ & $1366{\pm}44$ & $1331{\pm}20$ & $1346{\pm}21$ & \underline{$1329{\pm}21$} & $1336{\pm}21$ & $1462{\pm}25$ & $1443{\pm}25$ \\
 & 10 & $2768{\pm}29$ & \cellcolor{green!12}$4307{\pm}62$ & \cellcolor{green!12}$4161{\pm}63$ & $5067{\pm}83$ & $5167{\pm}130$ & $5122{\pm}61$ & $4403{\pm}41$ & $4599{\pm}57$ & \underline{$4348{\pm}37$} & $4465{\pm}49$ & $5591{\pm}78$ & $5486{\pm}73$ \\
 & 20 & $4004{\pm}55$ & $5769{\pm}54$ & $5786{\pm}84$ & $7120{\pm}93$ & $7001{\pm}72$ & $7538{\pm}81$ & $5703{\pm}34$ & $5971{\pm}40$ & \underline{$5669{\pm}40$} & $5800{\pm}50$ & $7825{\pm}41$ & $7581{\pm}44$ \\
 & 30 & $5202{\pm}48$ & $6961{\pm}95$ & $7051{\pm}81$ & $8890{\pm}110$ & $8770{\pm}93$ & $9527{\pm}96$ & $6653{\pm}34$ & $6928{\pm}30$ & \underline{$6594{\pm}51$} & $6743{\pm}34$ & $9615{\pm}105$ & $9345{\pm}90$ \\
105 & 1 & $554{\pm}30$ & \cellcolor{green!12}$1023{\pm}29$ & \cellcolor{green!12}$1010{\pm}21$ & $1134{\pm}41$ & $1124{\pm}44$ & \underline{$1026{\pm}28$} & $1075{\pm}32$ & $1082{\pm}33$ & $1073{\pm}31$ & $1075{\pm}33$ & $1109{\pm}35$ & $1104{\pm}35$ \\
 & 10 & $2154{\pm}27$ & \cellcolor{green!12}$3452{\pm}45$ & \cellcolor{green!12}$3405{\pm}40$ & $3775{\pm}34$ & \underline{$3758{\pm}48$} & $3869{\pm}36$ & $3920{\pm}36$ & $3959{\pm}36$ & $3909{\pm}37$ & $3943{\pm}35$ & $4198{\pm}40$ & $4168{\pm}37$ \\
 & 20 & $3233{\pm}36$ & \cellcolor{green!12}$4600{\pm}62$ & \cellcolor{green!12}$4507{\pm}46$ & \underline{$5297{\pm}67$} & $5325{\pm}68$ & $5621{\pm}43$ & $5302{\pm}40$ & $5383{\pm}32$ & $5313{\pm}30$ & $5355{\pm}35$ & $5826{\pm}42$ & $5734{\pm}42$ \\
 & 30 & $4165{\pm}23$ & \cellcolor{green!12}$5529{\pm}71$ & \cellcolor{green!12}$5450{\pm}67$ & $6476{\pm}33$ & $6556{\pm}51$ & $6989{\pm}54$ & $6360{\pm}38$ & $6473{\pm}33$ & \underline{$6324{\pm}42$} & $6400{\pm}37$ & $7085{\pm}48$ & $7029{\pm}46$ \\
\bottomrule
\end{tabular}%
}
\end{table}

\begin{table}[!htbp]
\centering
\caption{\textbf{Pendigits} per-cell results. $T{=}5{,}000$, $K{=}10$, 10 seeds.}
\label{tab:app-pendigits}
\scriptsize
\setlength{\tabcolsep}{3pt}
\resizebox{\textwidth}{!}{%
\begin{tabular}{llrrrrrrrrrrrr}
\toprule
$d$ & $r$ & Oracle & \textbf{SPSC} & \textbf{SPSC-Adp} & LowOFUL & VOFUL & LR-Rew & SW-Lin & D-Lin & Rst-Lin & LinUCB & LinTS & SW-LinTS \\
\midrule
5 & 1 & $39.0{\pm}2.0$ & $2914{\pm}159$ & $5480{\pm}391$ & $9536{\pm}199$ & $9623{\pm}138$ & $6606{\pm}77$ & $1090{\pm}9$ & $1267{\pm}8$ & \underline{$1022{\pm}14$} & $1073{\pm}15$ & $6629{\pm}42$ & $4630{\pm}41$ \\
55 & 1 & $18.0{\pm}1.0$ & $4019{\pm}79$ & $4073{\pm}200$ & $5913{\pm}46$ & $5939{\pm}64$ & $4542{\pm}48$ & $3863{\pm}11$ & $4056{\pm}9$ & \underline{$3620{\pm}7$} & $3658{\pm}9$ & $4343{\pm}16$ & $4648{\pm}9$ \\
 & 10 & $951{\pm}5$ & \cellcolor{green!12}$2774{\pm}58$ & \cellcolor{green!12}$2682{\pm}93$ & $3703{\pm}71$ & $3655{\pm}67$ & $4246{\pm}28$ & $3863{\pm}11$ & $4056{\pm}9$ & \underline{$3620{\pm}7$} & $3658{\pm}9$ & $4343{\pm}16$ & $4648{\pm}9$ \\
 & 20 & $1816{\pm}4$ & \cellcolor{green!12}$2803{\pm}60$ & \cellcolor{green!12}$2490{\pm}56$ & $3689{\pm}54$ & $3854{\pm}58$ & $4271{\pm}39$ & $3863{\pm}11$ & $4056{\pm}9$ & \underline{$3620{\pm}7$} & $3658{\pm}9$ & $4343{\pm}16$ & $4648{\pm}9$ \\
 & 30 & $2534{\pm}6$ & \cellcolor{green!12}$3051{\pm}62$ & \cellcolor{green!12}$2776{\pm}45$ & $3827{\pm}45$ & $3818{\pm}39$ & $4329{\pm}20$ & $3863{\pm}11$ & $4056{\pm}9$ & \underline{$3620{\pm}7$} & $3658{\pm}9$ & $4343{\pm}16$ & $4648{\pm}9$ \\
105 & 1 & $11.0{\pm}0.0$ & \cellcolor{green!12}$3011{\pm}88$ & $3120{\pm}134$ & $4254{\pm}114$ & $4277{\pm}107$ & \underline{$3051{\pm}89$} & $3297{\pm}6$ & $3400{\pm}6$ & $3152{\pm}6$ & $3184{\pm}5$ & $3117{\pm}10$ & $3557{\pm}4$ \\
 & 10 & $666{\pm}3$ & \cellcolor{green!12}$2066{\pm}64$ & \cellcolor{green!12}$1693{\pm}38$ & $2927{\pm}74$ & \underline{$2816{\pm}55$} & $3027{\pm}43$ & $3297{\pm}6$ & $3400{\pm}6$ & $3152{\pm}6$ & $3184{\pm}5$ & $3117{\pm}10$ & $3557{\pm}4$ \\
 & 20 & $1242{\pm}3$ & \cellcolor{green!12}$2067{\pm}62$ & \cellcolor{green!12}$1571{\pm}67$ & $2840{\pm}41$ & \underline{$2835{\pm}35$} & $3123{\pm}18$ & $3297{\pm}6$ & $3400{\pm}6$ & $3152{\pm}6$ & $3184{\pm}5$ & $3117{\pm}10$ & $3557{\pm}4$ \\
 & 30 & $1745{\pm}4$ & \cellcolor{green!12}$2247{\pm}58$ & \cellcolor{green!12}$1722{\pm}58$ & \underline{$2826{\pm}75$} & $2839{\pm}56$ & $3283{\pm}19$ & $3297{\pm}6$ & $3400{\pm}6$ & $3152{\pm}6$ & $3184{\pm}5$ & $3117{\pm}10$ & $3557{\pm}4$ \\
\bottomrule
\end{tabular}%
}
\end{table}

\begin{table}[!htbp]
\centering
\caption{\textbf{Satimage} per-cell results. $T{=}5{,}000$, $K{=}10$, 10 seeds.}
\label{tab:app-satimage}
\scriptsize
\setlength{\tabcolsep}{3pt}
\resizebox{\textwidth}{!}{%
\begin{tabular}{llrrrrrrrrrrrr}
\toprule
$d$ & $r$ & Oracle & \textbf{SPSC} & \textbf{SPSC-Adp} & LowOFUL & VOFUL & LR-Rew & SW-Lin & D-Lin & Rst-Lin & LinUCB & LinTS & SW-LinTS \\
\midrule
5 & 1 & $35.0{\pm}2.0$ & $1480{\pm}88$ & $1513{\pm}123$ & $2739{\pm}17$ & $2704{\pm}18$ & $2660{\pm}24$ & $655{\pm}7$ & $778{\pm}6$ & \underline{$629{\pm}6$} & $680{\pm}6$ & $2397{\pm}14$ & $1675{\pm}16$ \\
55 & 1 & $18.0{\pm}1.0$ & \cellcolor{green!12}$2439{\pm}81$ & \cellcolor{green!12}$2015{\pm}75$ & $2913{\pm}11$ & $2911{\pm}8$ & $3450{\pm}52$ & $3051{\pm}6$ & $3201{\pm}9$ & \underline{$2779{\pm}3$} & $2848{\pm}9$ & $3403{\pm}11$ & $3781{\pm}10$ \\
 & 10 & $597{\pm}3$ & \cellcolor{green!12}$1343{\pm}41$ & \cellcolor{green!12}$1073{\pm}49$ & $2294{\pm}99$ & \underline{$2288{\pm}96$} & $2689{\pm}62$ & $3051{\pm}6$ & $3201{\pm}9$ & $2779{\pm}3$ & $2848{\pm}9$ & $3403{\pm}11$ & $3781{\pm}10$ \\
 & 20 & $1218{\pm}4$ & \cellcolor{green!12}$1506{\pm}34$ & \cellcolor{green!12}$1201{\pm}33$ & \underline{$2455{\pm}106$} & $2479{\pm}65$ & $2879{\pm}55$ & $3051{\pm}6$ & $3201{\pm}9$ & $2779{\pm}3$ & $2848{\pm}9$ & $3403{\pm}11$ & $3781{\pm}10$ \\
 & 30 & $1725{\pm}8$ & \cellcolor{green!12}$1900{\pm}38$ & \cellcolor{green!12}$1635{\pm}40$ & \underline{$2641{\pm}54$} & $2772{\pm}47$ & $3085{\pm}53$ & $3051{\pm}6$ & $3201{\pm}9$ & $2779{\pm}3$ & $2848{\pm}9$ & $3403{\pm}11$ & $3781{\pm}10$ \\
105 & 1 & $16.0{\pm}1.0$ & \cellcolor{green!12}$2484{\pm}59$ & \cellcolor{green!12}$2333{\pm}108$ & $3638{\pm}7$ & $3655{\pm}12$ & $3911{\pm}47$ & $3572{\pm}8$ & $3671{\pm}9$ & \underline{$3441{\pm}6$} & $3484{\pm}10$ & $3600{\pm}9$ & $3891{\pm}7$ \\
 & 10 & $538{\pm}3$ & \cellcolor{green!12}$1377{\pm}44$ & \cellcolor{green!12}$1112{\pm}49$ & $2431{\pm}87$ & \underline{$2361{\pm}78$} & $2703{\pm}40$ & $3572{\pm}8$ & $3671{\pm}9$ & $3441{\pm}6$ & $3484{\pm}10$ & $3600{\pm}9$ & $3891{\pm}7$ \\
 & 20 & $1098{\pm}5$ & \cellcolor{green!12}$1454{\pm}33$ & \cellcolor{green!12}$1182{\pm}34$ & \underline{$2229{\pm}57$} & $2346{\pm}46$ & $2697{\pm}39$ & $3572{\pm}8$ & $3671{\pm}9$ & $3441{\pm}6$ & $3484{\pm}10$ & $3600{\pm}9$ & $3891{\pm}7$ \\
 & 30 & $1569{\pm}5$ & \cellcolor{green!12}$1769{\pm}32$ & \cellcolor{green!12}$1506{\pm}33$ & \underline{$2525{\pm}42$} & $2587{\pm}80$ & $2722{\pm}25$ & $3572{\pm}8$ & $3671{\pm}9$ & $3441{\pm}6$ & $3484{\pm}10$ & $3600{\pm}9$ & $3891{\pm}7$ \\
\bottomrule
\end{tabular}%
}
\end{table}

\begin{table}[!htbp]
\centering
\caption{\textbf{Fashion-MNIST} per-cell results. $T{=}5{,}000$, $K{=}10$, 10 seeds.}
\label{tab:app-fmnist}
\scriptsize
\setlength{\tabcolsep}{3pt}
\resizebox{\textwidth}{!}{%
\begin{tabular}{llrrrrrrrrrrrr}
\toprule
$d$ & $r$ & Oracle & \textbf{SPSC} & \textbf{SPSC-Adp} & LowOFUL & VOFUL & LR-Rew & SW-Lin & D-Lin & Rst-Lin & LinUCB & LinTS & SW-LinTS \\
\midrule
55 & 5 & $539{\pm}6$ & \cellcolor{green!12}$3773{\pm}91$ & \cellcolor{green!12}$3792{\pm}58$ & $4487{\pm}103$ & $4475{\pm}118$ & $4719{\pm}93$ & $4169{\pm}9$ & $4324{\pm}11$ & \underline{$3801{\pm}6$} & $3868{\pm}12$ & $5183{\pm}11$ & $5357{\pm}15$ \\
 & 10 & $945{\pm}5$ & \cellcolor{green!12}$3406{\pm}73$ & \cellcolor{green!12}$3377{\pm}92$ & $4425{\pm}126$ & $4471{\pm}142$ & $4740{\pm}59$ & $4169{\pm}9$ & $4324{\pm}11$ & \underline{$3801{\pm}6$} & $3868{\pm}12$ & $5183{\pm}11$ & $5357{\pm}15$ \\
 & 20 & $1676{\pm}6$ & \cellcolor{green!12}$3450{\pm}62$ & \cellcolor{green!12}$3162{\pm}81$ & $4677{\pm}119$ & $4708{\pm}75$ & $4924{\pm}56$ & $4169{\pm}9$ & $4324{\pm}11$ & \underline{$3801{\pm}6$} & $3868{\pm}12$ & $5183{\pm}11$ & $5357{\pm}15$ \\
105 & 5 & $399{\pm}3$ & \cellcolor{green!12}$3200{\pm}88$ & \cellcolor{green!12}$3208{\pm}60$ & \underline{$3365{\pm}71$} & $3522{\pm}92$ & $3688{\pm}81$ & $4152{\pm}8$ & $4236{\pm}7$ & $4016{\pm}8$ & $4034{\pm}10$ & $4376{\pm}9$ & $4496{\pm}11$ \\
 & 10 & $757{\pm}6$ & \cellcolor{green!12}$2927{\pm}76$ & \cellcolor{green!12}$2623{\pm}84$ & \underline{$3331{\pm}42$} & $3478{\pm}51$ & $3645{\pm}65$ & $4152{\pm}8$ & $4236{\pm}7$ & $4016{\pm}8$ & $4034{\pm}10$ & $4376{\pm}9$ & $4496{\pm}11$ \\
 & 20 & $1439{\pm}4$ & \cellcolor{green!12}$2855{\pm}75$ & \cellcolor{green!12}$2528{\pm}59$ & \underline{$3559{\pm}80$} & $3564{\pm}45$ & $3797{\pm}38$ & $4152{\pm}8$ & $4236{\pm}7$ & $4016{\pm}8$ & $4034{\pm}10$ & $4376{\pm}9$ & $4496{\pm}11$ \\
200 & 5 & $393{\pm}3$ & $2949{\pm}52$ & $2877{\pm}89$ & $3134{\pm}72$ & \underline{$2869{\pm}58$} & $2948{\pm}47$ & $4099{\pm}7$ & $4158{\pm}8$ & $4059{\pm}5$ & $4059{\pm}11$ & $4101{\pm}11$ & $4206{\pm}8$ \\
 & 10 & $730{\pm}4$ & \cellcolor{green!12}$2703{\pm}52$ & $2732{\pm}101$ & $2929{\pm}43$ & \underline{$2731{\pm}63$} & $2967{\pm}26$ & $4099{\pm}7$ & $4158{\pm}8$ & $4059{\pm}5$ & $4059{\pm}11$ & $4101{\pm}11$ & $4206{\pm}8$ \\
 & 20 & $1290{\pm}5$ & \cellcolor{green!12}$2654{\pm}52$ & \cellcolor{green!12}$2582{\pm}67$ & $3215{\pm}56$ & \underline{$3196{\pm}61$} & $3341{\pm}44$ & $4099{\pm}7$ & $4158{\pm}8$ & $4059{\pm}5$ & $4059{\pm}11$ & $4101{\pm}11$ & $4206{\pm}8$ \\
\bottomrule
\end{tabular}%
}
\end{table}

\begin{table}[!htbp]
\centering
\caption{\textbf{MNIST} per-cell results. $T{=}5{,}000$, $K{=}10$, 10 seeds.}
\label{tab:app-mnist}
\scriptsize
\setlength{\tabcolsep}{3pt}
\resizebox{\textwidth}{!}{%
\begin{tabular}{llrrrrrrrrrrrr}
\toprule
$d$ & $r$ & Oracle & \textbf{SPSC} & \textbf{SPSC-Adp} & LowOFUL & VOFUL & LR-Rew & SW-Lin & D-Lin & Rst-Lin & LinUCB & LinTS & SW-LinTS \\
\midrule
55 & 5 & $666{\pm}5$ & \cellcolor{green!12}$5076{\pm}88$ & \cellcolor{green!12}$5223{\pm}109$ & $6437{\pm}152$ & $6646{\pm}62$ & $6650{\pm}116$ & $5889{\pm}16$ & $6154{\pm}12$ & \underline{$5426{\pm}10$} & $5533{\pm}9$ & $6958{\pm}13$ & $7362{\pm}20$ \\
 & 10 & $1149{\pm}9$ & \cellcolor{green!12}$4648{\pm}77$ & \cellcolor{green!12}$4968{\pm}93$ & $6392{\pm}114$ & $6191{\pm}118$ & $6623{\pm}77$ & $5889{\pm}16$ & $6154{\pm}12$ & \underline{$5426{\pm}10$} & $5533{\pm}9$ & $6958{\pm}13$ & $7362{\pm}20$ \\
 & 20 & $2196{\pm}11$ & \cellcolor{green!12}$4745{\pm}76$ & \cellcolor{green!12}$4549{\pm}82$ & $6420{\pm}63$ & $6352{\pm}120$ & $6745{\pm}58$ & $5889{\pm}16$ & $6154{\pm}12$ & \underline{$5426{\pm}10$} & $5533{\pm}9$ & $6958{\pm}13$ & $7362{\pm}20$ \\
105 & 5 & $582{\pm}5$ & \cellcolor{green!12}$4658{\pm}104$ & \cellcolor{green!12}$4844{\pm}68$ & \underline{$5184{\pm}74$} & $5305{\pm}56$ & $5807{\pm}93$ & $5992{\pm}11$ & $6073{\pm}8$ & $5749{\pm}11$ & $5811{\pm}10$ & $5892{\pm}18$ & $6222{\pm}11$ \\
 & 10 & $1029{\pm}5$ & \cellcolor{green!12}$4298{\pm}96$ & \cellcolor{green!12}$4229{\pm}157$ & \underline{$4867{\pm}96$} & $4917{\pm}46$ & $5555{\pm}58$ & $5992{\pm}11$ & $6073{\pm}8$ & $5749{\pm}11$ & $5811{\pm}10$ & $5892{\pm}18$ & $6222{\pm}11$ \\
 & 20 & $1869{\pm}11$ & \cellcolor{green!12}$4273{\pm}101$ & \cellcolor{green!12}$4076{\pm}114$ & $5020{\pm}46$ & \underline{$4918{\pm}78$} & $5651{\pm}46$ & $5992{\pm}11$ & $6073{\pm}8$ & $5749{\pm}11$ & $5811{\pm}10$ & $5892{\pm}18$ & $6222{\pm}11$ \\
200 & 5 & $539{\pm}2$ & \cellcolor{green!12}$4377{\pm}84$ & \cellcolor{green!12}$4260{\pm}81$ & \underline{$4672{\pm}62$} & $4718{\pm}113$ & $4797{\pm}52$ & $5751{\pm}10$ & $5786{\pm}6$ & $5659{\pm}11$ & $5661{\pm}8$ & $5568{\pm}10$ & $5774{\pm}7$ \\
 & 10 & $907{\pm}5$ & \cellcolor{green!12}$4114{\pm}66$ & \cellcolor{green!12}$4023{\pm}113$ & \underline{$4184{\pm}63$} & $4375{\pm}75$ & $4818{\pm}42$ & $5751{\pm}10$ & $5786{\pm}6$ & $5659{\pm}11$ & $5661{\pm}8$ & $5568{\pm}10$ & $5774{\pm}7$ \\
 & 20 & $1675{\pm}6$ & \cellcolor{green!12}$4052{\pm}59$ & \cellcolor{green!12}$3828{\pm}69$ & \underline{$4379{\pm}34$} & $4413{\pm}50$ & $5022{\pm}27$ & $5751{\pm}10$ & $5786{\pm}6$ & $5659{\pm}11$ & $5661{\pm}8$ & $5568{\pm}10$ & $5774{\pm}7$ \\
\bottomrule
\end{tabular}%
}
\end{table}

\begin{table}[!htbp]
\centering
\caption{\textbf{MovieLens} (fully real ratings) per-cell results. $T{=}5{,}000$, $K{=}10$, 10 seeds.}
\label{tab:app-movielens}
\scriptsize
\setlength{\tabcolsep}{3pt}
\resizebox{\textwidth}{!}{%
\begin{tabular}{llrrrrrrrrrrrr}
\toprule
$d$ & $r$ & Oracle & \textbf{SPSC} & \textbf{SPSC-Adp} & LowOFUL & VOFUL & LR-Rew & SW-Lin & D-Lin & Rst-Lin & LinUCB & LinTS & SW-LinTS \\
\midrule
55 & 5 & $4013{\pm}33$ & \cellcolor{green!12}$6213{\pm}63$ & \cellcolor{green!12}$6060{\pm}77$ & $7676{\pm}164$ & $7504{\pm}143$ & $8103{\pm}60$ & $6313{\pm}6$ & $6363{\pm}6$ & \underline{$6239{\pm}10$} & $6258{\pm}8$ & $6396{\pm}11$ & $6587{\pm}11$ \\
 & 10 & $4802{\pm}24$ & $6250{\pm}46$ & \cellcolor{green!12}$6015{\pm}37$ & $7689{\pm}77$ & $7461{\pm}113$ & $7855{\pm}47$ & $6313{\pm}6$ & $6363{\pm}6$ & \underline{$6239{\pm}10$} & $6258{\pm}8$ & $6396{\pm}11$ & $6587{\pm}11$ \\
 & 20 & $5402{\pm}21$ & $6266{\pm}42$ & \cellcolor{green!12}$6216{\pm}44$ & $7407{\pm}56$ & $7329{\pm}108$ & $7391{\pm}30$ & $6313{\pm}6$ & $6363{\pm}6$ & \underline{$6239{\pm}10$} & $6258{\pm}8$ & $6396{\pm}11$ & $6587{\pm}11$ \\
105 & 5 & $3970{\pm}33$ & \cellcolor{green!12}$6535{\pm}56$ & \cellcolor{green!12}$6350{\pm}73$ & $8276{\pm}126$ & $8292{\pm}89$ & $8340{\pm}35$ & $6809{\pm}5$ & $6850{\pm}4$ & $6783{\pm}4$ & $6793{\pm}6$ & \underline{$6743{\pm}13$} & $6874{\pm}12$ \\
 & 10 & $4951{\pm}37$ & \cellcolor{green!12}$6560{\pm}43$ & \cellcolor{green!12}$6297{\pm}55$ & $7801{\pm}173$ & $7927{\pm}154$ & $8200{\pm}61$ & $6809{\pm}5$ & $6850{\pm}4$ & $6783{\pm}4$ & $6793{\pm}6$ & \underline{$6743{\pm}13$} & $6874{\pm}12$ \\
 & 20 & $5697{\pm}23$ & \cellcolor{green!12}$6636{\pm}32$ & \cellcolor{green!12}$6566{\pm}40$ & $7562{\pm}103$ & $7688{\pm}90$ & $7674{\pm}27$ & $6809{\pm}5$ & $6850{\pm}4$ & $6783{\pm}4$ & $6793{\pm}6$ & \underline{$6743{\pm}13$} & $6874{\pm}12$ \\
200 & 5 & $3939{\pm}27$ & \cellcolor{green!12}$6787{\pm}52$ & \cellcolor{green!12}$6516{\pm}101$ & $8365{\pm}102$ & $8450{\pm}118$ & $8420{\pm}46$ & $7396{\pm}6$ & $7421{\pm}6$ & $7404{\pm}4$ & $7417{\pm}7$ & \underline{$7134{\pm}12$} & $7203{\pm}7$ \\
 & 10 & $4998{\pm}48$ & \cellcolor{green!12}$6844{\pm}57$ & \cellcolor{green!12}$6576{\pm}36$ & $8383{\pm}131$ & $8163{\pm}135$ & $8513{\pm}50$ & $7396{\pm}6$ & $7421{\pm}6$ & $7404{\pm}4$ & $7417{\pm}7$ & \underline{$7134{\pm}12$} & $7203{\pm}7$ \\
 & 20 & $5871{\pm}25$ & \cellcolor{green!12}$6995{\pm}47$ & \cellcolor{green!12}$6844{\pm}53$ & $7296{\pm}118$ & $7605{\pm}117$ & $7627{\pm}58$ & $7396{\pm}6$ & $7421{\pm}6$ & $7404{\pm}4$ & $7417{\pm}7$ & \underline{$7134{\pm}12$} & $7203{\pm}7$ \\
\bottomrule
\end{tabular}%
}
\end{table}

\begin{table}[!htbp]
\centering
\caption{\textbf{Warfarin} ($d{=}93$, $K{=}8$, $T{=}5{,}000$) per-rank results across $r\in\{1,2,3,5,10\}$. 10 seeds.}
\label{tab:app-warfarin}
\scriptsize
\setlength{\tabcolsep}{3pt}
\resizebox{\textwidth}{!}{%
\begin{tabular}{llrrrrrrrrrrrr}
\toprule
$d$ & $r$ & Oracle & \textbf{SPSC} & \textbf{SPSC-Adp} & LowOFUL & VOFUL & LR-Rew & SW-Lin & D-Lin & Rst-Lin & LinUCB & LinTS & SW-LinTS \\
\midrule
93 & 1 & $6.4{\pm}0.7$ & \cellcolor{green!12}$1482{\pm}49$ & \cellcolor{green!12}$626{\pm}44$ & \underline{$1491{\pm}26$} & $1497{\pm}23$ & $1533{\pm}20$ & $2096{\pm}18$ & $2175{\pm}22$ & $1977{\pm}15$ & $2006{\pm}17$ & $1949{\pm}20$ & $2060{\pm}22$ \\
 & 2 & $93.9{\pm}1.3$ & $1404{\pm}35$ & \cellcolor{green!12}$706{\pm}37$ & $1076{\pm}130$ & \underline{$918{\pm}127$} & $1140{\pm}25$ & $2096{\pm}18$ & $2175{\pm}22$ & $1977{\pm}15$ & $2006{\pm}17$ & $1949{\pm}20$ & $2060{\pm}22$ \\
 & 3 & $177{\pm}4$ & $1372{\pm}33$ & \cellcolor{green!12}$657{\pm}39$ & \underline{$683{\pm}97$} & $779{\pm}93$ & $1056{\pm}40$ & $2096{\pm}18$ & $2175{\pm}22$ & $1977{\pm}15$ & $2006{\pm}17$ & $1949{\pm}20$ & $2060{\pm}22$ \\
 & 5 & $272{\pm}7$ & $1369{\pm}30$ & \cellcolor{green!12}$769{\pm}37$ & \underline{$905{\pm}110$} & $930{\pm}76$ & $1178{\pm}32$ & $2096{\pm}18$ & $2175{\pm}22$ & $1977{\pm}15$ & $2006{\pm}17$ & $1949{\pm}20$ & $2060{\pm}22$ \\
 & 10 & $569{\pm}4$ & $1481{\pm}19$ & \cellcolor{green!12}$1014{\pm}35$ & \underline{$1441{\pm}63$} & $1680{\pm}59$ & $2078{\pm}19$ & $2096{\pm}18$ & $2175{\pm}22$ & $1977{\pm}15$ & $2006{\pm}17$ & $1949{\pm}20$ & $2060{\pm}22$ \\
\bottomrule
\end{tabular}%
}
\end{table}

\begin{table}[!htbp]
\centering
\caption{\textbf{Open Bandit (ZOZOTOWN)} per-cell results. $T{=}5{,}000$, $K{=}10$, 10 seeds. Mean $\pm$ SE costed regret. Green = SPSC variant beats every non-oracle baseline; underline = best non-oracle. LinTS / SW-LinTS not run on this benchmark.}
\label{tab:app-openbandit}
\scriptsize
\setlength{\tabcolsep}{3pt}
\resizebox{\textwidth}{!}{%
\begin{tabular}{llrrrrrrrrrr}
\toprule
$d$ & $r$ & Oracle & \textbf{SPSC} & \textbf{SPSC-Adp} & LowOFUL & VOFUL & LR-Rew & SW-Lin & D-Lin & Rst-Lin & LinUCB \\
\midrule
55  & 5  & $34.2{\pm}0.5$ & \cellcolor{green!12}$83.0{\pm}2.3$ & \cellcolor{green!12}$91.1{\pm}4.0$ & $126.2{\pm}1.8$ & $126.7{\pm}1.5$ & $126.8{\pm}1.2$ & $103.0{\pm}0.5$ & $106.6{\pm}0.5$ & \underline{$101.9{\pm}0.6$} & $102.2{\pm}0.5$ \\
    & 10 & $65.2{\pm}0.8$ & \cellcolor{green!12}$93.9{\pm}1.5$ & \cellcolor{green!12}$94.1{\pm}1.7$ & $129.3{\pm}2.6$ & $130.2{\pm}1.9$ & $130.9{\pm}1.3$ & $103.0{\pm}0.5$ & $106.6{\pm}0.5$ & \underline{$101.9{\pm}0.6$} & $102.2{\pm}0.5$ \\
105 & 5  & $29.7{\pm}0.5$ & $106.6{\pm}4.8$ & \cellcolor{green!12}$103.3{\pm}2.9$ & $152.6{\pm}2.3$ & $150.1{\pm}1.9$ & $150.9{\pm}1.6$ & $107.4{\pm}1.1$ & $108.4{\pm}0.4$ & \underline{$106.4{\pm}0.5$} & $107.7{\pm}1.0$ \\
    & 10 & $59.0{\pm}0.7$ & $106.9{\pm}3.0$ & \cellcolor{green!12}$106.3{\pm}3.2$ & $153.3{\pm}1.0$ & $155.4{\pm}1.8$ & $156.1{\pm}1.5$ & $107.4{\pm}1.1$ & $108.4{\pm}0.4$ & \underline{$106.4{\pm}0.5$} & $107.7{\pm}1.0$ \\
\bottomrule
\end{tabular}%
}
\end{table}

\begin{figure}[!htbp]
\centering
\includegraphics[width=\textwidth]{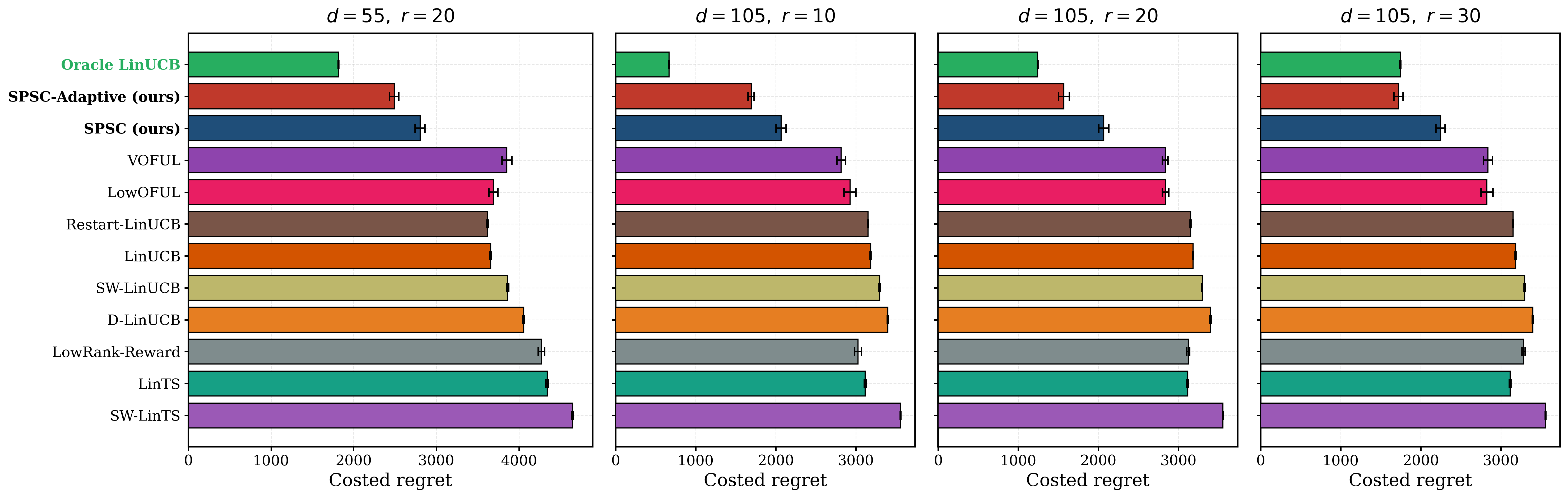}
\caption{\textbf{Pendigits operating-regime grid.}
SPSC and SPSC-Adaptive dominate every non-oracle baseline once
$d{\ge}55$ and $r{\ge}10$.}
\label{fig:pendigits-regime}
\end{figure}

\paragraph{Reading the tables.}
Three qualitative patterns are consistent across datasets.
(i) \emph{Rank-insensitive competitors.} LinUCB, SW-LinUCB, D-LinUCB
and Restart-LinUCB operate in ambient space and therefore produce the
\emph{same} regret in every row of a $d$ block. When a row shows all
four of these at identical values, this is by design: it reflects the
$r$-insensitivity of ambient-rate baselines.
(ii) \emph{Subspace-aware competitors} (LowOFUL, VOFUL, LowRank-Reward)
do improve with $d$ increasing but they assume a fixed subspace and
therefore still trail SPSC on segments where the subspace changes.
(iii) \emph{Within the synthetic and UCI grids, SPSC-Adaptive
typically trails SPSC-Alg.\,1 by a small margin}; this gap reflects
the detector's false-alarm budget. The ordering reverses on the
clinical benchmarks (most strikingly Vancomycin at $r{=}1$, where
SPSC-Adaptive cuts regret by $37\%$ vs.\ SPSC-Alg.\,1; see
Table~\ref{tab:vancomycin_full}), where the oracle schedule of
Alg.~\ref{alg:spsc} is misaligned with the true cohort change points
and the detector's adaptivity wins.

\paragraph{Synthetic $(d,r)$ grid.}
Table~\ref{tab:app-synthetic} gives the full synthetic grid underlying
Figure~\ref{fig:synthetic-phase}. SPSC wins on $37/40$ cells; losses
are confined to $d\le 10$ or $r\le 1$ (probe-cost dominant).

 \begin{table}[!t]
  \centering
  \caption{\textbf{Synthetic grid (full 40 cells).} Mean costed regret
  across 10 seeds for 10 methods over
  $(d,r)\in\{5,10,20,30,45,60,80,100\}\times\{1,3,5,10,15,20\}$ with
  $r<d$, $T{=}5{,}000$, $K{=}10$, probe period $50$, window $W{=}400$,
  $\lambda{=}0.01$. \textbf{Methods:} SPSC (Algorithm~\ref{alg:spsc},
  ours); \textbf{Adpt} = SPSC-Adaptive (Algorithm~\ref{alg:spsc_adaptive},
  ours); LinUCB; \textbf{Oracle} = Oracle-LinUCB (given the true
  subspace); \textbf{SW-Lin} = SW-LinUCB; \textbf{D-Lin} = D-LinUCB;
  \textbf{Rst} = Restart-LinUCB; \textbf{LR-Rew} = LowRank-Reward;
  LowOFUL; VOFUL. $\mathrm{S/L}=\mathrm{SPSC/LinUCB}$;
  $\mathrm{S/Best}=\mathrm{SPSC}/\min\{\text{non-oracle, non-Adpt
  baselines}\}$. \textbf{Verdict.} ``SPSC'' if $\mathrm{S/L}<1$ (SPSC
  has lower regret than LinUCB); ``LinUCB'' otherwise. Bold = SPSC
  verdict. Counts: \textbf{37 SPSC}, \textbf{3 LinUCB}; the 3
  LinUCB cells are all in the small-$d$ corner $d\le 10$ with $r$
  close to $d$, where the probe cost dominates the dimensionality
  savings.}
  \label{tab:app-synthetic}
  \scriptsize
  \setlength{\tabcolsep}{2pt}
  \begin{tabular}{rr|rrrrrrrrrr|rrl}
  \toprule
  $d$ & $r$ & SPSC & Adpt & LinUCB & Oracle & SW-Lin & D-Lin & Rst & LR-Rew & LowOFUL & VOFUL & S/L & S/Best & verdict\\
  \midrule
    5 &  1 &  355 &  344 &  203 &    8 &  213 &  234 &  206 &  416 &  514 &  553 & 1.742 & 1.742 & LinUCB \\
    5 &  3 &  339 &  384 &  221 &  134 &  221 &  254 &  217 &  629 &  930 &  889 & 1.529 & 1.563 & LinUCB \\
   10 &  1 &  270 &  289 &  255 &   24 &  260 &  271 &  258 &  254 &  317 &  315 & 1.055 & 1.060 & LinUCB \\
   10 &  3 &  \textbf{389} &  422 &  399 &  130 &  391 &  428 &  386 &  456 &  629 &  595 & \textbf{0.976} & 1.008 & \textbf{SPSC} \\
   10 &  5 &  \textbf{465} &  525 &  487 &  254 &  482 &  523 &  471 &  660 &  800 &  823 & \textbf{0.955} & 0.988 & \textbf{SPSC} \\
   20 &  1 &  \textbf{231} &  228 &  251 &   21 &  251 &  253 &  248 &  179 &  238 &  240 & \textbf{0.923} & 1.290 & \textbf{SPSC} \\
   20 &  3 &  \textbf{353} &  377 &  456 &  134 &  449 &  467 &  451 &  344 &  449 &  475 & \textbf{0.774} & 1.026 & \textbf{SPSC} \\
   20 &  5 &  \textbf{420} &  453 &  565 &  232 &  555 &  581 &  558 &  468 &  575 &  578 & \textbf{0.744} & \textbf{0.898} & \textbf{SPSC}
  \\
   20 & 10 &  \textbf{631} &  699 &  800 &  498 &  780 &  840 &  784 &  804 &  946 &  975 & \textbf{0.789} & \textbf{0.809} & \textbf{SPSC}
  \\
   20 & 15 &  \textbf{805} &  901 &  899 &  736 &  893 &  954 &  883 & 1011 & 1235 & 1222 & \textbf{0.895} & \textbf{0.912} & \textbf{SPSC}
  \\
   30 &  1 &  \textbf{189} &  206 &  220 &   24 &  220 &  222 &  219 &  154 &  198 &  203 & \textbf{0.859} & 1.225 & \textbf{SPSC} \\
   30 &  3 &  \textbf{307} &  330 &  407 &  125 &  409 &  413 &  404 &  296 &  355 &  355 & \textbf{0.754} & 1.038 & \textbf{SPSC} \\
   30 &  5 &  \textbf{372} &  421 &  544 &  217 &  541 &  553 &  536 &  411 &  498 &  492 & \textbf{0.683} & \textbf{0.904} & \textbf{SPSC}
  \\
   30 & 10 &  \textbf{572} &  615 &  775 &  451 &  763 &  790 &  760 &  690 &  855 &  828 & \textbf{0.738} & \textbf{0.829} & \textbf{SPSC}
  \\
   30 & 15 &  \textbf{741} &  810 &  930 &  663 &  917 &  957 &  913 &  956 & 1066 & 1113 & \textbf{0.796} & \textbf{0.811} & \textbf{SPSC}
  \\
   30 & 20 &  \textbf{890} &  960 & 1018 &  838 & 1002 & 1056 & 1009 & 1109 & 1236 & 1259 & \textbf{0.874} & \textbf{0.888} & \textbf{SPSC}
  \\
   45 &  1 &  \textbf{165} &  169 &  188 &   25 &  187 &  189 &  186 &  130 &  172 &  168 & \textbf{0.881} & 1.275 & \textbf{SPSC} \\
   45 &  3 &  \textbf{271} &  296 &  371 &  118 &  371 &  375 &  369 &  273 &  312 &  294 & \textbf{0.730} & \textbf{0.992} & \textbf{SPSC}
  \\
   45 &  5 &  \textbf{360} &  389 &  504 &  212 &  506 &  510 &  502 &  376 &  419 &  431 & \textbf{0.715} & \textbf{0.957} & \textbf{SPSC}
  \\
   45 & 10 &  \textbf{499} &  520 &  675 &  395 &  665 &  684 &  668 &  559 &  646 &  650 & \textbf{0.739} & \textbf{0.893} & \textbf{SPSC}
  \\
   45 & 15 &  \textbf{651} &  684 &  819 &  577 &  812 &  836 &  811 &  736 &  811 &  808 & \textbf{0.796} & \textbf{0.885} & \textbf{SPSC}
  \\
   45 & 20 &  \textbf{811} &  836 &  973 &  758 &  954 &  980 &  952 &  901 &  988 &  989 & \textbf{0.833} & \textbf{0.900} & \textbf{SPSC}
  \\
   60 &  1 &  \textbf{141} &  146 &  165 &   30 &  166 &  166 &  165 &  108 &  134 &  134 & \textbf{0.858} & 1.311 & \textbf{SPSC} \\
   60 &  3 &  \textbf{229} &  243 &  296 &  117 &  294 &  296 &  292 &  222 &  249 &  239 & \textbf{0.776} & 1.033 & \textbf{SPSC} \\
   60 &  5 &  \textbf{315} &  319 &  414 &  199 &  412 &  416 &  414 &  320 &  366 &  343 & \textbf{0.760} & \textbf{0.985} & \textbf{SPSC}
  \\
   60 & 10 &  \textbf{441} &  479 &  605 &  367 &  600 &  608 &  601 &  536 &  564 &  585 & \textbf{0.729} & \textbf{0.822} & \textbf{SPSC}
  \\
   60 & 15 &  \textbf{588} &  605 &  743 &  524 &  740 &  752 &  738 &  713 &  752 &  736 & \textbf{0.791} & \textbf{0.824} & \textbf{SPSC}
  \\
   60 & 20 &  \textbf{708} &  753 &  853 &  663 &  845 &  868 &  845 &  838 &  882 &  862 & \textbf{0.830} & \textbf{0.845} & \textbf{SPSC}
  \\
   80 &  1 &  \textbf{110} &  125 &  118 &   27 &  118 &  118 &  118 &   91 &  111 &  106 & \textbf{0.931} & 1.214 & \textbf{SPSC} \\
   80 &  3 &  \textbf{231} &  250 &  290 &  116 &  292 &  293 &  290 &  224 &  238 &  239 & \textbf{0.796} & 1.034 & \textbf{SPSC} \\
   80 &  5 &  \textbf{279} &  295 &  381 &  187 &  379 &  384 &  379 &  293 &  307 &  319 & \textbf{0.733} & \textbf{0.954} & \textbf{SPSC}
  \\
   80 & 10 &  \textbf{419} &  430 &  542 &  345 &  539 &  545 &  538 &  447 &  493 &  473 & \textbf{0.772} & \textbf{0.937} & \textbf{SPSC}
  \\
   80 & 15 &  \textbf{534} &  576 &  671 &  486 &  668 &  671 &  667 &  587 &  638 &  655 & \textbf{0.796} & \textbf{0.911} & \textbf{SPSC}
  \\
   80 & 20 &  \textbf{640} &  663 &  767 &  599 &  761 &  770 &  761 &  691 &  744 &  761 & \textbf{0.835} & \textbf{0.927} & \textbf{SPSC}
  \\
  100 &  1 &  \textbf{107} &  111 &  112 &   28 &  113 &  113 &  113 &   79 &   86 &   93 & \textbf{0.954} & 1.359 & \textbf{SPSC} \\
  100 &  3 &  \textbf{192} &  194 &  238 &  107 &  236 &  238 &  236 &  176 &  193 &  201 & \textbf{0.809} & 1.094 & \textbf{SPSC} \\
  100 &  5 &  \textbf{272} &  273 &  346 &  183 &  343 &  346 &  344 &  266 &  294 &  290 & \textbf{0.787} & 1.024 & \textbf{SPSC} \\
  100 & 10 &  \textbf{368} &  388 &  471 &  317 &  472 &  472 &  470 &  403 &  438 &  440 & \textbf{0.780} & \textbf{0.912} & \textbf{SPSC}
  \\
  100 & 15 &  \textbf{466} &  491 &  573 &  430 &  571 &  576 &  572 &  517 &  530 &  535 & \textbf{0.813} & \textbf{0.902} & \textbf{SPSC}
  \\
  100 & 20 &  \textbf{609} &  634 &  726 &  572 &  725 &  733 &  723 &  661 &  692 &  702 & \textbf{0.839} & \textbf{0.921} & \textbf{SPSC}
  \\
  \bottomrule
  \end{tabular}
  \end{table}

\section{Sensitivity and robustness studies}
\label{app:sensitivity}

\subsection{Probe-rate ablation}
\label{app:probe_ablation}

Sweeping the probe period over $\{5,10,20,30,50,100,300\}$ at
$d{=}4$, $r{=}1$, $K{=}4$, $T{=}6{,}000$ produces a clearly U-shaped
regret curve (Figure~\ref{fig:probe-ablation}): sparse probing
under-recovers the subspace; over-frequent probing sacrifices
exploitation. Optimum at probe period $20$--$30$, matching
$m_k^\star\propto \ell_k^{2/3}$ prediction. The right panel shows a
monotone relationship between late-segment subspace error and final
regret, directly linking SPSC's gains to subspace recovery quality.

\begin{figure}[!htbp]
\centering
\includegraphics[width=0.96\textwidth]{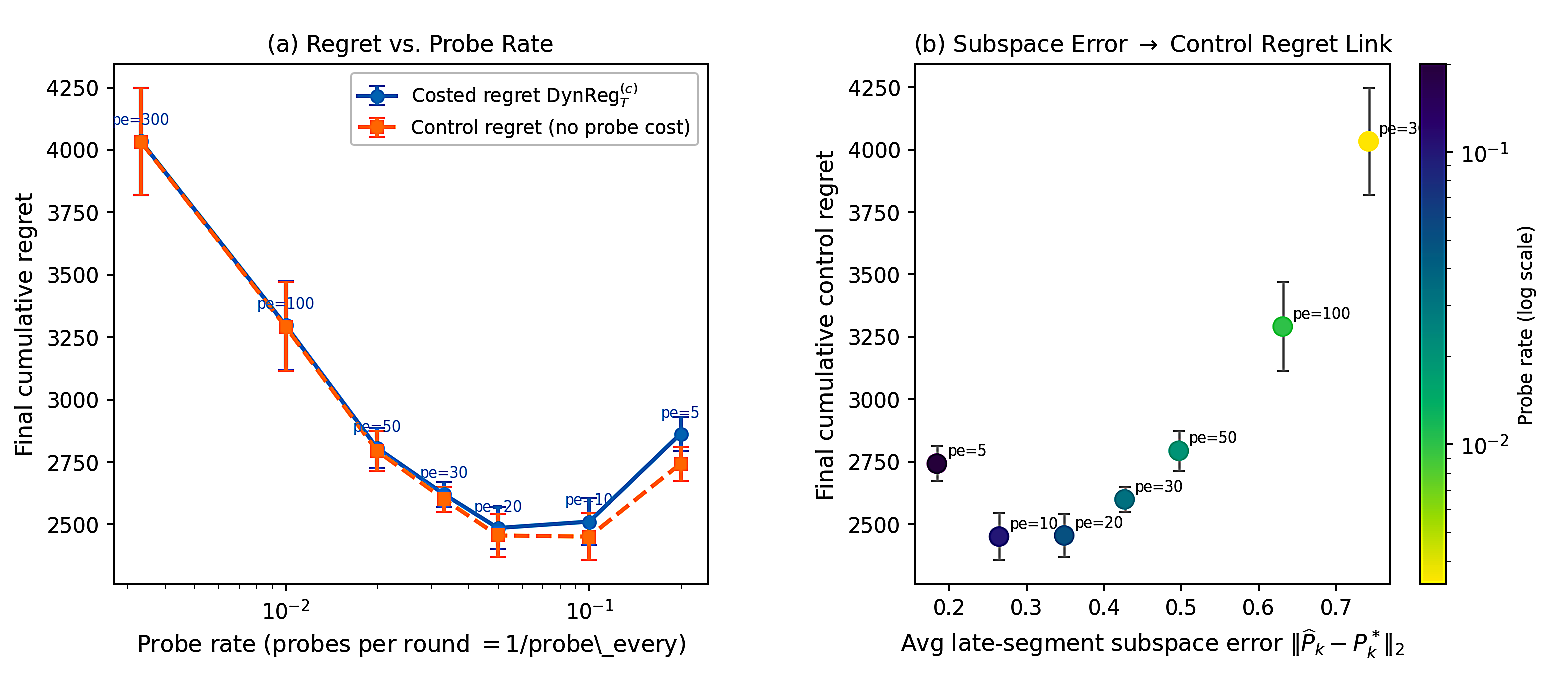}
\caption{Probe-rate ablation. Left: final regret vs.\ probe
frequency. Right: final control regret vs.\ late-segment subspace
error.}
\label{fig:probe-ablation}
\end{figure}

\subsection{Rank misspecification}
\label{app:rank_misspec}

On Covertype ($d{=}155$, true effective rank $r^\star{=}10$, $K{=}4$,
$T{=}10{,}000$, 5 seeds), sweeping specified $r$ at the grid
points $\{1,3,5,10,15,20,30,50,80\}$ (min $r{=}1$, max $r{=}80$)
gives the U-shaped curve in Figure~\ref{fig:rank_misspec}.
Underestimation loses signal ($+33\%$ at $r{=}1$); mild overestimation
captures residual structure (sweet spot at $r{=}30$, $-10\%$). SPSC
beats LinUCB across $r\in[15,50]$, so mild overestimation is a safe
default.

\begin{figure}[!htbp]
\centering
\includegraphics[width=0.85\textwidth]{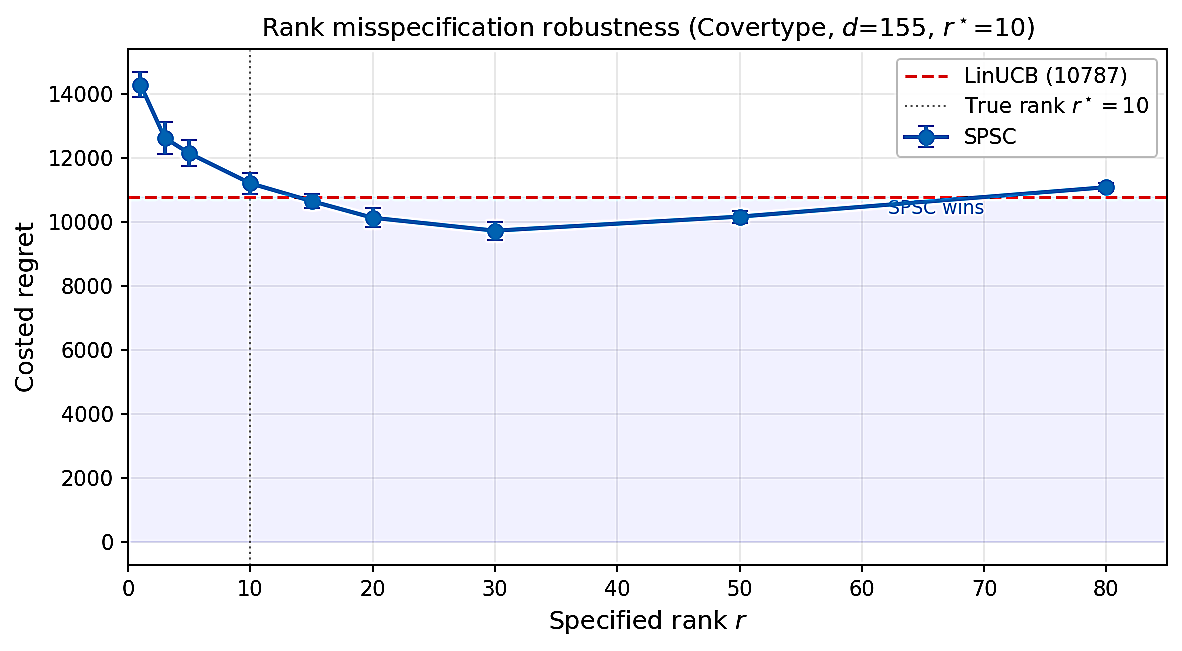}
\caption{Rank misspecification robustness (Covertype, $d{=}155$,
$r^\star{=}10$). SPSC beats LinUCB for $r\in[15,50]$; sweet spot at
$r{=}30$ ($10\%$ improvement).}
\label{fig:rank_misspec}
\end{figure}

\subsection{Robustness and necessity experiments (A/B/C)}
\label{app:robustness_abc}

Three ablations on the reference setting ($d{=}4$, $r{=}1$, $K{=}4$,
$T{=}6{,}000$, probe period 30, $W{=}100$, 10 seeds).

\paragraph{A. Variance misspecification.}
Injecting $|\delta_\sigma|\in\{0,0.01,0.02,0.05,0.1,0.2,0.5\}$ into
the probe statistic $s_t=y_t^2-(\sigma_\varepsilon^2+\delta_\sigma)$.
Regret and subspace error are stable through $|\delta_\sigma|=0.5$
(more than $5\times$ the true variance $\sigma_\varepsilon^2=0.09$);
at maximum misspecification, regret changes by $<5\%$. For
scaled-sphere probes (the theory probe distribution), a constant
centering error contributes the population-level identity shift
$\widetilde B = -(\delta_\sigma/d)I_d$, which uniformly shifts
eigenvalues without corrupting eigenvectors.

\paragraph{B. Bounded cross-correlation.}
Injecting $\varepsilon_t\to\varepsilon_t+\epsilon_\times x_t^\top\theta_t$
for $\epsilon_\times\in\{0,0.01,0.05,0.1,0.2,0.5,1.0\}$. Even at
$\epsilon_\times=1.0$ (cross-correlation equal to signal), regret
increases by $<1\%$, indicating that moderate violations of orthogonality are
practically benign.

\paragraph{C. Imperfect probe coverage.}
Restricting probe directions to the first $d_{\mathrm{cov}}\in\{1,2,3,4\}$
coordinates produces a dramatic monotonic regret blow-up: at
$d_{\mathrm{cov}}=1$, SPSC regret rises to $4{,}289$, nearly
matching LinUCB ($4{,}463$); subspace error degrades to $0.75$.
This directly confirms the necessity results: restricted coverage
destroys identifiability and forces $\Omega(T)$ regret within the
unobserved directions.

\begin{figure}[!htbp]
\centering
\includegraphics[width=0.96\textwidth]{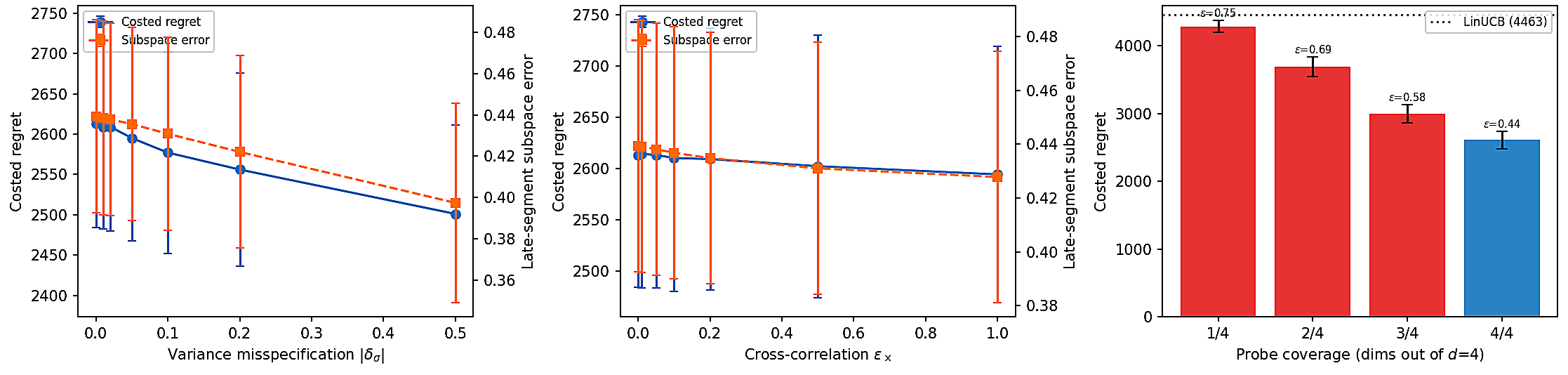}
\caption{Robustness and necessity experiments. (a) Variance
misspecification, regret stable through $|\delta_\sigma|=0.5$.
(b) Cross-correlation relaxation, $<1\%$ degradation through
$\epsilon_\times=1.0$. (c) Imperfect coverage, regret blow-up with
restricted probe directions.}
\label{fig:robustness-abc}
\end{figure}

\subsection{Assumption-violation master table (large-scale)}
\label{app:assumption_violation}

We complement the small-scale A/B/C ablations above with a
large-scale ($d{=}60$, $r{=}5$, $K{=}10$, $T{=}5{,}000$) assumption
sweep over (a) variance misspecification
$\widehat\sigma^2/\sigma_\varepsilon^2\in\{0.5,1,2,4\}$,
(b) approximate-rank perturbation
$\epsilon_k^\perp\in\{0,0.05,0.1,0.2,0.4\}$, and (c) spectral radius
$\rho(A_k)\in\{0.8,0.95,0.99,1.0\}$ (the last violates the stable-LDS
hypothesis). Variance misspec is absorbed into the
scaled-identity bias (Lemma~\ref{lem:G_unbiased_conf}) and the
relative gap stays flat ($\widehat\sigma^2/\sigma_\varepsilon^2$ ratio
$1.27\to 1.16$ as the sweep goes $0.5\to 4$); approximate-rank
degrades regret near-linearly in $\epsilon_k^\perp$; SPSC remains
well-behaved through $\rho=0.99$ and degrades gracefully at
$\rho=1$ (random-walk limit, outside the stable-LDS assumption).

\begin{table}[!htbp]
\centering
\caption{\textbf{Assumption-violation master table} ($d{=}60$,
$r{=}5$, $K{=}10$, $T{=}5{,}000$, 10 seeds, costed dynamic regret,
mean$\pm$SE). SPSC tracks Oracle through all three sweeps; bold =
best non-Oracle in row.}
\label{tab:assumption_master}
\footnotesize
\setlength{\tabcolsep}{4pt}
\begin{tabular}{lcrrrrrrr}
\toprule
Sweep & Level & Oracle & SPSC-Alg1 & SPSC-Adap & LinUCB & SW-LinUCB & D-LinUCB & LowOFUL \\
\midrule
\multirow{4}{*}{$\widehat\sigma^2/\sigma_\varepsilon^2$}
 & $0.5$ & $221\pm5$  & $\mathbf{281\pm10}$ & $304\pm11$ & $397\pm17$ & $395\pm16$ & $401\pm18$ & $370\pm19$ \\
 & $1.0$ & $257\pm6$  & $\mathbf{315\pm10}$ & $319\pm10$ & $413\pm18$ & $412\pm18$ & $415\pm19$ & $358\pm17$ \\
 & $2.0$ & $299\pm10$ & $357\pm13$ & $\mathbf{355\pm13}$ & $421\pm19$ & $421\pm18$ & $423\pm20$ & $377\pm14$ \\
 & $4.0$ & $344\pm12$ & $398\pm16$ & $398\pm16$ & $426\pm19$ & $425\pm19$ & $427\pm20$ & $\mathbf{390\pm20}$ \\
\midrule
\multirow{5}{*}{$\epsilon_k^\perp$}
 & $0.00$ & $257\pm6$ & $\mathbf{315\pm10}$ & $319\pm10$ & $413\pm18$ & $412\pm18$ & $415\pm19$ & $358\pm17$ \\
 & $0.05$ & $255\pm7$ & $\mathbf{311\pm12}$ & $331\pm10$ & $414\pm18$ & $414\pm18$ & $417\pm19$ & $379\pm20$ \\
 & $0.10$ & $261\pm7$ & $\mathbf{319\pm13}$ & $331\pm12$ & $422\pm19$ & $422\pm18$ & $425\pm19$ & $376\pm20$ \\
 & $0.20$ & $286\pm8$ & $\mathbf{340\pm13}$ & $348\pm6$  & $452\pm19$ & $453\pm18$ & $457\pm20$ & $399\pm19$ \\
 & $0.40$ & $358\pm9$ & $\mathbf{396\pm10}$ & $410\pm10$ & $559\pm21$ & $561\pm20$ & $565\pm22$ & $497\pm13$ \\
\midrule
\multirow{4}{*}{$\rho(A_k)$}
 & $0.80$ & $713\pm6$   & $745\pm7$    & $747\pm9$    & $744\pm8$   & $743\pm6$   & $745\pm8$   & $\mathbf{732\pm9}$ \\
 & $0.95$ & $1083\pm11$ & $\mathbf{1291\pm16}$ & $1307\pm14$  & $1403\pm16$ & $1406\pm15$ & $1404\pm19$ & $1330\pm18$ \\
 & $0.99$ & $1132\pm23$ & $2017\pm46$  & $\mathbf{1929\pm34}$ & $2751\pm46$ & $2732\pm44$ & $2750\pm52$ & $2585\pm54$ \\
 & $1.00$ & $914\pm19$  & $2899\pm83$  & $\mathbf{2813\pm59}$ & $4644\pm90$ & $4523\pm87$ & $4674\pm90$ & $5552\pm156$ \\
\bottomrule
\end{tabular}
\end{table}

\subsection{Subspace recovery rate and change-point adaptation}
\label{app:subspace}

Figure~\ref{fig:subspace-recovery} plots the binned mean subspace
error $\|\widehat P_k - P_k^\star\|_2$ vs.\ probe count; the empirical
rate matches the predicted $1/\sqrt{m}$ on a log-log scale.
Figure~\ref{fig:changepoint-adaptation} plots post-change
instantaneous regret and the subspace-error trajectory: SPSC recovers
substantially faster than ambient LinUCB after each boundary.

\begin{figure}[!htbp]
\centering
\includegraphics[width=\textwidth]{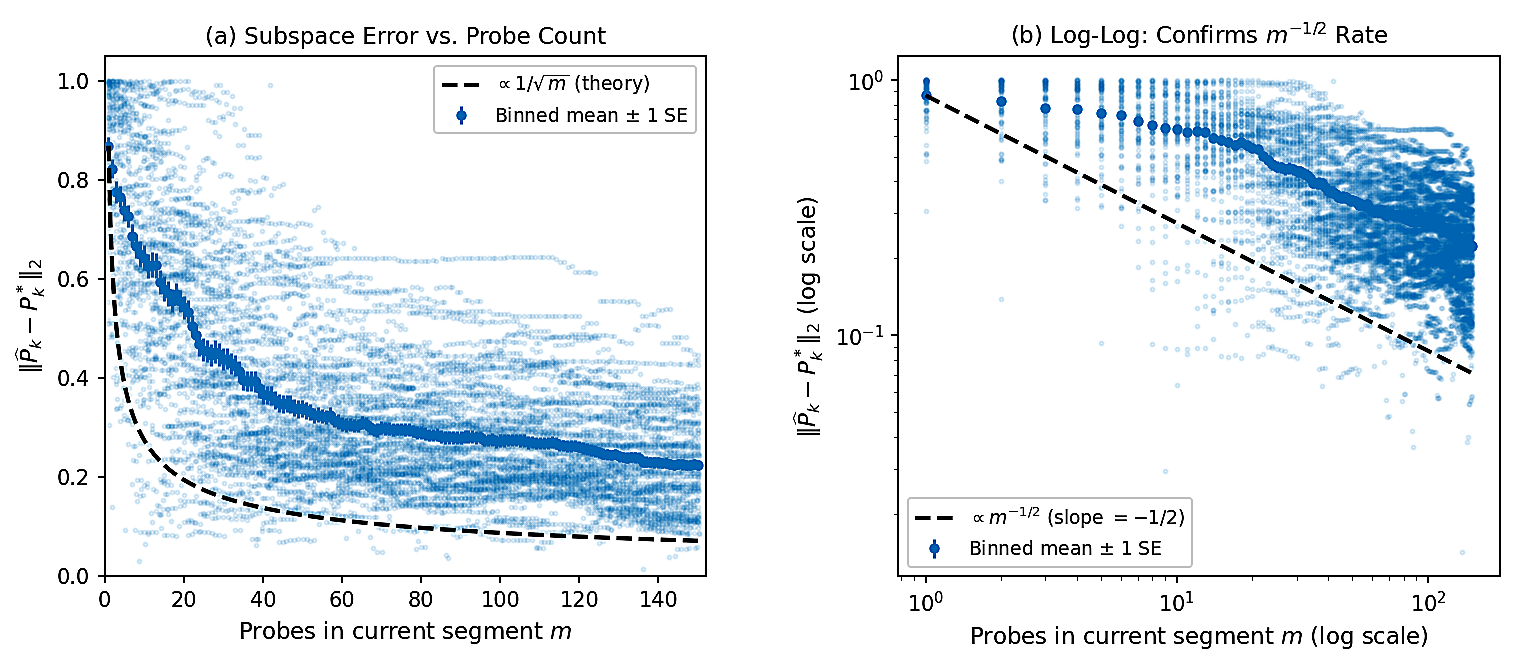}
\caption{\textbf{Subspace recovery.}
$\|\widehat P_k-P_k^\star\|_2$ vs.\ probe count matches the predicted
$1/\sqrt m$ rate on a log-log scale.}
\label{fig:subspace-recovery}
\end{figure}

\begin{figure}[!htbp]
\centering
\includegraphics[width=\textwidth]{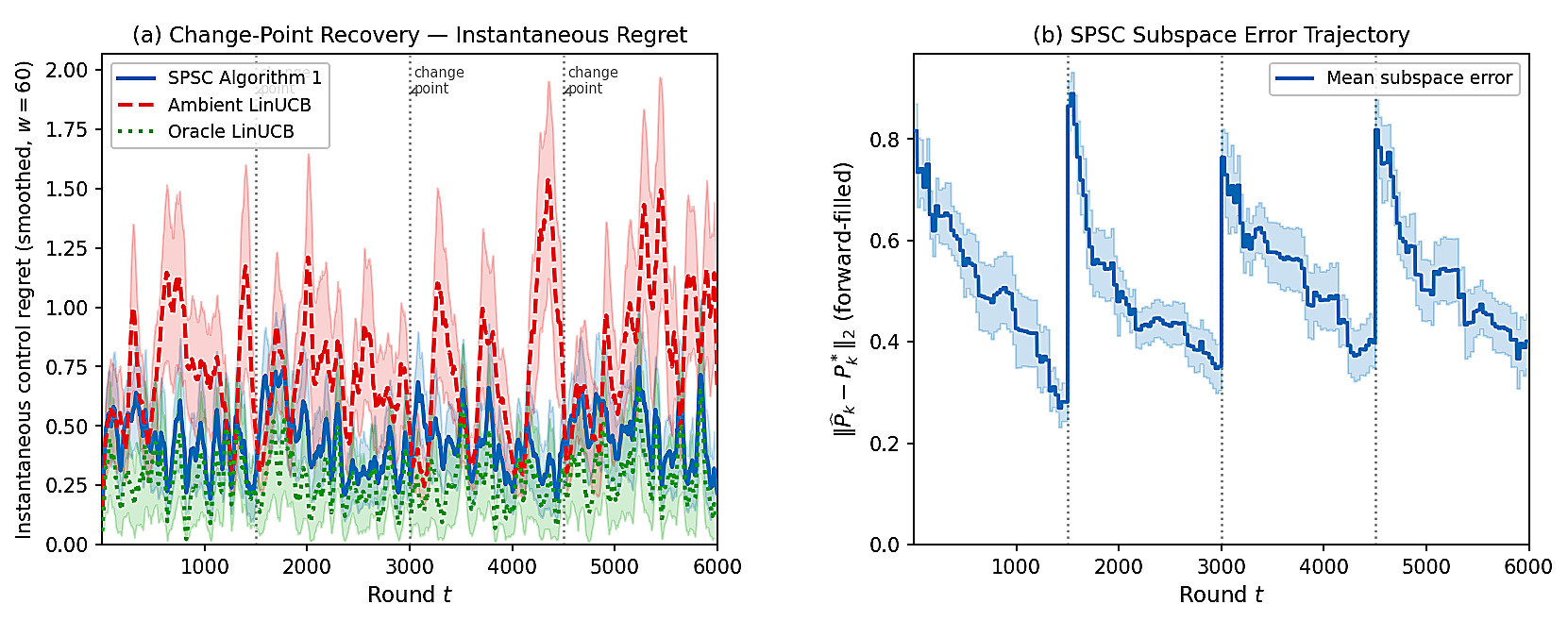}
\caption{\textbf{Change-point adaptation.}
Instantaneous regret drops sharply after each boundary as fresh
probes rebuild the subspace.}
\label{fig:changepoint-adaptation}
\end{figure}

\subsection{Noise robustness, change-point frequency, drift speed}
\label{app:noise_cp_drift}

Three single-parameter sweeps on the reference setting.

\paragraph{Noise.}
Sweeping $\sigma_\varepsilon\in\{0.05,0.10,0.20,0.30,0.50,0.80\}$:
SPSC/LinUCB ratio grows mildly with noise (from $0.566$ to $0.595$);
SPSC/Oracle stable at $\approx 1.55\times$. The low-rank benefit is
not a low-noise artifact.

\paragraph{Change-point frequency.}
$K\in\{1,2,4,6,8,12\}$ at $T{=}6{,}000$. Ratio grows from
$0.414$ (long segments) to $0.855$ (very short); SPSC retains an
edge at all frequencies but the margin shrinks as the per-segment
probe budget collapses below the identifiability threshold.

\paragraph{Drift speed.}
LDS spectral radius $\rho\in\{0.30,\ldots,0.99\}$. Sharp phase
transition: SPSC/LinUCB $\approx 1.01$ for $\rho\le 0.70$, drops to
$0.861$ at $\rho{=}0.95$, and to $0.582$ at $\rho{=}0.99$. Subspace
error is nearly constant across $\rho$ ($0.28$--$0.33$), so the
missing benefit at low $\rho$ is a signal-energy effect, not
estimation failure: low-rank structure in a signal already
dominated by noise leaves little to extract.

\begin{figure}[!htbp]
\centering
\includegraphics[width=\textwidth]{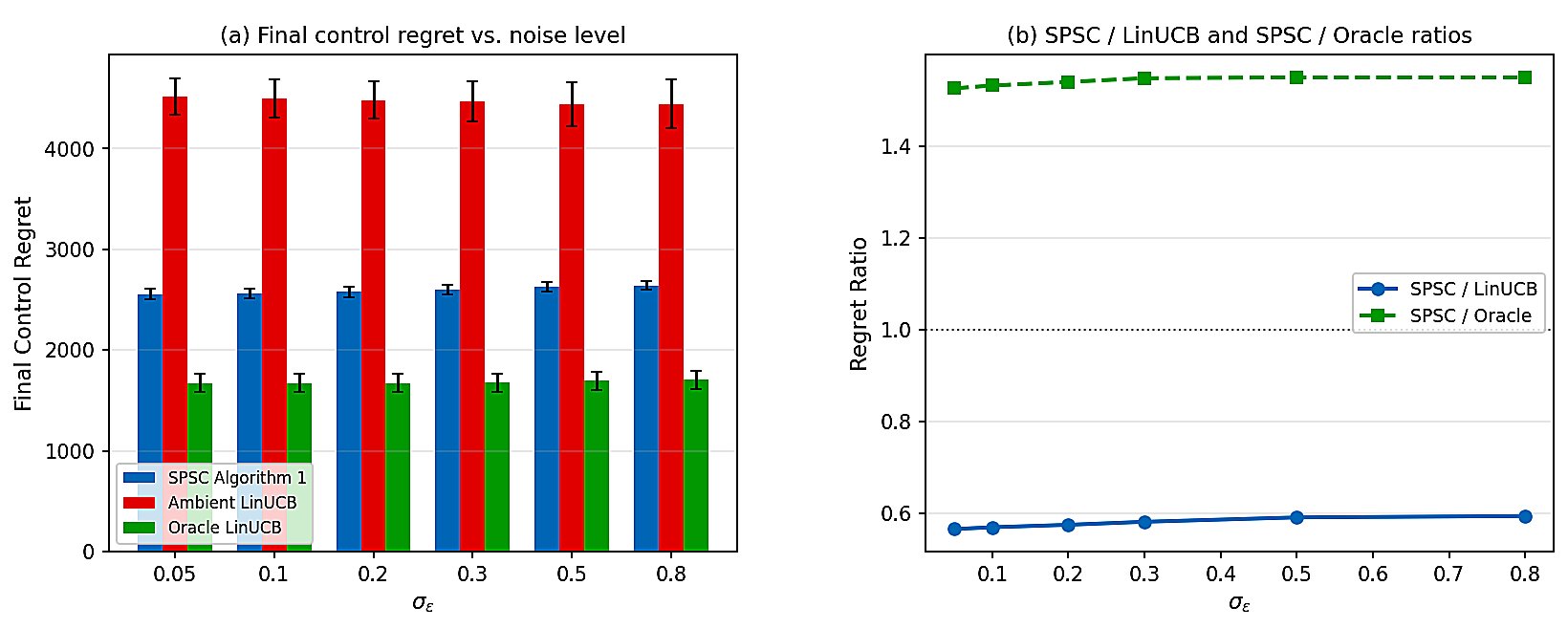}
\caption{Noise robustness: SPSC advantage is noise-monotone.}
\label{fig:exp6}
\end{figure}

\begin{figure}[!htbp]
\centering
\includegraphics[width=\textwidth]{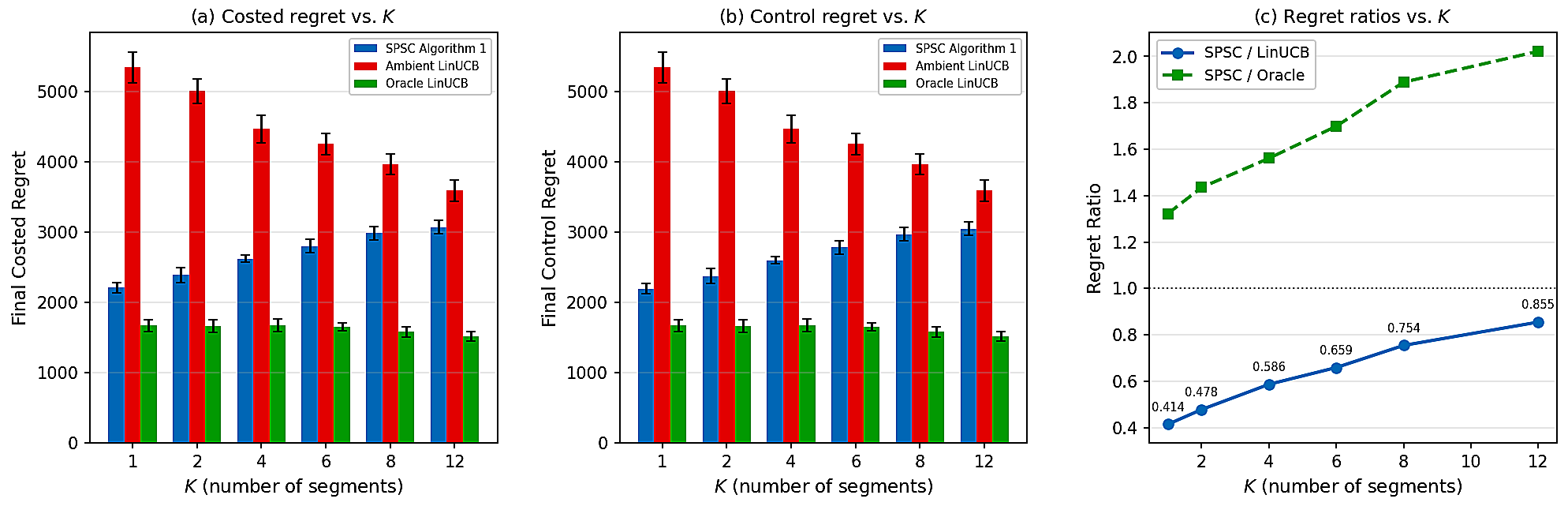}
\caption{Change-point frequency.}
\label{fig:exp7}
\end{figure}

\begin{figure}[!htbp]
\centering
\includegraphics[width=\textwidth]{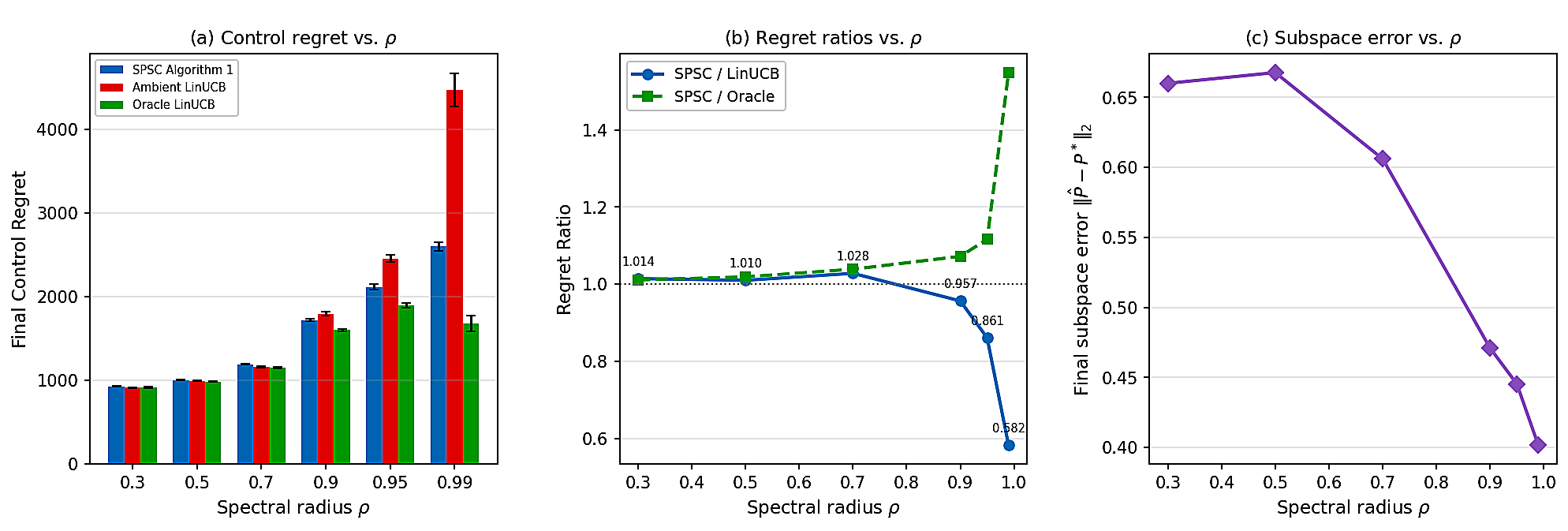}
\caption{Drift speed. Benefit gated by LDS correlation time
$\tau_\rho=1/(1-\rho)$.}
\label{fig:exp8}
\end{figure}

\section{Comparison with nonstationary baselines (small-$d$ stress test)}
\label{app:sota_compare}

We adopt the small-$d$ piecewise-stationary stress test of
\citet{russac2019weighted} both to verify that SPSC's gains do not
depend on large-$d$ low-rank regimes and to compare SPSC against the
standard sliding-window and discounted baselines on the setup used
to validate D-LinUCB. Setup: $d{=}2$, $r{=}1$, $K{=}4$,
$T{=}6{,}000$, $\sigma_\varepsilon{=}1$, 50 arms, 30 seeds. SPSC
reduces control regret by $46.5\%$ vs.\ D-LinUCB and $34.8\%$ vs.\
SW-LinUCB; probe overhead is a $0.9\%$ fraction.

\begin{table}[!htbp]
\centering
\caption{Final control and costed regret on the small-$d$
piecewise-stationary stress test of \citet{russac2019weighted}
(mean$\pm$SE, 30 seeds).}
\label{tab:exp9}
\small
\begin{tabular}{lrrl}
\toprule
Algorithm & Control & Costed & vs.\ D-LinUCB \\
\midrule
\textbf{SPSC Alg.\,1 (ours)}        & $2{,}229 \pm 54$  & $2{,}249$ & $\mathbf{-46.5\%}$\\
Oracle LinUCB                       & $1{,}757 \pm 55$  & $1{,}757$ & $-57.8\%$\\
D-LinUCB~\citep{russac2019weighted} & $4{,}167 \pm 110$ & $4{,}167$ & ---\\
SW-LinUCB~\citep{cheung2019learning}& $3{,}417 \pm 82$  & $3{,}417$ & $-18.0\%$\\
OFUL~\citep{abbasi2011improved}     & $4{,}645 \pm 149$ & $4{,}645$ & $+11.5\%$\\
Reset-LinUCB (oracle)               & $4{,}647 \pm 147$ & $4{,}647$ & $+11.5\%$\\
\bottomrule
\end{tabular}
\end{table}

\begin{figure}[!htbp]
\centering
\includegraphics[width=\textwidth]{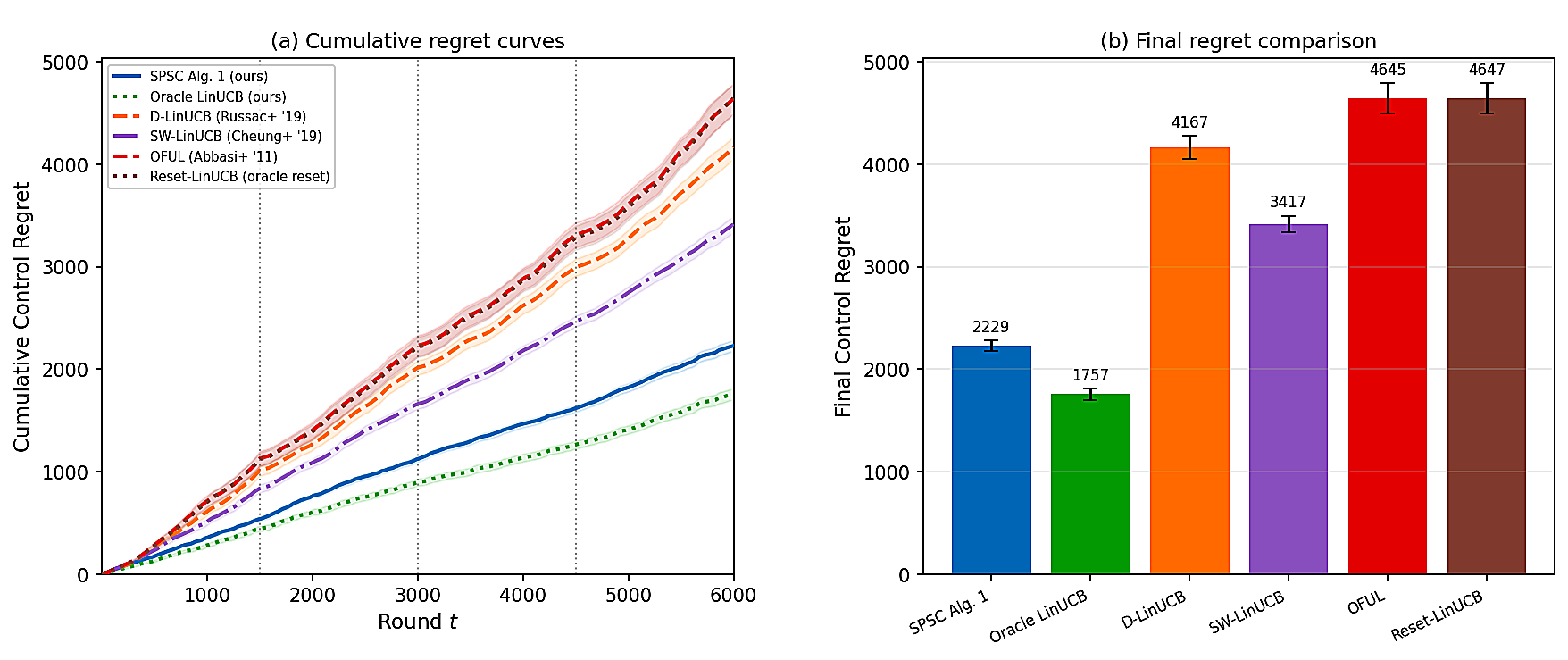}
\caption{Small-$d$ piecewise-stationary stress test of
\citet{russac2019weighted} ($d{=}2$, $r{=}1$, $K{=}4$,
$T{=}6{,}000$, 30 seeds). \textbf{(a)} Cumulative control regret;
vertical dotted lines mark the three change points.
\textbf{(b)} Final control regret per method; SPSC is the best
non-oracle baseline.}
\label{fig:exp9}
\end{figure}

\section{Warfarin: random-subspace ablation}
\label{app:warfarin_ablation}

The Warfarin gain reported in \S\ref{sec:exp_warfarin} admits two
distinct explanations: projection of the $93$-dimensional regression
onto an $r$-dimensional subspace, which alone reduces the effective
dimension and stabilizes the ridge estimate; and identification of
the \emph{correct} rank-$r$ subspace via the probe-based estimator.
The two cannot be separated from a single regret comparison, since
SPSC performs both simultaneously. We disentangle them by replacing
the learned subspace $\widehat U_t$ with a fresh random rank-$r$
projection at every segment, leaving every other component of SPSC
unchanged; the remaining setup mirrors \S\ref{sec:exp_warfarin}
($d{=}93$, $K{=}8$, $T{=}5{,}000$, $r$ swept over $\{1,2,3,5,10\}$,
$10$ seeds). The random-subspace variant captures roughly half of
the SPSC-vs-LinUCB regret gap, so dimensionality reduction accounts
for approximately half of the gain and identification of the
correct subspace for the remainder. Consistent with this
decomposition, the learned subspace retains between $2\times$ and
$6\times$ more signal variance than a random rank-$r$ projection at
every $r$, quantifying the value of the probe-based identification.

\begin{figure}[!htbp]
\centering
\includegraphics[width=0.95\textwidth]{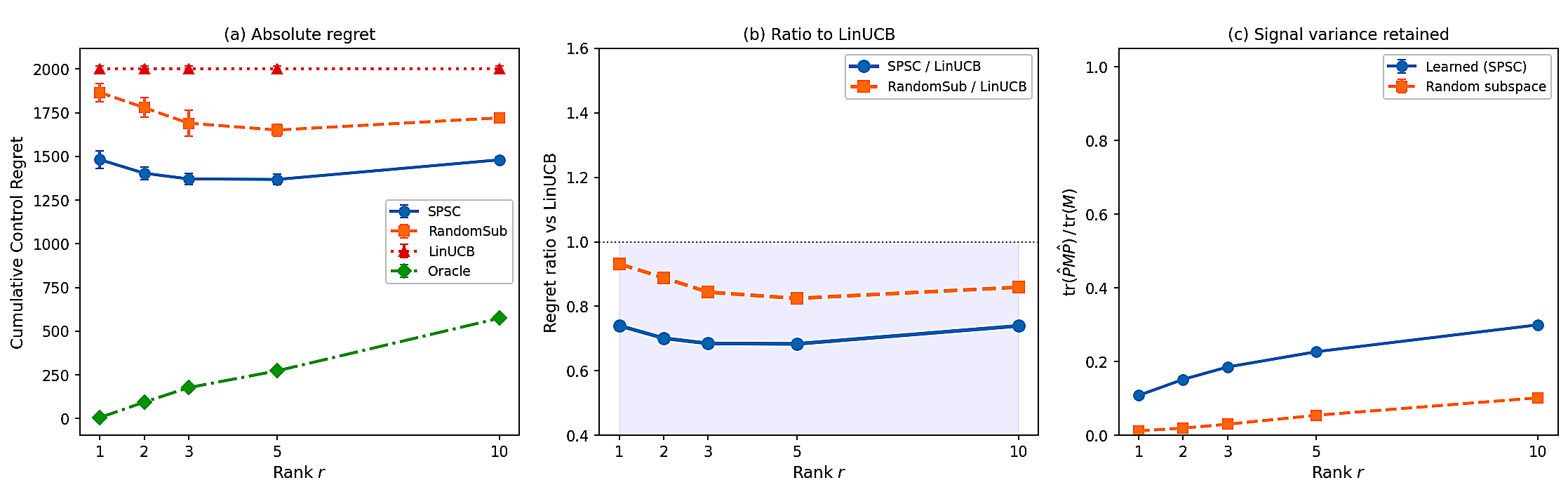}
\caption{\textbf{Warfarin random-subspace ablation.} SPSC's learned
subspace vs.\ a fresh random rank-$r$ projection at every segment.
Half of the gap to LinUCB is dimensionality reduction; the other half
is learned signal.}
\label{fig:warfarin-ablation}
\end{figure}

\section{BOSS/Jedra adaptation to the piecewise setting}
\label{app:boss_jedra_detail}

Stationary low-rank methods exploit the same rank structure as SPSC
but assume the subspace is fixed, so without segment-aware restarts
they would fail trivially in the piecewise-stationary setting. To
isolate the contribution of \emph{online} subspace identification
rather than mere low-rank exploitation, we provide
BOSS~\citep{duong2024beyond} and Jedra~\citep{jedra2024low} with the
true segment boundaries (an oracle advantage SPSC does not receive),
restart each at every boundary, and otherwise leave their
hyperparameters at the authors' recommended defaults. The question
is then whether identifying a \emph{changing} subspace online still
outperforms exploiting a \emph{fixed} low-rank reward offline, even
under this oracle advantage. The benchmark is piecewise LDS,
$K{=}10, T{=}5{,}000, \sigma_\varepsilon{=}0.3$, 40 actions, 10
seeds, probe cost $c{=}0.1$, on a grid $d\in\{55,105,200\}\times
r\in\{5,10,20\}$ (eight cells:
$(55,5),(55,10),(105,5),(105,10),(105,20),(200,5),(200,10),(200,20)$).

Results (cf.\ Table~\ref{tab:boss_jedra_grid}): SPSC Alg.\,1 wins
every cell, by $5.8$--$19.9\%$, with the lead largest at small $d$ (where
the ambient $d$ is closest to $r$ and probing pays off most) and
narrowing to $5.8$--$8.0\%$ at $d{=}200$ as oracle-restart-plus-PCA
closes the gap. SPSC-Adaptive (no oracle restarts) also beats BOSS
and Jedra on every cell, and the ranking is consistent on every
seed. The point is that ``low-rank exploitation'' alone is not
enough for changing-subspace problems: the learner must
\emph{identify a changing subspace online}, which is SPSC's core
capability.

\section{Vancomycin: full per-rank tables}
\label{app:vancomycin_full}

The body table reports Vancomycin at a single rank, but the true
rank of the pharmacokinetic response is not known in advance. We
sweep $r$ here to confirm that SPSC-Adaptive's advantage holds
across rank choices. Setup: $d{=}93$, $K{=}8$, $T{=}5{,}000$,
$\sigma_\varepsilon{=}0.3$, 10 seeds, probe period 10, window 200,
probe cost $c{=}0.1$ (per-rank breakdown of the
Table~\ref{tab:vancomycin_body} summary in the main body).
SPSC-Adaptive beats every non-oracle baseline at every rank.

\begin{table}[!htbp]
\centering
\caption{\textbf{Vancomycin per-rank results} (mean$\pm$SE, 10
seeds). ``vs.\ best comp.'' compares SPSC-Adaptive against the
strongest non-oracle baseline in each column (marked $^\star$).
SPSC-Adaptive wins every rank.}
\label{tab:vancomycin_full}
\footnotesize
\setlength{\tabcolsep}{2pt}
\begin{tabular}{l@{\hskip 3pt}rrrrr}
\toprule
Method & $r{=}1$ & $r{=}2$ & $r{=}3$ & $r{=}5$ & $r{=}10$ \\
\midrule
Oracle LinUCB                & $34.3\pm 1.5$    & $115.1\pm 2.0$   & $197.2\pm 3.0$         & $268.1\pm 3.4$         & $373.2\pm 5.7$ \\
SPSC Alg.\,1 (ours)          & $840.3\pm 24.5$  & $805.9\pm 15.4$  & $802.3\pm 12.3$        & $805.3\pm 9.5$         & $811.0\pm 12.2$ \\
\textbf{SPSC-Adaptive}       & $\mathbf{525.3\pm 45.8}$ & $\mathbf{613.6\pm 33.1}$ & $\mathbf{640.3\pm 33.0}$ & $\mathbf{665.9\pm 27.2}$ & $\mathbf{705.6\pm 19.5}$ \\
LowOFUL                      & $1049.9\pm 12.8$ & $740.1\pm 36.3^\star$  & $734.5\pm 37.8$        & $791.5\pm 26.0$        & $801.9\pm 20.1$ \\
VOFUL                        & $1063.5\pm 5.5$  & $789.7\pm 24.8$  & $712.5\pm 25.5^\star$  & $763.0\pm 30.6^\star$  & $758.6\pm 27.9^\star$ \\
LowRank-Reward               & $1175.6\pm 11.6$ & $806.8\pm 11.2$  & $835.1\pm 10.7$        & $898.8\pm 6.7$         & $880.5\pm 5.5$ \\
SW-LinUCB                    & $875.6\pm 6.7$   & $875.6\pm 6.7$   & $875.6\pm 6.7$         & $875.6\pm 6.7$         & $875.6\pm 6.7$ \\
D-LinUCB                     & $925.3\pm 5.8$   & $925.3\pm 5.8$   & $925.3\pm 5.8$         & $925.3\pm 5.8$         & $925.3\pm 5.8$ \\
Restart-LinUCB               & $808.9\pm 5.3^\star$ & $808.9\pm 5.3$ & $808.9\pm 5.3$         & $808.9\pm 5.3$         & $808.9\pm 5.3$ \\
LinUCB                       & $825.3\pm 6.1$   & $825.3\pm 6.1$   & $825.3\pm 6.1$         & $825.3\pm 6.1$         & $825.3\pm 6.1$ \\
\midrule
\multicolumn{1}{l}{SPSC-Adaptive vs.\ LinUCB}     & $-36.4\%$ & $-25.7\%$ & $-22.4\%$ & $-19.3\%$ & $-14.5\%$ \\
\multicolumn{1}{l}{SPSC-Adaptive vs.\ best comp.} & $-35.1\%$ & $-17.1\%$ & $-10.1\%$ & $-12.7\%$ & $\phantom{-}-7.0\%$ \\
\bottomrule
\end{tabular}
\end{table}

\section{When the adaptive variant wins}
\label{app:adaptive_winner}

This appendix disentangles the two SPSC variants:
Alg.~\ref{alg:spsc} assumes oracle change points;
Alg.~\ref{alg:spsc_adaptive} detects them online via the detector
statistic~\eqref{eq:detector}. When the oracle boundaries are
correct, Alg.~\ref{alg:spsc} is slightly better because the adaptive
detector pays a small statistical price for unknown change-point
localization. When the oracle provides misspecified (false-alarm)
boundaries, the picture reverses sharply.

\begin{figure}[!htbp]
\centering
\includegraphics[width=0.95\textwidth]{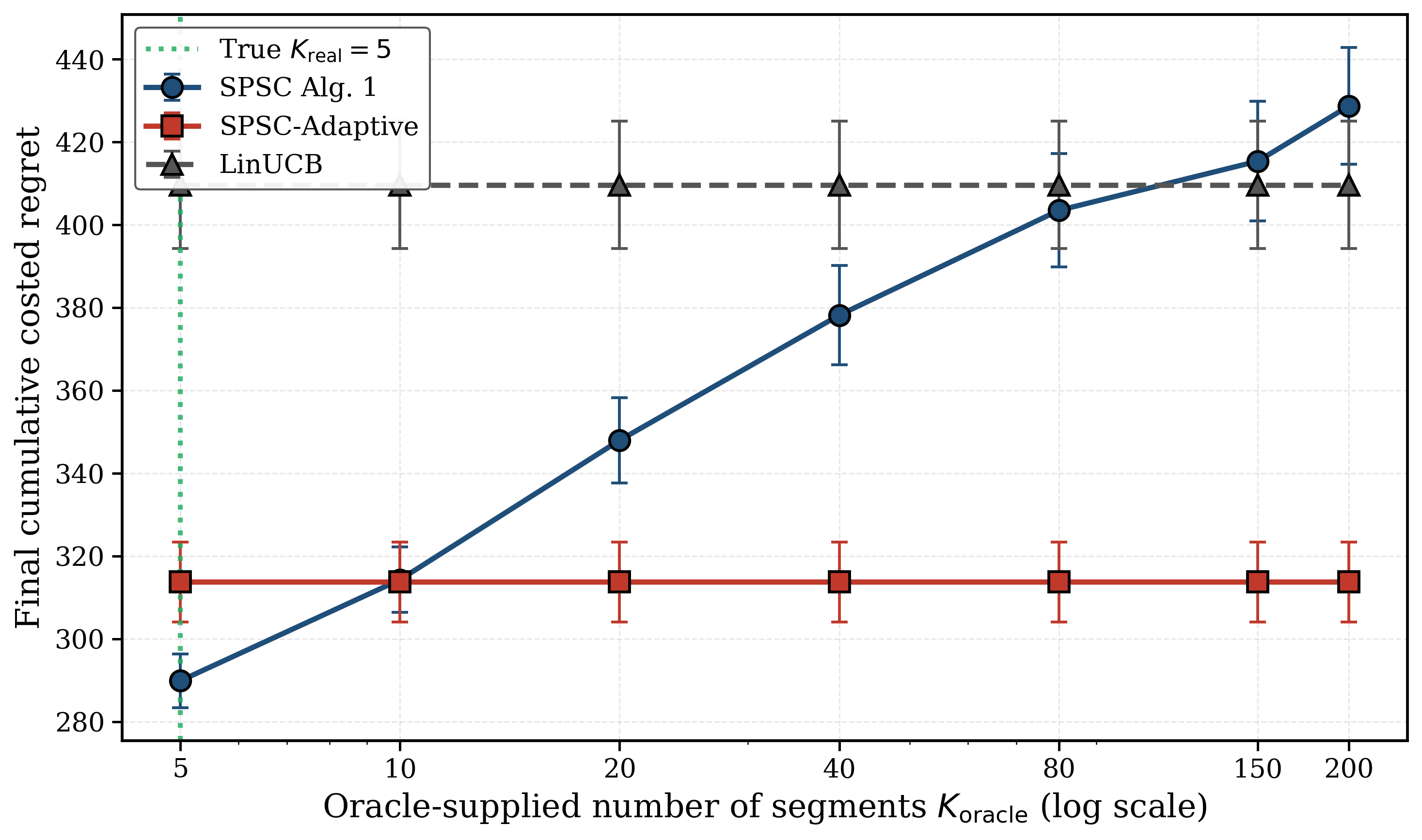}
\caption{\textbf{SPSC Alg.\,1 vs.\ SPSC-Adaptive under oracle
misspecification.} Final cumulative costed regret as the
oracle-supplied number of segments $K_{\mathrm{oracle}}$ varies, with
$K_{\mathrm{real}}{=}5$ true subspace shifts ($d{=}60$, $r{=}5$,
$T{=}5{,}000$, $20$ seeds; mean $\pm$ SE). Alg.~\ref{alg:spsc} is fed
$K_{\mathrm{oracle}}\in\{5,10,20,40,80,150,200\}$ evenly-spaced
boundaries (extras beyond $K_{\mathrm{real}}$ are false alarms) and
degrades monotonically; Alg.~\ref{alg:spsc_adaptive} runs CUSUM on
the clean stream and ignores the oracle, so its regret is flat;
LinUCB is an oracle-free reference. The two SPSC variants cross over
at $K_{\mathrm{oracle}}{\approx}10$: at the correct oracle
($K_{\mathrm{oracle}}{=}K_{\mathrm{real}}{=}5$) Alg.~\ref{alg:spsc}
is ${\sim}8\%$ better, while at $K_{\mathrm{oracle}}{=}200$ it is
${\sim}37\%$ worse.}
\label{fig:adaptive}
\end{figure}

SPSC-Adaptive's regret is flat as $K_{\mathrm{oracle}}$ grows from
the truth ($K_{\mathrm{real}}{=}5$) to $40\times$ over-specified
($K_{\mathrm{oracle}}{=}200$), while Alg.~\ref{alg:spsc} degrades
monotonically and is ${\sim}37\%$ worse than Adaptive at
$K_{\mathrm{oracle}}{=}200$ (a $1.37\times$ ratio): a detector-based
reset is strictly safer than a fixed schedule when boundaries are
unknown or noisy.

\end{document}